\documentclass{article}
\usepackage{iclr2027_conference,times}
\usepackage{amsmath,amssymb,amsthm,mathtools}
\usepackage{bm,enumitem,booktabs,multirow}
\usepackage{mathrsfs}
\usepackage{graphicx}
\usepackage{xcolor}
\usepackage[colorlinks=true,linkcolor=blue,citecolor=blue]{hyperref}

\newtheorem{theorem}{Theorem}
\newtheorem{proposition}[theorem]{Proposition}
\newtheorem{lemma}[theorem]{Lemma}
\newtheorem{corollary}[theorem]{Corollary}
\newtheorem{observation}[theorem]{Observation}
\theoremstyle{definition}
\newtheorem{assumption}[theorem]{Assumption}

\newtheorem{remark}[theorem]{Remark}
\newtheorem{example}[theorem]{Example}

\DeclareMathOperator{\Cov}{Cov}
\DeclareMathOperator*{\argmin}{arg\,min}
\DeclareMathOperator{\tr}{tr}
\DeclareMathOperator{\spn}{span}
\DeclareMathOperator{\rank}{rank}
\DeclareMathOperator{\relerr}{relErr}
\newcommand{\A}{\mathcal{A}}
\newcommand{\R}{\mathbb{R}}
\newcommand{\E}{\mathbb{E}}
\newcommand{\Ltwo}{L^2(\rho)}
\newcommand{\snorm}[1]{\lVert#1\rVert}

\title{Beyond Gradient Flow: Identifiability and Recovery from Distribution Snapshots}
\author{
  Nam D. Nguyen$^{1,2,3,4}$ \quad Valeriya Malysheva$^{1,2,3,5}$ \\[0.6em]
  \normalfont\small
  \begin{tabular}[t]{@{}l@{}}
  $^{1}$VIB, Center for Molecular Neurology, Antwerp, Belgium\\
  $^{2}$VIB, Center for AI and Computational Biology, Leuven, Belgium\\
  $^{3}$Faculty of Pharmaceutical, Biomedical and Veterinary Sciences,\\
  \phantom{$^{3}$}University of Antwerp, Antwerp, Belgium\\
  $^{4}$Research Foundation -- Flanders (FWO), Brussels, Belgium\\
  $^{5}$Trinity Hall, University of Cambridge, Cambridge, UK\\[0.45em]
  \texttt{ducnam.nguyen@uantwerpen.be} \quad \texttt{valeriya.malysheva@uantwerpen.be}
  \end{tabular}
}
\iclrfinalcopy

\hypersetup{
  pdftitle={Beyond Gradient Flow: Identifiability and Recovery from Distribution Snapshots},
  pdfauthor={Nam D. Nguyen, Valeriya Malysheva},
  pdfsubject={Identifiability and estimation of solenoidal drift in autonomous SDEs from population snapshots},
  pdfkeywords={snapshot dynamics, identifiability, Fokker--Planck, Stein operator, solenoidal drift, weak-form estimation}
}

\begin{document}
\maketitle
\lhead{Preprint. Under review.}

\begin{abstract}
Inferring dynamics from snapshots of evolving distributions is fundamentally
underdetermined: the Fokker--Planck equation constrains the drift $F$ only
through its score-weighted divergence
$\nabla\cdot F+F\cdot\nabla\log\rho$, leaving a $\rho$-solenoidal gauge
invisible to any single-time constraint. Time-indexed transport formulations cannot resolve this
ambiguity: every admissible marginal path admits a curl-free explanation,
minimum-action reconstruction selects it, and marginal fit alone cannot
distinguish dynamically inequivalent explanations. Requiring one autonomous
field to explain several marginals instead makes part of the hidden circulation
visible as $\nabla\log\rho$ changes across marginals. Separating instantaneous Fokker--Planck source
constraints from the snapshot experiment, we show that the source constraints
identify the field modulo the kernel of a stacked score-weighted divergence
operator. For generic Gaussian shape variation, source constraints at $K\ge m$
time points in intrinsic dimension $m$ eliminate every polynomial gauge
direction, whereas finitely many density snapshots alone admit aliasing; we
give the obstruction explicitly. At a Gaussian anchor, for Sobolev smoothness \(s\) and \(n\) samples per time point, we derive a conditional lower rate \((nK)^{-2s/(2s+m+1)}\) for the tangent snapshot experiment, with a matching upper rate in a degreewise benchmark.
Strong-form fitting is non-orthogonal to score error and cannot be repaired by
spectral filtering. Instead, we estimate using smooth test functions while
retaining the known diffusion term, and derive a finite-sample bound that separates sampling error from fixed-grid quadrature bias. Planted-circulation experiments
confirm the predicted gauge contraction and expose a design tension between
cross-slice information and covariance-aware whitening.
\end{abstract}

\section{Introduction}\label{sec:intro}

Inferring dynamical systems from \emph{snapshots} of their time-evolving
distributions is a central problem in fields ranging from single-cell biology
to climate science. Suppose a latent state $Z_t$ evolves according to
\begin{equation}\label{eq:sde}
    dZ_t = F_t(Z_t)\,dt + \sqrt{2\sigma^2}\,dW_t,
\end{equation}
with $W_t$ a standard Brownian motion and $\sigma\ge0$ a \emph{known, fixed} diffusion
scale. Only independent samples from the marginals
$\rho_{t_1},\dots,\rho_{t_K}$ are observed, and the goal is to recover the
drift $F_t$. The problem is underdetermined: the Fokker--Planck equation
constrains only the gradient component of the weighted Helmholtz decomposition
$F_t=-\nabla\Phi_t+h_t$, with scalar potential $\Phi_t$ and $\nabla\!\cdot(\rho_th_t)=0$, leaving the
weighted-solenoidal field $h_t$ as a time-local gauge freedom.

At stationarity the gauge is precisely the probability current $J=\rho h$
separating equilibrium from nonequilibrium steady states: $h=0$ is detailed
balance, $h\ne0$ permits persistent circulation without changing the stationary
marginal. Such fluxes are widespread in living systems
\citep{battle2016broken,gnesotto2018broken}, motivate landscape--flux
descriptions of cell fate \citep{wang2008potential,kwon2005structure}, and
drive irreversible Langevin samplers
\citep{hwang1993accelerating,lelievre2013optimal,reybellet2015irreversible},
game dynamics \citep{balduzzi2018mechanics} and stochastic optimization
\citep{kunin2021limiting}. We ask whether it can be inferred when the dynamics
are unobserved.

Existing methods resolve the underdetermination by imposing a selection
principle --- transport cost, a reference process, a coupling, or an action
\citep{schiebinger2019optimal,tong2020trajectorynet,huguet2022manifold,%
bunne2022proximal,terpin2024learning,debortoli2021diffusion,shi2023,%
albergo2023building,lipman2023flow,tong2024,lee2025mmsfm,%
neklyudov2023action,benamou2000computational}. These select a representative
compatible with the marginals rather than identifying the circulation; within
an unrestricted \emph{time-indexed} class every admissible marginal path admits a
curl-free explanation (\S\ref{sec:blind}), so marginal fit cannot distinguish
true circulation from a gauge choice.

The situation changes under \emph{autonomy}, $F_t\equiv F$: one field must now explain
all observed marginals. A direction invisible under one density $\rho_t$ may become
visible as its score $\nabla\log\rho_t$ changes, so the persistent gauge
contracts to directions invisible across all source constraints. We formalize
this through Stein operators, ask what becomes structurally identifiable, how
the remaining near-gauge directions control statistical recovery, and how that
structure dictates an estimator.

\begin{enumerate}

\item \textbf{Source constraints and snapshot identifiability}
(\S\ref{sec:identifiability}).
We separate known \emph{sources}
and finitely many \emph{snapshots}. For a time-indexed drift, minimum-action
reconstruction selects the gradient representative (Cor.~\ref{cor:blind}). Under autonomy, the residual gauge under the source constraints consists of directions lying in every per-time Stein kernel and is contracted pointwise by changing
scores (Prop.~\ref{prop:count}). Snapshots identify strictly less: this shared
gauge is neither necessary nor sufficient there, and we give the aliasing
obstruction (Ex.~\ref{ex:alias}).

\item \textbf{Statistical limits and a conditional rate}
(\S\ref{sec:oracle}).
For generic Gaussian shape variation, source constraints at \(K\ge m\) time
points eliminate every finite-degree polynomial gauge direction in intrinsic
dimension \(m\) (Thm.~\ref{thm:gauge}). Linearizing the snapshot map around a
reference system, we derive a conditional lower rate for the tangent snapshot
experiment, with a matching upper rate in a degreewise benchmark
(Thm.~\ref{thm:oracle-rate}).

\item \textbf{A gauge-aware estimator, and a design tension}
(\S\ref{sec:estimator}, \S\ref{sec:exp}).
Strong-form estimation is non-orthogonal to score error
(Prop.~\ref{prop:notorth}) and spectral filtering cannot separate the resulting
bias from signal (Prop.~\ref{prop:filter}). A weak formulation removes the
nuisance and, under independent slices, yields moment equations with exact
block-tridiagonal covariance (Prop.~\ref{prop:gls}). Empirically we find that
the regime in which that covariance structure would pay is the regime in which
the target ceases to be identifiable (Obs.~\ref{obs:tension}).

\end{enumerate}

\textbf{Positioning.} Trajectory inference is usually posed as transport between temporal marginals
and closed by an optimal-transport, reference-process, flow-matching, or
action-based principle.
These construct dynamics consistent with prescribed distributions; we ask which
aspects of the dynamics those distributions determine. Consistency results
typically assume gradient drift \citep{lavenant2024towards} and action matching
explicitly targets a gradient representative \citep{neklyudov2023action};
rotational ambiguity from population snapshots was noted by
\citet{weinreb2018fundamental}. The closest identifiability result,
\citet{guan2024}, characterizes linear additive-noise SDEs through generalized
rotational symmetries of the initial law; we treat an infinite-dimensional
field class, condition on the observed marginal shapes rather than the initial
law, and study finite-sample recovery (App.~\ref{app:gauss}). Velocity,
pairing and lineage measurements supply Lagrangian information we deliberately
exclude (App.~\ref{app:caseB}). App.~\ref{app:related} gives the full
discussion.

\section{What can be identified from evolving marginals}
\label{sec:identifiability}
\label{sec:setup}

Marginal evolution does not reveal a vector field directly, only its effect on
probability mass. A Stein operator captures this effect. We use it to separate source constraints from the snapshot experiment and to characterize time-indexed blindness. Proofs are in App.~\ref{app:sec2}.

Let $\rho_t$ be the marginal of \eqref{eq:sde} and $s_t=\nabla\log\rho_t$ its
score. Dividing the Fokker--Planck equation by $\rho_t$ gives
\begin{equation}\label{eq:scoreform}
    \A_{\rho_t}F_t=b_t,
    \quad
    \A_\rho F:=\nabla\!\cdot F+F\cdot\nabla\log\rho,
    \quad
    b_t=-\partial_t\log\rho_t+\sigma^2\bigl(\nabla\!\cdot s_t+\snorm{s_t}^2\bigr).
\end{equation}
The operator $\A_\rho$ is the vector Stein operator associated with $\rho$;
we call $\A_\rho F$ the \emph{footprint} of $F$ on $\rho$. We fix $\sigma$
throughout; joint drift--diffusion identifiability is discussed in
App.~\ref{app:scope}. A vector field $h$ is invisible at $\rho$ exactly when its footprint vanishes,
$\ker\A_\rho=\{h:\nabla\!\cdot(\rho h)=0\}$; we call such fields
\emph{$\rho$-solenoidal}\footnote{Throughout, \emph{weighted-solenoidal} or
\emph{$\rho$-solenoidal} denotes
this exact mathematical condition; \emph{circulation} refers to its dynamical
interpretation, while \emph{antisymmetry} is used only in linear settings.}.

Two information levels must be distinguished. At the \emph{source-level}, one knows $(\rho_t,\partial_t\rho_t)$ at the times under consideration,
so \eqref{eq:scoreform} is an observed population constraint; a smooth complete
marginal path provides these constraints for every $t$. In the
\emph{snapshot experiment} one observes only independent samples from
$\rho_{t_1},\ldots,\rho_{t_K}$, so the temporal source is not directly
available.

\subsection{Time-indexed saturation and architectural blindness}
\label{sec:blind}

Optimal transport and related models fit a separate field $F_t$ at each time,
letting every slice choose its own representative of the single-slice gauge.

\begin{proposition}[Gradient saturation]
\label{prop:saturated}
Suppose $\{\rho_t\}$ is a smooth path of positive densities such that each
$\rho_t$ satisfies a Poincar\'e inequality and
$\partial_t\log\rho_t\in L^2(\rho_t)$. Then for every $t$ there is
$\phi_t\in H^1(\rho_t)/\mathbb R$ with
$\nabla\!\cdot(\rho_t\nabla\phi_t)=-\partial_t\rho_t$ weakly, so the path
solves the continuity equation driven by the curl-free velocity
$v_t=\nabla\phi_t$. The corresponding It\^o drift $F_t=v_t+\sigma^2s_t$ is also curl-free.
\end{proposition}

Zero circulation is therefore unfalsifiable from marginal fit within a
time-indexed gradient class, however strongly the underlying system rotates.

\begin{corollary}[Minimum-action projection]
\label{cor:blind}
At each $t$, among all velocities meeting the same marginal constraint, the
minimum-$L^2(\rho_t)$ action solution lies in
$(\ker\A_{\rho_t})^\perp=\overline{\{\nabla\phi\}}$ and has zero
$\rho_t$-solenoidal component.
\end{corollary}

Optimal transport is the closed-form instance: for Gaussian marginals the
Benamou--Brenier displacement velocity is a gradient on every segment, so its
solenoidal relative error equals $1$ at every sample size while the induced
marginal path matches the observed one exactly (Cor.~\ref{cor:otblind},
App.~\ref{app:sec2}). The same population floor holds for conditional flow matching with exact OT
couplings; Schr\"odinger bridges add only a gradient correction to the reference
drift and hence preserve its $\rho_t$-solenoidal projection
(Cor.~\ref{cor:fmsb}). Fig.~\ref{fig:concept} shows both halves of the
blindness: identical marginal transport and an error floor that no amount of
data removes.

\begin{figure}[tb]
\centering
\begin{minipage}[b]{0.47\linewidth}
\centering
\includegraphics[width=\linewidth]{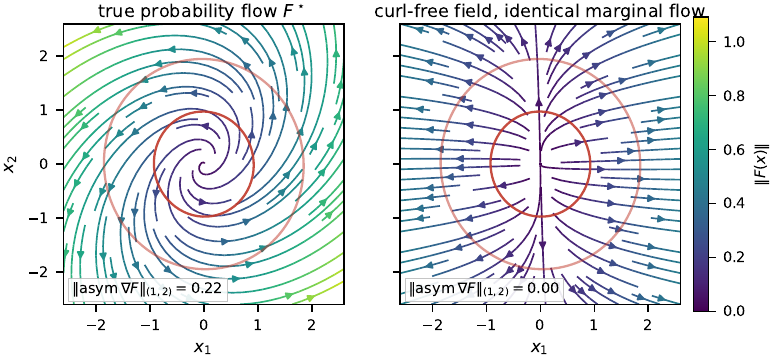}\\
(a)
\end{minipage}
\hspace{0.02\linewidth}
\begin{minipage}[b]{0.49\linewidth}
\centering
\includegraphics[width=\linewidth]{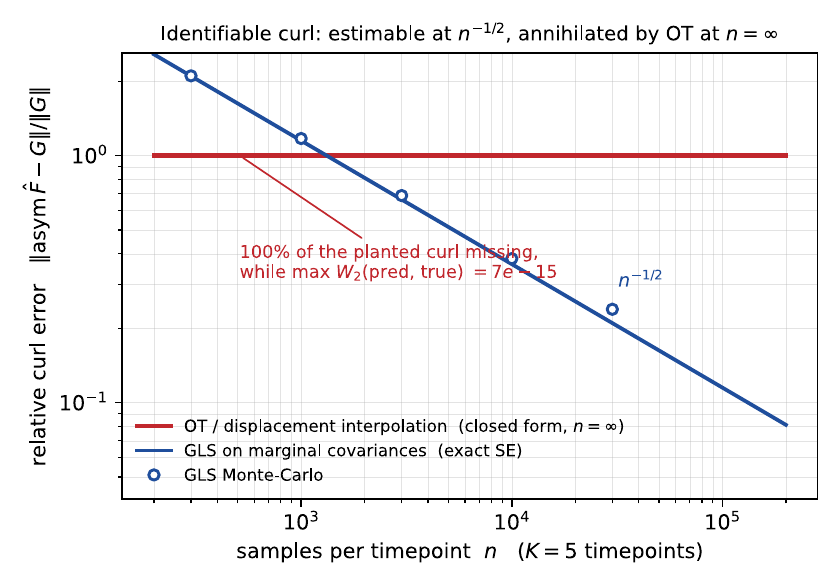}\\
(b)
\end{minipage}
\caption{\textbf{Identical marginal flow, different dynamics.}
(a) A flow with a planted rotational component and a curl-free transport
generate the same Gaussian marginals. (b) On the reference model of
Tab.~\ref{tab:ou}, linearized-GLS decays at the $n^{-1/2}$ rate whereas
displacement interpolation has solenoidal relative error $1$ for every $n$,
including $n=\infty$. Both match the observed marginals to machine precision,
so marginal fit cannot distinguish them.}
\label{fig:concept}
\end{figure}

This blindness is architectural: a time-conditioned curl head is
compensable slice by slice, making circulation unidentifiable even given oracle
nuisances (Prop.~\ref{prop:compensation}). Autonomy instead forces one
compensator to satisfy the constraints across all selected times.

\subsection{Autonomy, snapshot equivalence, and the residual gauge}
\label{sec:autonomy}
\label{sec:kernel}

This motivates an autonomy hypothesis at the representation level: the chosen
state is sufficiently informative that a single field can explain the entire
marginal evolution.
\begin{assumption}[Representation-level autonomy]
\label{asm:autonomy}
On the chosen $m$-dimensional state space $\mathcal M$, the observed evolution is generated by
a single time-independent It\^o drift
$F:\mathcal M\to T\mathcal M$.
\end{assumption}

\begin{proposition}[Source-constraint equivalence]
\label{prop:equiv}
Two autonomous fields satisfy the same source constraints
\eqref{eq:scoreform} iff their difference lies in
$\ker\A_K:=\bigcap_{k=1}^K\ker\A_{\rho_{t_k}}$, where
$\A_KF:=(\A_{\rho_{t_k}}F)_{k=1}^K$. For a continuum of source times the intersection is taken over the whole
interval; with the same initial law and forward uniqueness, this is equivalent
to generating the same complete marginal path.
\end{proposition}

Changing marginals can reveal directions hidden at a single slice: if
$\A_{\rho_{t_1}}h=0$ then $\A_{\rho_{t_k}}h=h\cdot(s_k-s_1)$, so cross-slice
score variation creates a footprint for otherwise invisible directions.

Finite snapshots are weaker. Their population observation map is nonlinear,
$\mathscr S_K(F,\rho_0)=(\rho^F_{t_1},\dots,\rho^F_{t_K})$, and its
\emph{fibres}, not $\ker\A_K$, are the observational equivalence classes; an
unknown initial law is a further nuisance. In a finite-dimensional
parameterization, full column rank of the derivative of a continuously
differentiable snapshot-statistic map at the truth gives local injectivity by
the inverse function theorem applied to a full-rank square submap, but neither
implies global injectivity nor follows from instantaneous score rank.

\begin{example}[Temporal aliasing]
\label{ex:alias}
Let $J=\left(\begin{smallmatrix}0&-1\\1&0\end{smallmatrix}\right)$,
$F_\omega(x)=(-aI+\omega J)x$ with $a>0$, and $\rho_0=\mathcal N(0,\Sigma_0)$
with $\Sigma_0$ anisotropic. Then
$\Sigma_t=e^{-2at}R_{\omega t}\Sigma_0R_{\omega t}^\top
+\tfrac{\sigma^2}{a}(1-e^{-2at})I$ with $R_\theta=e^{\theta J}$. At
$t_k=k\Delta$, the drifts $F_\omega$ and $F_{\omega+\pi/\Delta}$ give
\emph{identical} Gaussian snapshots, yet their difference $(\pi/\Delta)Jx$ is
\emph{not} in $\ker\A_{\rho_0}$.
\end{example}

Thus $\ker\A_K$ is neither sufficient for snapshot equivalence --- vanishing
footprints at the selected instants leave the path unconstrained in between ---
nor necessary, by Ex.~\ref{ex:alias}. It characterizes source-level
identifiability, whereas local snapshot identifiability is governed by
$D\mathscr S_K$ (\S\ref{sec:how-fast}).

Unlike the time-indexed model, autonomy is falsifiable. If one autonomous field
explains the slices but no autonomous gradient field does, then non-gradient
structure is required within the specified model. The same cross-slice
restriction therefore creates both falsifiability and additional
identifiability. 

Autonomy is a property of the chosen representation: failure of one field to
explain several slices may instead signal missing dynamical state rather than
genuinely time-varying dynamics. App.~\ref{app:stress} gives both cases at the
population level, including one where every stationary snapshot is compatible
with a zero-rotation drift although the underlying system circulates. We return
to this representation-level tradeoff in App.~\ref{app:scope}.

Although not the finite-snapshot equivalence class, $\ker\A_K$ is the \emph{residual
gauge} of the stacked source constraints and the structural benchmark below. Its geometry has a pointwise reduction.

\begin{proposition}[Residual fibre dimension]
\label{prop:count}
Define
$\mathcal C_x:=\operatorname{span}\{s_k(x)-s_1(x):k=2,\ldots,K\}^{\perp}
\subseteq T_x\mathcal M$
and
$r(x):=\rank[s_2(x)-s_1(x),\ldots,s_K(x)-s_1(x)]$.
Then
$\ker\A_K
=\{h:\A_{\rho_{t_1}}h=0,\;
h(x)\in\mathcal C_x\ \rho_{t_1}\text{-a.e.}\}$,
with $\dim\mathcal C_x=m-r(x)$.
\end{proposition}

The relevant quantity is the score-rank $r(x)$, not $K$ alone, and at
stationarity all constraints coincide. The count $m-r(x)$ is a \emph{fibre}
dimension only: the unknown is a square-integrable section, so even
low-dimensional fibres support infinitely many such fields. A
fixed finite-dimensional model class replaces this infinite-dimensional freedom
by a finite-dimensional rank problem, though nonidentifiability can remain
through null directions or temporal aliasing (App.~\ref{app:gauss}).

\section{Source identifiability and tangent statistical limits}
\label{sec:oracle}

Section~\ref{sec:identifiability} separated the stacked \emph{source
constraints} from the finite snapshot map $\mathscr S_K$. We now characterize
the Gaussian source gauge and the statistical recovery of locally visible
directions in the tangent snapshot experiment. The lower bound is matched in a
degreewise benchmark; transfer to sampled snapshots is not claimed. We work in
intrinsic dimension $m\ge2$. Proofs are in
Apps.~\ref{app:hermite}--\ref{app:gauss}, with ambient reduction, manifold concentration and the stationary specialization in
Apps.~\ref{app:intrinsic}--\ref{app:temporal-ou}.

\subsection{Gaussian shape change contracts the weighted gauge}
\label{sec:how-much}
\label{sec:gaussian-gauge}

Fix comparison slices $\bar\rho_{t_1},\dots,\bar\rho_{t_K}$ and an anchor
$t_\star$ such that, after centering and whitening,
$\rho_\star:=\bar\rho_{t_\star}=\mathcal N(0,I_m)$; in \S\ref{sec:how-fast}
these are the slices of the reference system at the observation times. Let
$\{\Psi_\beta\}$ be the normalized probabilists' Hermite basis,
$\mathcal P_q$ the degree-$q$ chaos,
$D_q:=\dim\mathcal P_q=\binom{q+m-1}{m-1}$,
$W_q:=\mathcal P_{q-1}\otimes\R^m$, and
$S_q:=\ker(\A_{\rho_\star}|_{W_q})$.

\begin{proposition}[Anchor-weighted gauge]
\label{prop:hermite}
$\A_{\rho_\star}(\Psi_\beta e_i)=-\sqrt{\beta_i+1}\,\Psi_{\beta+e_i}$ for every
multi-index $\beta$ and coordinate $i$. Hence $\A_{\rho_\star}:W_q\to\mathcal
P_q$ is onto and, for $q\ge2$, $\dim S_q=mD_{q-1}-D_q$, with
$\dim S_q/\dim W_q\to(m-1)/m$ as $q\to\infty$.
\end{proposition}

Thus a single Gaussian anchor slice leaves most high-degree directions in
its weighted gauge. With $\bar s_{t_k}:=\nabla\log\bar\rho_{t_k}$ and the
stacked reference operator $\bar\A_Kh:=(\A_{\bar\rho_{t_k}}h)_k$, the
cross-slice footprint of $h\in S_q$ is
$T_{K,q}h:=(h\cdot(\bar s_{t_k}-s_\star))_{k\ne\star}$, and
$\ker T_{K,q}=S_q\cap\ker\bar\A_K$.

For an explicit criterion, suppose further that the comparison slices
are Gaussian with a common mean. After a common translation and whitening, write
$\bar\rho_{t_k}=\mathcal N(0,\Sigma_k)$ with $\Sigma^\star=I_m$, and define
the precision contrasts $\Xi_k:=\Sigma_k^{-1}-I_m$. Then
$T_{K,q}h=-(h(x)^\top\Xi_kx)_{k\ne\star}$.

\begin{theorem}[Gauge contraction by shape change]
\label{thm:gauge}
Let $\Xi_1,\dots,\Xi_r$ be the distinct precision contrasts among the
comparison slices, and call the shapes \emph{shape-rich} if
$\spn\{x,\Xi_1x,\dots,\Xi_rx\}=\R^m$ for $x$ in a nonempty open set. With
$\mathcal S_{\le Q}:=\bigoplus_{2\le q\le Q}S_q$, shape richness implies
$\mathcal S_{\le Q}\cap\ker\bar\A_K=\{0\}$ for every $Q\ge2$.
Generically, shape richness holds iff $r\ge m-1$, and this threshold is sharp:
for $r=m-2$ the nonzero field
$w(x):=\star(x\wedge\Xi_1x\wedge\cdots\wedge\Xi_{m-2}x)$ lies in
$\mathcal S_{\le m}\cap\ker\bar\A_K$. Hence, with distinct generic comparison
shapes, $K\ge m$ is necessary and sufficient to eliminate every finite-degree
polynomial gauge.
\end{theorem}

Thm.~\ref{thm:gauge} is proved in App.~\ref{app:gauge-contraction}. Below the
threshold the residual polynomial gauge can remain macroscopic: for
$m=3$ it contains the Euler rigid-body field $x\times\Xi_1x$, and its degreewise
size is quantified in App.~\ref{app:below}. More generally, with
$K\le m-1$ generic shapes, the first surviving polynomial gauge appears at
degree $K$ (App.~\ref{app:below}). At $q=2$, $S_2=\{Gx:G^\top=-G\}$,
and invisibility is equivalent to $[G,\Sigma_k^{-1}]=0$ for every $k$. Thus one
generic comparison shape eliminates the linear gauge; the threshold $K\ge m$ is
imposed by higher polynomial degrees (cf.~\citet{guan2024};
App.~\ref{app:gauss}).

\subsection{A conditional tangent benchmark}
\label{sec:how-fast}

Visibility to the source constraints does not determine how strongly a direction
appears in snapshots. Linearize around a known reference
$(\bar F,\bar\rho_t)$ whose observed slices are those of
\S\ref{sec:how-much}, with anchor $\bar\rho_{t_\star}=\rho_\star$ and fixed
initial law. For $F^\delta=\bar F+\delta g$, write
$\rho_t^\delta=\bar\rho_t(1+\delta u_t)+o(\delta)$, defining the
relative-density perturbation $u_t$. Let $\mathcal L_t$ be its linearized
Fokker--Planck operator and $\mathcal U(t,s)$ the propagator of
$\partial_tu=\mathcal L_tu$ (App.~\ref{app:transient}). Then
$\partial_tu_t=\mathcal L_tu_t-\A_{\bar\rho_t}g$ with $u_0=0$, and
\begin{equation}\label{eq:tangent}
u_{t_k}=-(\mathcal B_Kg)_k,\qquad
(\mathcal B_Kg)_k:=\int_0^{t_k}\mathcal U(t_k,s)\,
\A_{\bar\rho_s}g\,ds.
\end{equation}
Thus $\mathcal B_K=-D\mathscr S_K(\bar F)$. Unlike $T_{K,q}$,
$\mathcal B_K$ integrates source footprints along the reference path, so
$\ker\mathcal B_K$ and $\ker\bar\A_K$ need not be nested: between-slice
visibility and Duhamel cancellation can both occur
(App.~\ref{app:transient}). An unknown initial law would add a propagated
nuisance term.

Let $\mathsf M_K$ map the stacked tangent densities to the retained snapshot
statistics, let $\Gamma_K$ be the covariance of their $\sqrt n$-scaled
fluctuations, and set
$\widetilde{\mathcal B}_K:=\Gamma_K^{-1/2}\mathsf M_K\mathcal B_K$. The whitened
\emph{tangent snapshot experiment} observes $Y=\widetilde{\mathcal B}_Kg+n^{-1/2}\xi$
with $\E\xi=0$ and $\Cov(\xi)=I$. At the anchor, write the degree-$q$
anchor-weighted Helmholtz split $W_q=G_q\oplus S_q$, with
$G_q:=S_q^\perp\cap W_q=\nabla\mathcal P_q$, and decompose
$g_q=g_{q,\mathrm{grad}}+h_q$ accordingly. Treating the gradient part as an
unrestricted nuisance and profiling it out degree by degree gives
\begin{equation}\label{eq:Jsol}
    \mathcal J_{\mathrm{sol},q}
    :=\bigl(\widetilde{\mathcal B}_K|_{S_q}\bigr)^\ast
      (I-\Pi_{q,\mathrm{grad}})
      \bigl(\widetilde{\mathcal B}_K|_{S_q}\bigr),
\end{equation}
where $\Pi_{q,\mathrm{grad}}$ projects onto the range of
$\widetilde{\mathcal B}_K|_{G_q}$. Each operator in \eqref{eq:Jsol} is
self-adjoint and nonnegative on $S_q$; let
$\kappa_{q,1}^2\ge\cdots\ge\kappa_{q,r_q}^2\ge0$ be its eigenvalues, zeros
included, so that $r_q=\dim S_q\asymp q^{m-1}$ by Prop.~\ref{prop:hermite}.

To align tangent with source identifiability we assume, on the degrees
considered, the \emph{faithfulness} condition
$\ker\mathcal J_{\mathrm{sol},q}\subseteq\ker T_{K,q}$: tangent propagation
creates no invisible direction beyond the source gauge. Under the shape-rich
setting of Thm.~\ref{thm:gauge} it makes $\mathcal J_{\mathrm{sol},q}$
injective on $S_q$, so all $r_q$ eigenvalues are positive and the tangent
experiment targets exactly the circulation made visible by Thm.~\ref{thm:gauge}.
The rate theorem below does not use it.

\begin{assumption}[Block comparability]
\label{asm:comparability}
Let $\mathcal J_{\mathrm{sol},\le Q}$ be the information on
$\mathcal S_{\le Q}$ obtained by profiling the gradient directions of degree at
most $Q$ jointly, with all components above degree $Q$ set to zero, and let
$\mathcal J_Q^{\mathrm{blk}}:=\bigoplus_{q=2}^Q\mathcal J_{\mathrm{sol},q}$.
There is $C<\infty$, independent of $Q$, with
$\mathcal J_{\mathrm{sol},\le Q}\preceq C\,\mathcal J_Q^{\mathrm{blk}}$.
\end{assumption}

Joint profiling therefore cannot raise information beyond the degreewise
benchmark by more than a constant, which suffices for the lower bound; the
reverse comparison can fail (App.~\ref{app:joint}) and is not used. Let $\{e_{q,j}\}_{j=1}^{r_q}$ be an orthonormal eigenbasis of
$\mathcal J_{\mathrm{sol},q}$ and write
$h_q=\sum_{j=1}^{r_q}\theta_{q,j}e_{q,j}$. In these eigencoordinates, the block-diagonal \emph{degreewise benchmark} is
\begin{equation}\label{eq:sequence}
    Z_{q,j}=\kappa_{q,j}\theta_{q,j}+n^{-1/2}\xi_{q,j},
    \qquad
    \sum_{q\ge2}(1+q)^{2s}\|\theta_q\|^2\le R^2,
\end{equation}
which defines the Hermite--Sobolev ball $\mathcal F_m^s(R)$ for $h=\sum_{q\ge2}h_q$ (conventions in
App.~\ref{app:rate-proofs}). The dimension law $r_q\asymp q^{m-1}$ is proved;
the information scale is not.

\begin{assumption}[First-moment information scale]
\label{asm:scale}
Uniformly over the degrees and designs considered,
$\bar\mu_q:=r_q^{-1}\sum_j\kappa_{q,j}^2\asymp K/q$.
\end{assumption}

Here the factor $K$ describes temporal excitation, not merely the number of
snapshots, and does not follow from the structural threshold $K\ge m$; weakly
excited designs can violate it (\S\ref{sec:exp-tension}), and
App.~\ref{app:scale} sketches where the scale comes from. By Markov's
inequality, Asm.~\ref{asm:scale} implies that a constant fraction of degree-$q$
directions have information $O(K/q)$, which drives the lower bound. A matching
upper bound requires additional control near zero.

\begin{assumption}[Uniform thin lower tail]
\label{asm:thin-tail}
There exist $C'<\infty$ and $\gamma>1$, independent of $q,n,K$, such that
$\mathcal N_q(t):=r_q^{-1}\#\{j:\kappa_{q,j}^2\le t\bar\mu_q\}\le C't^\gamma$ for
$0<t\le1$.
\end{assumption}

This requires no degree-uniform spectral gap, excludes zero eigenvalues, and,
together with Asm.~\ref{asm:scale}, yields
$\sum_j\kappa_{q,j}^{-2}\asymp q^m/K$ (App.~\ref{app:spectral}).

\begin{theorem}[Conditional tangent benchmark]
\label{thm:oracle-rate}
Suppose the whitened noise in the tangent snapshot experiment is standard
Gaussian. Then Asm.~\ref{asm:scale} and Asm.~\ref{asm:comparability}
imply that in that experiment, as $nK\to\infty$,
\[
    \inf_{\widehat h}\sup_{h\in\mathcal F_m^s(R)}
    \E\bigl\|\widehat h-h\bigr\|_{L^2(\rho_\star)}^2
    \;\gtrsim\;(nK)^{-2s/(2s+m+1)} .
\]
If Asm.~\ref{asm:thin-tail} also holds, then in the degreewise benchmark
\eqref{eq:sequence}, for any noise $\xi_{q,j}$ there with $\E\xi=0$ and
$\Cov(\xi)=I$, a linear estimator with degree cutoff
$Q_\star\asymp(nK)^{1/(2s+m+1)}$ attains the matching rate. Constants may depend
on $s,m,R$ and the hypothesis constants; at $\gamma=1$ the upper bound incurs a
logarithmic factor.
\end{theorem}

Indeed, $r_q\asymp q^{m-1}$ directions with average inverse information of
order $q/K$ give cumulative variance $\asymp Q^{m+1}/(nK)$, which balances the
Sobolev bias $R^2Q^{-2s}$ at the stated rate (App.~\ref{app:rate-proofs}).

\begin{remark}[What remains for sampled snapshots]
\label{rem:oracle-gap}
Thm.~\ref{thm:oracle-rate} gives the lower bound for the tangent snapshot
experiment and the matching upper bound for the degreewise benchmark. By our
argument, extending the upper bound to the full tangent experiment would need
the reverse block comparison and control of omitted-degree leakage;
transferring either result to nonlinear sampled snapshots would further need
uniform control of the linearization remainder and of the sampling likelihood
(App.~\ref{app:transfer}). We do not prove these transfers; the estimator of
\S\ref{sec:estimator} instead has its own finite-basis guarantee
(Thm.~\ref{thm:weak-bound}).
\end{remark}

\section{Estimation from snapshot moments}\label{sec:estimator}

Sections~\ref{sec:identifiability}--\ref{sec:oracle} characterize the identifiable
target and its local statistical difficulty. A direct implementation of
\eqref{eq:scoreform} requires estimating the score and the source $b_t$,
including the temporal log-density derivative, before fitting the field, so
plug-in nuisance error enters at first order.

\begin{proposition}[Strong-form nuisance amplification]
\label{prop:notorth}
Writing the strong-form moment against test functions $\chi_p$ as
$m_p=\E[r_\theta\chi_p]$, its derivative with respect to the score nuisance is
nonzero. Along a solenoidal eigendirection of the profiled strong-form Hessian (Prop.~\ref{prop:hessian})
with curvature $\lambda>0$, the corresponding component of the first-order
coefficient bias is amplified by $\lambda^{-1}$.
\end{proposition}

Nuisance error is amplified along weakly visible solenoidal
directions. Spectral filtering cannot fix this: it attenuates signal and
first-order nuisance bias together (Prop.~\ref{prop:filter}). A coefficient-side
Riesz correction likewise leaves a nonzero nuisance derivative, so it is only a
one-step estimating-equation correction, not a Neyman orthogonalization
(App.~\ref{app:strong}). We therefore remove the score nuisance altogether,
retaining only a covariance-based strength audit from the strong-form geometry.

\begin{proposition}[Covariance-only strength audit]
\label{prop:hessian}
For centered Gaussian slices and a linear field $F(x)=Hx$, the unweighted
population loss has quadratic part
\(
\frac1K\sum_{k=1}^K
\operatorname{vec}(H)^\top
\bigl(
\Sigma_{t_k}\otimes\Sigma_{t_k}^{-1}
+\mathsf C_m
\bigr)
\operatorname{vec}(H),
\)
with $\operatorname{vec}$ column-stacked and $\mathsf C_m$ the commutation matrix.
Hence the profiled solenoidal curvature --- the corresponding Schur complement
after profiling the gradient part --- depends only on the slice covariances; so
does any version weighted deterministically from them.
\end{proposition}

Thus, in the Gaussian linear model, zero eigenvalues mark source-level gauge
directions and small positive ones weakly visible directions. We use this
covariance-only geometry only as a second-moment \emph{design} audit: beyond
Gaussian slices, strong-form curvature depends on higher moments, while
feasible-GLS information generally involves fourth moments
(App.~\ref{app:rate-proofs}). Estimation instead uses score-free weak moments.

\subsection{Nuisance-free weak moments}\label{sec:weak}

For any smooth test function $\psi$ in the generator domain, It\^o's formula
gives $\frac{d}{dt}\E_{\rho_t}[\psi]=\E_{\rho_t}[F\cdot\nabla\psi+\sigma^2
\Delta\psi]$, and the trapezoidal rule between adjacent snapshots yields
\begin{equation}\label{eq:weak}
    \E_{k+1}[\psi_j]-\E_k[\psi_j]
    =\frac{\Delta_k}{2}(\E_k+\E_{k+1})
     \bigl[F\cdot\nabla\psi_j+\sigma^2\Delta\psi_j\bigr]+O(\Delta_k^3),
\end{equation}
with $\Delta_k=t_{k+1}-t_k$ and $\E_k:=\E_{\rho_{t_k}}$. The diffusion
term requires only an empirical moment, so \eqref{eq:weak}
estimates the It\^o drift $F$: no score, density estimate, or pointwise source. With
$F_a(x)=\sum_{p=1}^Pa_p\phi_p(x)$ the system is linear in $a$.

\begin{proposition}[Block-tridiagonal moment noise under independent slices]
\label{prop:gls}
Let the slices be sampled independently with sizes $n_k$, let
$\Psi=(\psi_1,\dots,\psi_J)^\top$, write
$f_a:=(F_a\cdot\nabla\psi_j+\sigma^2\Delta\psi_j)_j$ and
$u_k^{\mp}(x;a):=\mp\Psi(x)-\tfrac{\Delta_k}{2}f_a(x)$. Then
$V_{kk}(a)=\Cov_k(u_k^-)/n_k+\Cov_{k+1}(u_k^+)/n_{k+1}$,
$V_{k,k+1}(a)=\Cov_{k+1}(u_k^+,u_{k+1}^-)/n_{k+1}$, and $V_{k\ell}=0$ for
$|k-\ell|>1$. The stacked covariance is block tridiagonal, adjacent blocks
sharing a slice with opposite sign.
\end{proposition}

The covariance depends on the field through $f_a$, so feasible GLS needs a
pilot (App.~\ref{app:weak-impl}). For the implemented finite representation
this yields a snapshot-only guarantee of a different kind from
Thm.~\ref{thm:oracle-rate}: Thm.~\ref{thm:weak-bound} bounds any linear
readout by a sampling term of order $n_{\min}^{-1/2}$ plus a non-stochastic
quadrature residual, both amplified by the smallest whitened singular value,
so increasing $n$ leaves a floor (App.~\ref{app:weak-proof}).
Independence across slices is a hypothesis,
not a formality: with lineage-tracked particles the true covariance is dominated
by blocks this surrogate sets to zero (App.~\ref{app:caseB}).

Stacking \eqref{eq:weak} over intervals $k$ and probes $j$, with $\widehat\E_k$
the empirical mean over slice $k$, gives the response and design
\[
    \widehat b_{kj}:=\widehat\E_{k+1}[\psi_j]-\widehat\E_k[\psi_j]
      -\tfrac{\Delta_k}{2}(\widehat\E_k+\widehat\E_{k+1})[\sigma^2\Delta\psi_j],
    \qquad
    \widehat X_{kj,p}:=\tfrac{\Delta_k}{2}(\widehat\E_k+\widehat\E_{k+1})
      [\phi_p\cdot\nabla\psi_j],
\]
and $\widehat V$ is the covariance of Prop.~\ref{prop:gls} at a pilot fit. The
estimator is
$\widehat a_\zeta\in\argmin_{a\in\mathcal A_{\rm vis}}
\|\widehat V^{-1/2}(\widehat b-\widehat Xa)\|_2^2+\zeta\|a\|_2^2$, with
$\mathcal A_{\rm vis}$ excluding the exact snapshot gauge, removed rather than
weakly penalized so it carries neither an estimate nor spurious variance. We use
the ridge penalty $\mathcal R(a)=\|a\|_2^2$ (App.~\ref{app:selection}) or a
truncated-SVD cutoff of the whitened design (App.~\ref{app:weak-impl}). The local residual in \eqref{eq:weak} is
$O(\Delta_k^3)$ whereas the corresponding design row is $O(\Delta_k)$, so under
stable inversion fixed gaps leave the usual $O(\Delta_{\max}^2)$
coefficient-level quadrature floor, which increasing $n$ does not remove (App.~\ref{app:weak-impl}).

\subsection{Representation and stabilization}\label{sec:repr}

Three ingredients follow: GLS whitening, exact gauge exclusion before fitting,
and a representation --- a theory-aligned frame or fixed features --- solved
linearly rather than by descent. Regularization is then selected for the
solenoidal target, not for held-out moment fit.

Fit-based selection is aimed at the wrong target: the regularization strength
minimizing moment-prediction risk and the one minimizing the risk of a linear
functional of the field have different minimizers in general, with the gap
governed by the near-gauge spectrum (Prop.~\ref{prop:misalign},
App.~\ref{app:selection}).

Because whitening makes the response law exactly $\mathcal N(0,I)$, the
solenoidal risk is estimable without ground truth by perturbing the whitened
response under its known law; its failure mode is documented in
App.~\ref{app:selection}.

\section{Experiments}\label{sec:exp}

Because marginal agreement cannot validate an ambiguous component, we use
planted dynamics with known identifiable part and exact gauge, matched across
data, seeds, and readouts. We report solenoidal relative error
$\relerr:=\|\widehat h-h^\star\|/\|h^\star\|$, so $\relerr=1$ for a
zero-circulation estimate. These finite-dimensional experiments test the
structural predictions and estimator, not the rate exponent of
Thm.~\ref{thm:oracle-rate}. Protocol and receipts are in
Apps.~\ref{app:log}--\ref{app:repro}.

\paragraph{Structural predictions.}\label{sec:exp-theory}
For a planted rotation of norm $0.5$, exactly matched marginals yield
antisymmetric displacement at the $10^{-16}$ level and gradient-only
$\relerr=1.0000$, illustrating Cor.~\ref{cor:otblind}
(Fig.~\ref{fig:concept}b). An isotropic-block rotation leaves all marginals
unchanged to machine precision (Prop.~\ref{prop:symmetry}); operator checks are
in App.~\ref{app:verification}.

\subsection{Recovery from weak moments}\label{sec:exp-estimator}

Our main benchmark is an autonomous planted-curl Ornstein--Uhlenbeck system
with ambient $d=12$, intrinsic $m=4$, $K=5$ snapshots, and diffusion $\sigma^2$ equal to
the relaxation rate $\nu$ (Tab.~\ref{tab:ou}). Slices are drawn independently, as
Prop.~\ref{prop:gls} requires. The single-draw columns at $n=40$k and $200$k
are consistent with the parametric $n^{-1/2}$ benchmark, but each is one draw
from a heavy-tailed seed distribution, so we do not read them as a measured
rate; the pre-estimation audit of Prop.~\ref{prop:hessian} correctly
refuses the $n=4$k design, predicting $\relerr>1$ against a realized $3.70$.

\begin{table}[t]
\centering
\caption{\textbf{Planted-curl OU benchmark} ($d{=}12$, $m{=}4$, $K{=}5$,
$\sigma^2{=}\nu{=}1$). Entries are $\relerr$. Only weak rows are nuisance-free:
the parametric reference is given the symmetric--antisymmetric tie, while
strong-form rows use the exact score and temporal source. Multi-seed entries are
mean\,$\pm$\,s.d.; the four snapshot-only arms share identical draws over $12$
paired seeds, whereas strong-form rows use a separate $8$-seed set. See
App.~\ref{app:log} for single-seed and seed-spread caveats and solver variants.}
\label{tab:ou}
\begin{tabular}{@{}llccc@{}}
\toprule
estimator & ingredients & $n{=}40$k & $n{=}200$k & $n{=}40$k, $12$ seeds\\
\midrule
gradient-only / displacement & --- & $1.000$ & $1.000$ & ---\\
weak, unweighted & moments & --- & --- & $0.734\pm0.157$\\
weak, diagonal weights & moments & $0.893$ & $0.303$ & $0.782\pm0.199$\\
weak, \textsf{GLS} & moments & $\mathbf{0.651}$ & $\mathbf{0.282}$ & $0.751\pm0.260$\\
parametric reference & exact linear model & $0.511$ & $0.179$ & $0.308\pm0.088$\\
\midrule
strong form, oracle nuisances & exact score \& source & --- & --- & $0.535\pm0.177$\\
strong form, $10\%$ score error & corrupted score & --- & --- & $2.461\pm0.689$\\
\bottomrule
\end{tabular}

\end{table}

\textbf{Whitening.} We find no point-estimation gain from weighting of any kind
here. Paired differences are reported as mean\,$\pm$\,standard error over seeds,
and $t$ is the paired $t$-statistic, their ratio. Over the twelve paired seeds diagonal-to-\textsf{GLS} is
$+0.031\pm0.036$ ($t{=}+0.88$), unweighted-to-\textsf{GLS} is $-0.017\pm0.051$
($t{=}-0.32$), and unweighted is marginally better than diagonal
($-0.048\pm0.026$, $t{=}-1.82$). Calibration is already good without whitening
($1.03\times$ diagonal, $0.87\times$ \textsf{GLS}). We therefore report
Prop.~\ref{prop:gls} as correct algebra for the independent-slice model, not a
demonstrated gain; \S\ref{sec:exp-tension} says why.

\textbf{The cost of not knowing the model class.} The parametric reference,
told the exact linear tie between symmetric and antisymmetric coefficients,
beats every nuisance-free arm: $0.308$ against $0.751$ for \textsf{GLS}, paired
$+0.443\pm0.075$ ($t{=}+5.87$). The price of model-class agnosticism is therefore substantial here.

\subsection{A design tension}\label{sec:exp-tension}

Prop.~\ref{prop:gls} gives whitening its largest potential efficiency gain when
adjacent moment equations are strongly correlated through shared snapshots,
which favors short gaps. Thm.~\ref{thm:gauge} instead requires enough
cross-slice shape variation to expose circulation, which on this system favors
longer gaps.

\begin{observation}[Efficiency and identifiability compete for the same knob]
\label{obs:tension}
Compressing the five snapshot times sixfold, from a span of $1.45$ to $0.20$
while holding everything else fixed, raises solenoidal error from $0.78$ to
$30.9\pm18.3$ with diagonal weights and $25.9\pm13.5$ with \textsf{GLS},
roughly a $40$-fold increase and far above the no-information level. The
whitening contrast remains null there ($t{=}+1.62$ over six paired seeds).
\end{observation}

The collapse is consistent with the mechanism behind
Thm.~\ref{thm:gauge}: compressing time suppresses the covariance contrasts that
expose circulation, driving the planted component toward the near-gauge
directions diagnosed by Prop.~\ref{prop:hessian}. In this design, the regime
that strengthens the covariance exploited by whitening weakens
the shape variation needed for recovery. We treat this as a limitation rather
than a general law: we have not characterized when both requirements can hold
simultaneously.

\subsection{Ablations}\label{sec:exp-ablations}

All ablations are reported in full in the appendices. With exact score and
source, a closed-form strong-form solve recovers the planted circulation;
fitting the score raises relative error to $0.862$ and fitting both nuisances
to $2.19$. A probe-wise decomposition attributes $38\%$ amplitude bias to the
temporal log-density value channel versus $7\%$ to the score channel; denoising
score matching supplies $\nabla\log\rho_t$ but not the additive normalization
of $\log\rho_t$, which \eqref{eq:weak} eliminates
(App.~\ref{app:strong}). On the nonlinear system of App.~\ref{app:caseB}, no
snapshot-only recovery remains once slices are drawn independently: curl-channel
error is $1.65$ (diagonal) and $1.01$ (\textsf{GLS}) against a no-information
level of $1.0$. The recovery under lineage tracking therefore came from particle
identity rather than snapshot moments.

Polynomial frames scale as $P=504$, $5{,}460$, $81{,}900$ at $d=6,12,25$,
with the $O(P^3)$ whitening Cholesky as the bottleneck. Random features reach
parity with the correctly specified degree-three frame at $m=6$ ($t{=}0.49$,
six seeds), while a random rotation of all $12$ features leaves recovery
unchanged and the audit recovers the $m=4$ tangent space to $2.58^\circ$.
Noise, parameterization, conditioning, latent-regime, and autonomy stress tests
are in Apps.~\ref{app:noise}, \ref{app:param}, \ref{app:conditioning},
\ref{app:robustness} and \ref{app:stress}.

\section{Discussion}\label{sec:discussion}

Snapshot data supports a smaller estimand than the literature implicitly
assumes, parts of its visibility can be diagnosed from the snapshot geometry
before fitting, and the resulting estimator is a whitened linear solve rather
than a trained field. Reporting zero circulation is a property of a
parameterization, not a finding: within a time-indexed class it is unfalsifiable
from marginal fit (Prop.~\ref{prop:saturated}) and least-action reconstruction
actively selects it (Cor.~\ref{cor:blind}). Marginal metrics therefore cannot
validate non-gradient structure; planted-circulation benchmarks can. Autonomy
converts an untestable model into a testable one, but it is an assumption about
the chosen state representation rather than necessarily about the underlying
biology, and App.~\ref{app:stress} shows it can fail invisibly. The price is
explicit: the rate is governed by intrinsic dimension $m$, and $K\ge m$ generic
shapes eliminate all finite-degree polynomial gauge directions at the source
level (Thm.~\ref{thm:gauge}); fewer shapes can already identify lower-degree
components. Assumptions~\ref{asm:comparability}--\ref{asm:thin-tail}
and the faithfulness condition of \S\ref{sec:how-fast} remain unproved;
App.~\ref{app:ledger} tags every claim.

The current implementation is computationally limited by the intrinsic
representation dimension: polynomial frames grow combinatorially in $m$, and
dense whitening has $O(P^3)$ Cholesky cost. Thus the exact-solve estimator is
most practical when the dynamics admit a low-dimensional representation or a
compact feature basis, consistent with the intrinsic-dimensional formulation of
the theory. All validation here is synthetic. We therefore do not claim a
biological benchmark result or a neural-baseline leaderboard; methods supplied
with velocity, pairing, or lineage information answer a strictly richer
observation problem, while snapshot-only neural transport baselines remain
useful empirical comparisons for future work.

\section*{Reproducibility statement}

All theoretical claims are stated with their hypotheses in the main text and
proved in the appendices; App.~\ref{app:ledger} tabulates every numbered claim
against its status, distinguishing results proved as stated, results proved
conditional on stated assumptions, results with numerical support only, and
open problems. Assumptions that remain unproved are identified as such at the
point of use, and Rem.~\ref{rem:oracle-gap} states explicitly what
Thm.~\ref{thm:oracle-rate} does and does not establish.

All experiments are synthetic with closed-form ground truth. The supplementary
material contains the estimator library, the benchmark suite, every experiment
script, and machine-readable JSON receipts from which each number and table in
this paper is generated; App.~\ref{app:repro} maps each reported quantity to
its receipt and generating command, and the weak-form path runs on NumPy and
SciPy alone, while the strong-form and neural comparison arms also require PyTorch. Gates were fixed before each run. App.~\ref{app:log} records the
negative results and every retracted or superseded claim, including two
corrections that changed reported numbers, with superseded intermediate values
retained in the receipts rather than deleted. The complete suite runs on a
single CPU core.

\section*{AI use statement}

Large language models (Claude, Anthropic, and ChatGPT, OpenAI) were used in
preparing and conducting parts of this work. Their roles are disclosed below
by subtask, following the ICLR 2027 AI Policy for Authors.

\textbf{Manuscript preparation.}
The models revised the experiments and discussion sections, drafted appendices
from existing proof notes and experimental receipts, repaired cross-references,
and audited reported numbers against the machine-readable receipts in the
supplement. They also assisted with a feedback-driven algebra and exposition
audit, including corrections to the Gaussian fourth-moment calculation, the
finite-basis proof, and scope statements.

\textbf{Conceptual and theoretical development.}
The models assisted in developing and refining the conceptual and theoretical
framework of the work, including stress-testing the formulation of the
identifiability problem and proposing or refining hypotheses and intermediate
conjectures considered during development. The authors evaluated these
suggestions and determined the final framework, assumptions, and claims.

\textbf{Mathematical claims and proofs.}
The models assisted in checking and refining mathematical claims, derivations,
and proof arguments, and in drafting portions of proofs from the authors'
existing arguments and notes. This included identifying gaps or errors and
suggesting corrections during the proof audit. The authors checked the final
statements and proofs and take responsibility for their correctness.

\textbf{Methodology, implementation, and interpretation.}
The models provided feedback on experimental methodology and benchmark design
and assisted with implementation, diagnostic code, and interpretation of the
resulting experiments. In particular, a model identified two defects in the
shipped experimental code---an omitted diffusion term in the weak-form response
and a lineage-tracked sampling design inconsistent with the independence
hypothesis of Prop.~\ref{prop:gls}---wrote the corrections and diagnostic
scripts that quantified them, and re-ran the affected benchmarks. The resulting
retractions in App.~\ref{app:log} and the design tension reported in
Obs.~\ref{obs:tension} follow from those runs. All corrections, re-runs, and
interpretations were verified by the authors against the receipts.

\textbf{Synthetic data.}
The synthetic data used in the experimental benchmarks were generated by
simulation code whose design, implementation, or debugging was assisted by the
models. The authors verified the data-generating mechanisms and experimental
configurations used for the reported results.

\textbf{Literature retrieval and bibliography.}
The models assisted in searching for and identifying potentially relevant
literature and in checking bibliographic records. The authors verified cited
works against primary sources and checked titles, author lists, identifiers,
and the relevance of the cited results.

\textbf{Other uses.}
Generative AI was not used for translation, cleaning or reformatting an
external dataset, or qualitative or thematic data analysis; these tasks were
not applicable to this work.

The authors reviewed all AI-assisted work and take full responsibility for the
final content of this paper, including all text, mathematical claims, proofs,
code, experimental results, citations, and conclusions produced with model
assistance.

\section*{Acknowledgements}

N.D.N. is a Senior Postdoctoral Fellow of the Research Foundation -- Flanders
(FWO; fellowship 12ADS26N-7029), and was further supported by a Pilot Grant
from the Stichting Alzheimer Onderzoek -- Fondation Recherche Alzheimer
(SAO-FRA; grant 20230054) and a Young Researcher Grant from the Queen Elisabeth
Medical Foundation for Neurosciences (QEMF). N.D.N. thanks his wife, Dao Ha
Anh, for her steady support, and for caring for their newborn child while this
work was completed.

\bibliographystyle{iclr2027_conference}
\bibliography{refs}

\appendix
\section*{Appendix}
The appendices are organized as follows. App.~\ref{app:ledger} is the
hypothesis ledger: every numbered claim against its status. App.~\ref{app:related}
extends the positioning. Apps.~\ref{app:sec2}--\ref{app:gauge-contraction}
prove the structural results of \S\ref{sec:identifiability}--\ref{sec:how-much}.
Apps.~\ref{app:transient}--\ref{app:temporal-ou} develop the tangent analysis, its
Gaussian specialization and the intrinsic-dimension reduction.
Apps.~\ref{app:strong}--\ref{app:selection} cover the estimator, including its
finite-basis guarantee (App.~\ref{app:weak-proof}).
Apps.~\ref{app:caseB}--\ref{app:repro} contain the remaining experiments, stress
tests, scope, negative results and reproduction instructions.

\section{Hypothesis ledger}\label{app:ledger}

Tab.~\ref{tab:ledger} labels each claim \textbf{P} (proved as stated), \textbf{C} (proved
conditional on stated assumptions), \textbf{N} (numerically supported, not
proved), or \textbf{O} (open).

\begin{table}[htbp]
\centering
\caption{\textbf{Hypothesis ledger.} Status of every numbered claim and what it
rests on (App.~\ref{app:ledger}).}
\label{tab:ledger}
\begin{tabular}{@{}llp{0.64\linewidth}@{}}
\toprule
claim & status & scope and what it rests on\\
\midrule
Prop.~\ref{prop:saturated} & P & Poincar\'e inequality,
$\partial_t\log\rho_t\in L^2(\rho_t)$; App.~\ref{app:sec2}\\
Cor.~\ref{cor:blind} & P & $(\ker\A_\rho)^\perp=\overline{\{\nabla\phi\}}$;
single time\\
Cor.~\ref{cor:otblind} & P & Gaussian marginals; closed-form
Benamou--Brenier\\
Cor.~\ref{cor:fmsb} & P & population level; (i) Gaussian slices, exact OT
coupling, noiseless interpolant; (ii) any (possibly non-gradient) reference
drift, smooth positive Schr\"odinger potentials\\
Prop.~\ref{prop:equiv} & P & source constraints only, at finitely many or a
continuum of times; forward uniqueness for the continuum converse\\
Ex.~\ref{ex:alias} & P & explicit; shows $\ker\A_K$ is not necessary for
snapshot equivalence\\
Prop.~\ref{prop:count} & P & smooth $h$; fibre dimension, not function-space
dimension\\
Prop.~\ref{prop:compensation} & P & time-conditioned head; per-slice Poisson
solvability\\
Prop.~\ref{prop:symmetry} & P & isotropic block; exact marginal invariance\\
Prop.~\ref{prop:hermite} & P & Gaussian anchor slice; Hermite ladder
(App.~\ref{app:hermite})\\
Thm.~\ref{thm:gauge} & P & compared slices Gaussian with a common mean (the
reference slices in the local analysis); exact, no linearization; sharpness by
the explicit field of Lem.~\ref{lem:wedge}\\
Prop.~\ref{prop:gauge-structure} & P & leading-degree bound; its attainment
below threshold is numerical (App.~\ref{app:below}), open in general\\
degree cap (App.~\ref{app:below}) & P/N & exact gauge of degree $K$ for
$K\le m-1$ proved; identification up to degree $K-1$ proved for $K=m-1$,
certified exactly for $m\le6$, open in general\\
faithfulness (\S\ref{sec:how-fast}) & O & bridges source identifiability to the tangent statistical limits; not used
by Thm.~\ref{thm:oracle-rate}; implied by Asm.~\ref{asm:thin-tail}; holds with
equality in the shape-rich cases computed (App.~\ref{app:transient})\\
Asm.~\ref{asm:comparability} & O & one-sided Loewner comparison at every cutoff;
used for the lower bound through Lem.~\ref{lem:compare}; not verified\\
Asm.~\ref{asm:scale} & O & $r_q\asymp q^{m-1}$ proved in Prop.~\ref{prop:hermite}; the
$K/q$ scale is a heuristic factorization with four stated gaps
(App.~\ref{app:scale})\\
Asm.~\ref{asm:thin-tail} & N & tail exponent measured only on the source
operator; the implied trace law measured on the propagated design
(App.~\ref{app:spectral})\\
Thm.~\ref{thm:oracle-rate} & C & lower bound in the tangent snapshot experiment
under Asm.~\ref{asm:scale} and \ref{asm:comparability}; upper bound in the
degreewise benchmark under Asm.~\ref{asm:scale} and \ref{asm:thin-tail};
\emph{not} a sampled-snapshot theorem (Rem.~\ref{rem:oracle-gap})\\
Thm.~\ref{thm:weak-bound} & P & snapshot-only, finite basis, finite sample;
needs $\alpha=s_{\min}(WX)>0$ on the retained subspace\\
Prop.~\ref{prop:notorth} & P & nonzero solenoidal component; first-order
expansion\\
Prop.~\ref{prop:filter} & P & diagonal filters in the profiled basis only\\
Prop.~\ref{prop:hessian} & P & centered Gaussian slices, covariance-determined
weighting; fails for general densities\\
Prop.~\ref{prop:gls} & P & independent slices, fixed $a$; violated by
lineage-tracked designs (App.~\ref{app:caseB})\\
Prop.~\ref{prop:misalign} & P & whitened linear experiment; ridge penalty;
risks averaged over coordinate signs of $a^\star$; linear functional\\
extensions (App.~\ref{app:scope}) & P & known state-dependent diffusion, for
\S\ref{sec:identifiability} and the weak moments only; not \S\ref{sec:how-fast}\\
Prop.~\ref{prop:ratesplit} & C & smoothed-manifold regime of
App.~\ref{app:manifold}\\
Obs.~\ref{obs:tension} & N & one system, $m=4$; no characterization of when
both requirements can hold\\
\bottomrule
\end{tabular}
\end{table}

Three cautions carry across the paper. First, every rate statement lives in the
tangent snapshot experiment or its degreewise benchmark, never in the
sampled-snapshot experiment. Second, the
finite-dimensional experiments of \S\ref{sec:exp} exercise the structural
predictions and the estimator; they do not validate the nonparametric exponent.
Third, $\rho$-solenoidal, Euclidean divergence-free, and ``antisymmetric linear
coefficient'' are three different things, and results proved for one do not
transfer to another.

\section{Extended related work}\label{app:related}
\textbf{Population dynamics and generative transport.}
Trajectory inference is commonly posed as transport between temporal marginals:
optimal transport \citep{schiebinger2019optimal}, neural-ODE and JKO approaches
\citep{tong2020trajectorynet,huguet2022manifold,bunne2022proximal,%
terpin2024learning}, Schr\"odinger bridges \citep{debortoli2021diffusion,shi2023},
and flow-matching or generative formulations
\citep{hashimoto2016learning,yeo2021generative,lipman2023flow,tong2024,%
lee2025mmsfm}. These construct dynamics consistent with prescribed
distributions; we ask which aspects of the dynamics those distributions
determine, and Cor.~\ref{cor:fmsb} shows that OT-coupled flow matching returns
a gradient field at the population level and that Schr\"odinger-bridge matching,
for any reference drift, only ever adds a gradient correction to it, so neither
method can manufacture circulation a curl-free reference lacks. Consistency results often assume gradient drift
\citep{lavenant2024towards}, and action matching explicitly targets a gradient
representative \citep{neklyudov2023action}; rotational ambiguity from
population snapshots was noted by \citet{weinreb2018fundamental}. Scores near a
smoothed manifold encode support geometry strongly \citep{li2025scores}, which
we connect to the present setting in App.~\ref{app:manifold}.

\textbf{Breaking the gauge with dynamical information.}
Autonomy couples Eulerian marginals through one shared field; velocity,
pairing or lineage measurements instead supply Lagrangian information
\citep{petrovic2025curly,lamanno2018rna,bergen2020generalizing}. We isolate
what marginals alone determine; a lineage-tracked variant of our nonlinear
benchmark is in App.~\ref{app:caseB}.

\textbf{Identifiability from marginals.}
The closest identifiability results either impose gradient dynamics
\citep{hashimoto2016learning,lavenant2024towards,neklyudov2023action} or
specialize to linear-Gaussian models. \citet{guan2024} characterize when the
drift and diffusion of a linear additive-noise SDE are identifiable from
population marginals, in terms of generalized rotational symmetries of the
initial law. We treat an infinite-dimensional field class, condition on the
observed shape family rather than on the initial law, and study finite-sample
recovery; App.~\ref{app:gauss} makes the comparison precise.

\textbf{Score-based and weak-form estimation.}
Strong-form approaches combine score estimation
\citep{hyvarinen2005estimation,vincent2011connection,song2019generative} with
continuity or Fokker--Planck equations \citep{maoutsa2020interacting,%
song2021scorebased,boffi2023probability}; we show score error is a
non-orthogonal nuisance there. Weak formulations are classical in system
identification \citep{brunton2016discovering,messenger2021weak} and correspond
to generalized method of moments \citep{hansen1982large}; we adapt them to
unpaired snapshots. Related nonequilibrium methods infer forces from full
trajectories \citep{frishman2020learning,gnesotto2018broken}.

\section{Proofs for Section~\ref{sec:identifiability}}\label{app:sec2}

Throughout, densities are smooth and positive with enough decay for the
displayed integrations by parts, and $\A_\rho F=\rho^{-1}\nabla\!\cdot(\rho F)$.

\subsection{The adjoint and the two orthogonal complements}

Scalars live in $L^2(\rho)$ and vector fields in $L^2(\rho;\R^m)$, with
$\langle F,G\rangle_{\Ltwo}:=\int F\cdot G\,\rho$; orthogonal complements and
closures of sets of fields are taken in $L^2(\rho;\R^m)$. For scalar $\phi$ and
field $F$ in the relevant domains,
$\langle\A_\rho F,\phi\rangle_{\Ltwo}=\int\nabla\!\cdot(\rho F)\,\phi
=-\int\rho\,F\cdot\nabla\phi=-\langle F,\nabla\phi\rangle_{\Ltwo}$, so
$\A_\rho^\ast\phi=-\nabla\phi$ and
\begin{equation}\label{eq:perp}
    (\ker\A_\rho)^\perp=\overline{\operatorname{range}\A_\rho^\ast}
    =\overline{\{\nabla\phi\}}.
\end{equation}
This is a statement about \emph{one} measure and does not survive stacking: for
the stacked operator, with fields in $L^2(\mu;\R^m)$ for a reference density
$\mu$ and the $k$-th component in $L^2(\rho_{t_k})$,
$\A_K^\ast(\phi_k)_k=-\mu^{-1}\sum_k\rho_{t_k}\nabla\phi_k$, whose range is the span of
\emph{reweighted} gradients --- strictly larger than $\overline{\{\nabla\phi\}}$
as soon as two slices differ. Much of the folklore about $L^2$ actions losing
the curl is \eqref{eq:perp} applied at a single time; \S\ref{sec:autonomy} is
what happens when it is not.

\begin{proof}[Proof of Prop.~\ref{prop:saturated}]
Fix $t$, write $\rho=\rho_t$ and $f=\partial_t\rho_t$, and let $C_P$ be the
Poincar\'e constant of $\rho$, so that
$\operatorname{Var}_\rho(\psi)\le C_P\int|\nabla\psi|^2\rho$. On
$\mathcal H:=H^1(\rho)/\mathbb R$ let
\[
    a(\phi,\psi):=\int\nabla\phi\cdot\nabla\psi\,\rho,
    \qquad
    \ell(\psi):=\int\psi f.
\]
Mass conservation gives $\int f=\tfrac{d}{dt}\int\rho_t=0$, so $\ell$
annihilates constants and is well defined on $\mathcal H$. The form $a$ is
bounded and, by the Poincar\'e inequality, coercive on $\mathcal H$, and
\[
    |\ell(\psi)|
    =\Bigl|\int\bigl(\psi-\E_\rho\psi\bigr)\,\partial_t\log\rho\,\rho\Bigr|
    \le C_P^{1/2}\,\|\nabla\psi\|_{L^2(\rho)}\,
        \|\partial_t\log\rho\|_{L^2(\rho)},
\]
so $\ell$ is bounded. Lax--Milgram gives a unique $\phi_t\in\mathcal H$ with
$a(\phi_t,\cdot)=\ell$, which is the weak form of
$\nabla\!\cdot(\rho_t\nabla\phi_t)=-\partial_t\rho_t$. Thus
$v_t:=\nabla\phi_t$ satisfies $\partial_t\rho_t+\nabla\!\cdot(\rho_tv_t)=0$,
and testing with $\psi=\phi_t$ gives
$\|v_t\|_{L^2(\rho_t)}\le C_P^{1/2}\|\partial_t\log\rho_t\|_{L^2(\rho_t)}$.
That transporting $\rho_0$ along $v$ recovers $\{\rho_t\}$, rather than merely
being consistent with it, would further require uniqueness for the continuity
equation, which needs regularity of $v$ beyond $L^2$; the unfalsifiability
conclusion uses only the consistency proved here.
\end{proof}

\begin{proof}[Proof of Cor.~\ref{cor:blind}]
The velocities compatible with the observed path at time $t$ form the affine
set $v_t^0+\ker\A_{\rho_t}$ for any particular solution $v_t^0$. The element of
least $L^2(\rho_t)$ norm in an affine set is its projection onto the orthogonal
complement of the direction space, which is $\overline{\{\nabla\phi\}}$ by
\eqref{eq:perp}. Its $\rho_t$-solenoidal component is therefore zero. Note this
is a projection, not a bias: least-action reconstruction does not shrink the
circulation, it deletes it.
\end{proof}

\begin{corollary}[Displacement interpolation is curl-blind in closed form]
\label{cor:otblind}
For Gaussian marginals with a planted rotation, the Benamou--Brenier
displacement velocity between consecutive slices is a gradient on every
segment. Its solenoidal relative error equals $1$ at every sample size, while
the induced marginal path matches the observed one exactly.
\end{corollary}

\begin{proof}
Between consecutive centered Gaussians $\mathcal N(0,\Sigma_k)$ and
$\mathcal N(0,\Sigma_{k+1})$ the quadratic-cost optimal map is the linear map
$T_k=\Sigma_k^{-1/2}(\Sigma_k^{1/2}\Sigma_{k+1}\Sigma_k^{1/2})^{1/2}
\Sigma_k^{-1/2}$, which is symmetric positive definite. Displacement
interpolation uses $X_s=((1-s)I+sT_k)X_0$, whose velocity field at time $s$ is
$v_s(x)=(T_k-I)((1-s)I+sT_k)^{-1}x$. All three factors are symmetric and
commute (they are polynomials in $T_k$), so $v_s$ is a symmetric matrix field,
hence the gradient of a quadratic form and exactly curl-free. Writing the
readout as $\relerr=\|\widehat h-h^\star\|/\|h^\star\|$ with $\widehat h=0$ gives $1$ at every
sample size, while by construction the interpolation matches both endpoint
marginals exactly.
\end{proof}

\begin{corollary}[Flow matching and Schr\"odinger bridges cannot manufacture circulation]
\label{cor:fmsb}
(i) Between consecutive centered Gaussian slices, conditional flow matching with
the exact quadratic-cost OT coupling and the noiseless linear interpolant has
population regression target equal to the displacement velocity $v_s$ of
Cor.~\ref{cor:otblind}, whose $\rho_s$-solenoidal component vanishes, so its
solenoidal relative error is $1$ at every sample size. (ii) Let a reference
process obey \eqref{eq:sde} with any smooth, possibly time-dependent, possibly
non-gradient drift $F_{\mathrm{ref},t}$, and diffusion $\sigma^2$ as fixed
throughout. A Schr\"odinger bridge between two marginals relative to this
reference, with smooth positive Schr\"odinger potentials (as holds for Gaussian
marginals), is Markov with It\^o drift
$F_{\mathrm{ref},t}+2\sigma^2\nabla\log \eta_t$: the reference drift plus a
gradient correction. Consequently the bridge's $\rho_t$-solenoidal component at
every $t$ equals that of $F_{\mathrm{ref},t}$ exactly, whether or not
$F_{\mathrm{ref},t}$ is itself a gradient: a Schr\"odinger bridge can neither
create circulation a curl-free reference lacks nor remove circulation a
rotational reference has, also for multi-marginal constructions built segment
by segment. In particular, a curl-free reference (Brownian motion, or any
gradient drift) forces the bridge to be curl-free, with solenoidal relative
error $1$ at every sample size, as for (i); with a circulating reference the
fitted circulation is whatever the modeller built into the reference, not
information extracted from the marginals.
\end{corollary}

\begin{proof}
(i) The OT plan between the two Gaussians is deterministic, $X_1=T_kX_0$ (proof
of Cor.~\ref{cor:otblind}). The interpolant
$X_s=(1-s)X_0+sX_1=((1-s)I+sT_k)X_0$ is an invertible linear image of $X_0$, and
conditional flow matching regresses on $X_1-X_0=(T_k-I)X_0$, so its
$L^2$-minimizer is
$\E[X_1-X_0\mid X_s=x]=(T_k-I)((1-s)I+sT_k)^{-1}x=v_s(x)$, the velocity of
Cor.~\ref{cor:otblind}; the corresponding It\^o drift $v_s+\sigma^2s_s$ is a
gradient as well (\S\ref{sec:blind}).
(ii) The bridge law is the reference path law reweighted by $f_0(X_0)f_1(X_1)$ for
the Schr\"odinger potentials $f_0,f_1$. Conditionally on $X_t=x$ the future is the
reference reweighted by $f_1(X_1)$, a Doob transform with
$\eta_t(x):=\E_{\mathrm{ref}}[f_1(X_1)\mid X_t=x]$, which is space--time harmonic,
$(\partial_t+\mathcal L_{\mathrm{ref},t})\eta_t=0$ with
$\mathcal L_{\mathrm{ref},t}u:=F_{\mathrm{ref},t}\cdot\nabla u+\sigma^2\Delta u$
for the (possibly non-gradient) reference generator. For space--time harmonic
$\eta_t$ and any smooth $u$, the product rule gives, at every $(x,t)$,
\begin{align*}
\eta_t^{-1}\bigl(\partial_t+\mathcal L_{\mathrm{ref},t}\bigr)(\eta_tu)
&=\eta_t^{-1}\bigl[(\partial_t\eta_t)u+\eta_t\partial_tu
  +F_{\mathrm{ref},t}\cdot(u\nabla\eta_t+\eta_t\nabla u)\bigr]\\
&\quad+\sigma^2\eta_t^{-1}\bigl(u\Delta\eta_t+2\nabla\eta_t\cdot\nabla u
  +\eta_t\Delta u\bigr)\\
&=u\,\eta_t^{-1}\bigl[(\partial_t+\mathcal L_{\mathrm{ref},t})\eta_t\bigr]
  +\partial_tu+\mathcal L_{\mathrm{ref},t}u
  +2\sigma^2\nabla\log\eta_t\cdot\nabla u\\
&=\partial_tu+\mathcal L_{\mathrm{ref},t}u+2\sigma^2\nabla\log\eta_t\cdot\nabla u,
\end{align*}
the bracketed harmonicity term vanishing identically; this uses only that
$F_{\mathrm{ref},t}\cdot\nabla$ is a first-order operator, not that
$F_{\mathrm{ref},t}$ is a gradient. Hence the bridge It\^o drift is
$F_{\mathrm{ref},t}+2\sigma^2\nabla\log \eta_t$. By \eqref{eq:perp},
$\overline{\{\nabla\phi\}}=(\ker\A_{\rho_t})^\perp$ for every $\rho_t$, so adding
a gradient field to \emph{any} vector field leaves its $\rho_t$-solenoidal
projection unchanged: the bridge and the reference drift have the same
$\rho_t$-solenoidal component at every $t$. When $F_{\mathrm{ref},t}$ is itself
a gradient this component is zero and the readout of Cor.~\ref{cor:otblind}
gives $1$.
\end{proof}

\begin{proof}[Proof of Prop.~\ref{prop:equiv}]
Finitely many times: if $\A_{\rho_{t_k}}F=b_{t_k}$ and
$\A_{\rho_{t_k}}F'=b_{t_k}$ for all $k$, subtracting gives
$\A_{\rho_{t_k}}(F-F')=0$ for every $k$, i.e.\ $F-F'\in\ker\A_K$; the converse
is immediate by linearity. A continuum of times: the same argument at every $t$
gives the intersection over the interval. For the converse, let
$h:=F'-F$ satisfy $\nabla\!\cdot(\rho_th)=0$ for all $t$. Then
$\partial_t\rho_t+\nabla\!\cdot(\rho_tF')-\sigma^2\Delta\rho_t
=\partial_t\rho_t+\nabla\!\cdot(\rho_tF)-\sigma^2\Delta\rho_t=0$, so $\rho_t$
solves the Fokker--Planck equation for $F'$ as well; with the same initial law
and uniqueness of the forward equation, $F'$ generates exactly the path
$\{\rho_t\}$.
\end{proof}

The converse fails for finitely many instants, which is the content of
Rem.~\ref{rem:whypath} below and the reason \S\ref{sec:identifiability} distinguishes
source constraints from the snapshot experiment.

\begin{remark}[Path operator, not snapshot operator]\label{rem:whypath}
Always $\ker\A_{[0,T]}\subseteq\ker\A_K\subseteq\ker\A_{\rho_{t_1}}$, where
$\A_{[0,T]}F:=(\A_{\rho_t}F)_{t\in[0,T]}$. Membership in $\ker\A_K$ is not
sufficient for path equivalence: annihilating $\nabla\!\cdot(\rho_{t_k}h)$ at
$K$ instants leaves $\rho_t$ free to solve a different equation between them, so
the paths --- and hence later marginals --- separate. Only the continuum
condition makes $\rho_t$ solve the perturbed equation. At first order around a
\emph{stationary} base all three operators collapse onto $\A_\rho$, which is why
the distinction is easy to miss.
\end{remark}

\begin{proof}[Verification of Ex.~\ref{ex:alias}]
For $F_\omega(x)=(-aI+\omega J)x$ with isotropic diffusion $\sigma^2I$, the
covariance solves $\dot\Sigma=(-aI+\omega J)\Sigma+\Sigma(-aI+\omega J)^\top
+2\sigma^2I$, whose solution with $R_\theta:=e^{\theta J}$ is the displayed
$\Sigma_t$, using $J^\top=-J$ and $R_\theta R_\theta^\top=I$. At $t_k=k\Delta$,
replacing $\omega$ by $\omega+\pi/\Delta$ replaces $R_{\omega t_k}$ by
$R_{\omega t_k+\pi k}=(-1)^kR_{\omega t_k}$, and
$(-1)^kR\Sigma_0R^\top(-1)^k=R\Sigma_0R^\top$; every centered Gaussian snapshot
is therefore unchanged. The drift difference is $(\pi/\Delta)Jx$. A linear field
$Gx$ is $\rho$-solenoidal for $\rho=\mathcal N(0,\Sigma)$ iff $\tr G=0$ and
$\Sigma^{-1}G$ is antisymmetric (Lem.~\ref{lem:linear}). For $G=J$ and
$\Sigma_0=\operatorname{diag}(\sigma_1^2,\sigma_2^2)$,
$\Sigma_0^{-1}J=\bigl(\begin{smallmatrix}0&-\sigma_1^{-2}\\
\sigma_2^{-2}&0\end{smallmatrix}\bigr)$, antisymmetric iff
$\sigma_1=\sigma_2$. For anisotropic $\Sigma_0$ the difference is therefore not
in $\ker\A_{\rho_0}$, so $\ker\A_K$ is not necessary for snapshot equivalence.
With $a=0.4$, $\sigma^2=0.3$, $\Delta=0.5$, $\omega=0.8$ and
$\Sigma_0=\operatorname{diag}(2,0.7)$, the maximum covariance discrepancy over
five sampled times is $2.0\times10^{-14}$ while the discrepancy halfway to the
next snapshot is $1.505$.
\end{proof}

\begin{proof}[Proof of Prop.~\ref{prop:count}]
If $h\in\ker\A_K$ then $\nabla\!\cdot h+h\cdot s_k=0$ for every $k$.
Subtracting the $k=1$ equation from the $k$-th eliminates $\nabla\!\cdot h$ and
leaves $h(x)\cdot(s_k(x)-s_1(x))=0$ pointwise. Conversely these $K-1$
conditions together with $\A_{\rho_{t_1}}h=0$ reconstitute all $K$ equations.
At each $x$ the conditions say $h(x)\perp\spn\{s_k(x)-s_1(x)\}_{k\ge2}=:\mathcal
C_x^\perp$, a subspace of dimension $r(x)$, so $h(x)$ is confined to
$\mathcal C_x$ of dimension $m-r(x)$. This is a constraint on the \emph{value}
at each point; the unknown is a section $x\mapsto h(x)\in\mathcal C_x$ in
$L^2(\rho_{t_1})$, and the space of such sections is infinite-dimensional
whenever $m-r(x)\ge1$ on a set of positive measure. The differential condition
$\A_{\rho_{t_1}}h=0$ then removes further sections; Thm.~\ref{thm:gauge} shows
that on polynomials it removes all of them once the shapes are rich.
\end{proof}

\begin{proposition}[A time-conditioned curl head is exactly compensable]
\label{prop:compensation}
Let a model represent $F_t=-\nabla\Phi_t+h_t$ with $h_t$ an unconstrained
time-conditioned $\rho_t$-solenoidal head, and suppose the marginal path
satisfies the hypotheses of Prop.~\ref{prop:saturated}. Then for \emph{every}
choice of $\{h_t\}$ there is a $\{\Phi_t\}$ reproducing the observed path
exactly. The circulation is therefore unidentifiable within such a model even
given oracle nuisances.
\end{proposition}

\begin{proof}
Fix $\{h_t\}$ with $\nabla\!\cdot(\rho_th_t)=0$. We need $\Phi_t$ with
$\partial_t\rho_t+\nabla\!\cdot(\rho_t(-\nabla\Phi_t+h_t))
-\sigma^2\Delta\rho_t=0$. Since $\nabla\!\cdot(\rho_th_t)=0$ this reduces to
$\nabla\!\cdot(\rho_t\nabla\Phi_t)=\partial_t\rho_t-\sigma^2\Delta\rho_t$,
whose right-hand side integrates to zero, so Prop.~\ref{prop:saturated} applies
verbatim at each $t$ and supplies $\Phi_t$. The compensation is per-slice and
independent across $t$.
\end{proof}

Autonomy is exactly what breaks this: a single $F$ must satisfy $K$ constraints
with one potential, so the compensator would have to solve $K$ Poisson problems
simultaneously with a shared solution, which Thm.~\ref{thm:gauge} shows is
impossible on polynomials under shape richness.

\begin{proposition}[Rotation in an isotropic block is an exact gauge]
\label{prop:symmetry}
Let the state split as $\R^m=E\oplus E^\perp$ with $\Sigma_0|_{E^\perp}=
\varsigma^2I$, let $G$ be antisymmetric supported in $E^\perp$, and let the
reference dynamics act blockwise. Then $G$ leaves every marginal unchanged, so
it lies in $\ker\A_K$ for any $K$ and any snapshot times.
\end{proposition}

\begin{proof}
With $Q_t=e^{Gt}$ orthogonal and supported in $E^\perp$, the covariance
contribution from that block is $Q_t(\varsigma^2I)Q_t^\top=\varsigma^2Q_tQ_t^\top
=\varsigma^2I$, unchanged for every $t$; the isotropic part of the stationary
covariance is likewise $Q_t$-invariant. Hence $\Sigma_t$ and every centered
Gaussian marginal are identical with and without $G$. Numerically, planting a
rotation of Frobenius norm $0.4$ in the isotropic normal block of the $d=12$
benchmark leaves the maximal off-diagonal entry of that block below
$10^{-16}$ at all five snapshot times, whereas a tangent-block rotation is
visible.
\end{proof}

\section{Weighted-solenoidal Hermite algebra}\label{app:hermite}

Fix $\rho_\star=\mathcal N(0,I_m)$. Write $\delta_j:=x_j-\partial_j$ for the
Gaussian creation operator, so that $\delta_j\Psi_\beta=\sqrt{\beta_j+1}\,
\Psi_{\beta+e_j}$ and $\partial_j\Psi_\beta=\sqrt{\beta_j}\,\Psi_{\beta-e_j}$;
the $\delta_j$ commute, and $x_j=\delta_j+\partial_j$ on polynomials.

Prop.~\ref{prop:hermite} is Lem.~\ref{lem:koszul}(i) together with
Lem.~\ref{lem:surj}.

\begin{lemma}[Degree raising and the Koszul model]\label{lem:koszul}
(i) $\A_{\rho_\star}h=-\delta\cdot h:=-\sum_i\delta_ih_i$ for every polynomial
field $h$; in particular $\A_{\rho_\star}(\Psi_\beta e_i)=-\sqrt{\beta_i+1}\,
\Psi_{\beta+e_i}$. (ii) The linear map $\iota:\Psi_\beta\mapsto
x^\beta/\sqrt{\beta!}$, extended componentwise to fields, is an isomorphism of
polynomial spaces intertwining $\delta_j$ with multiplication by $x_j$ and
$\partial_j$ with $\partial_j$. Consequently $\iota$ maps $S_q$ onto the
homogeneous fields $g$ of degree $q-1$ with $x\cdot g\equiv0$.
\end{lemma}

\begin{proof}
(i) $\A_{\rho_\star}h=\nabla\!\cdot h+h\cdot\nabla\log\rho_\star=\sum_i
(\partial_i-x_i)h_i=-\sum_i\delta_ih_i$, and the action on $\Psi_\beta e_i$ is
the creation identity. (ii) On unnormalized Hermite polynomials
$\delta_jH_\beta=H_{\beta+e_j}$ and $\partial_jH_\beta=\beta_jH_{\beta-e_j}$,
which are the actions of $x_j\cdot$ and $\partial_j$ on $x^\beta$;
normalization is a diagonal change of basis. By (i), $h\in S_q$ iff
$\delta\cdot h=0$, which $\iota$ carries to $x\cdot g=0$.
\end{proof}

\begin{lemma}[Surjectivity and the gauge dimension]\label{lem:surj}
For $q\ge1$, $\A_{\rho_\star}|_{W_q}:W_q\to\mathcal P_q$ is onto, so for $q\ge2$
$\dim S_q=mD_{q-1}-D_q\asymp q^{m-1}$ and $\dim S_q/\dim W_q\to(m-1)/m$.
\end{lemma}

\begin{proof}
Given $|\alpha|=q$ choose $i$ with $\alpha_i\ge1$; then
$\A_{\rho_\star}(\Psi_{\alpha-e_i}e_i)=-\sqrt{\alpha_i}\,\Psi_\alpha$ by
Lem.~\ref{lem:koszul}(i). Rank--nullity gives the dimension, and
$D_q/D_{q-1}=(q+m-1)/q\to1$ gives the ratio.
\end{proof}

The limit $(m-1)/m$ is a reference-weighted gauge fraction and should not be
confused with any fraction computed in the Euclidean divergence-free class, a
different subspace with a different Stein operator. On $S_q$ the reference
Stein operator vanishes identically by definition, which is why a stationary
spectrum computed in that other class cannot verify anything about $S_q$.

\begin{proposition}[Structure of the single-slice gauge]\label{prop:Sq-structure}
For $q\ge2$, $S_q$ is exactly the image of an antisymmetric matrix potential,
\[
    S_q=\Bigl\{h_i=\textstyle\sum_j\delta_ja_{ij}\;:\;
        a_{ij}=-a_{ji}\in\mathcal P_{q-2}\Bigr\}
       =\Bigl\{\rho_\star^{-1}\nabla\!\cdot(\rho_\star a)\;:\;a^\top=-a,\ 
        a_{ij}\in\mathcal P_{q-2}\Bigr\},
\]
and the parameterization $a\mapsto h$ has kernel
$\{(\sum_k\delta_kb_{ijk})_{ij}:b\text{ totally antisymmetric}\}$.
\end{proposition}

\begin{proof}
If $h_i=\sum_j\delta_ja_{ij}$ with $a$ antisymmetric then
$\delta\cdot h=\sum_{ij}\delta_i\delta_ja_{ij}=0$, since the $\delta$'s commute
and $a$ is antisymmetric, so the image lies in $S_q$. Under $\iota$ the claim
becomes exactness of the Koszul complex of $(x_1,\dots,x_m)$ on polynomials: a
homogeneous field $g$ with $x\cdot g=0$ is $g_i=\sum_jx_ja_{ij}$ for some
antisymmetric $a$ of one lower degree, unique modulo $a_{ij}=\sum_kx_kb_{ijk}$
with $b$ totally antisymmetric. The second description is
$\rho_\star^{-1}\partial_j(\rho_\star a_{ij})=(\partial_j-x_j)a_{ij}=-\delta_j
a_{ij}$, a sign convention.
\end{proof}

\begin{lemma}[Linear fields]\label{lem:linear}
For $\rho=\mathcal N(0,\Sigma)$ and $h(x)=Gx$, $h$ is $\rho$-solenoidal iff
$\Sigma^{-1}G$ is antisymmetric; equivalently $G=\Sigma N$ with $N^\top=-N$,
equivalently $G\Sigma+\Sigma G^\top=0$. For $\Sigma=I$ this is $G^\top=-G$.
\end{lemma}

\begin{proof}
$\A_\rho(Gx)=\tr G-x^\top G^\top\Sigma^{-1}x$, which vanishes identically iff
$\tr G=0$ and $G^\top\Sigma^{-1}+\Sigma^{-1}G=0$. With $N:=\Sigma^{-1}G$ the
second condition is $N^\top=-N$, and then $\tr G=\tr(\Sigma N)=0$ automatically,
as the trace of a symmetric times an antisymmetric matrix. Multiplying
$G^\top\Sigma^{-1}+\Sigma^{-1}G=0$ by $\Sigma$ on both sides gives
$G\Sigma+\Sigma G^\top=0$.
\end{proof}

Lem.~\ref{lem:linear} is the bridge to \citet{guan2024}: their
$\Sigma$-generalized rotations solve $A\Sigma+\Sigma A^\top=0$, so the
$\rho$-solenoidal linear fields are exactly those rotations
(App.~\ref{app:gauss}).

Multiplication by $x_j$ is not degree-preserving in the Hermite basis, since
$x_j=\delta_j+\partial_j$ raises and lowers. This is what makes the cross-slice
operator banded rather than block diagonal.

\begin{lemma}[Raising/lowering split]\label{lem:split}
For symmetric $B$ and $h\in W_q$ put $R_Bh:=-\delta\cdot(Bh)\in\mathcal P_q$ and
$L_Bh:=-\nabla\!\cdot(Bh)\in\mathcal P_{q-2}$. Then $-h\cdot Bx=R_Bh+L_Bh$, the
two summands are orthogonal in $L^2(\rho_\star)$, and under $\iota$ the operator
$R_B$ becomes $g\mapsto-g\cdot Bx$.
\end{lemma}

\begin{proof}
By symmetry of $B$, $-h\cdot Bx=-\sum_{ij}B_{ij}x_jh_i=-\sum_jx_j(Bh)_j=
-\sum_j(\delta_j+\partial_j)(Bh)_j$. The two sums lie in different chaoses,
hence are orthogonal, and the last claim is Lem.~\ref{lem:koszul}(ii).
\end{proof}

\section{Gauge contraction under changing shape}\label{app:gauge-contraction}

Throughout, the compared slices are Gaussian with a common mean and are whitened
at $t_\star$ as in \S\ref{sec:how-much} --- the reference slices in the local
analysis, or the observed marginals when these are themselves Gaussian --- and
$\A_K$ denotes their stacked Stein operator ($\bar\A_K$ in \S\ref{sec:how-much}),
so $s_k-s_\star=-\Xi_kx$ exactly. Write
$T_K:=\bigoplus_qT_{K,q}$ on $\mathcal S_{\le Q}$ and
$R_{K,q}:=(R_{\Xi_k})_{k\ne\star}$ for its raising part.

\begin{proposition}[Structure of the persistent gauge]\label{prop:gauge-structure}
\begin{enumerate}[label=(\roman*)]
\item $\mathcal S_{\le Q}\cap\ker\A_K=\ker(T_K|_{\mathcal S_{\le Q}})$, and
$\ker\A_K\subseteq\ker\A_{\rho_\star}$.
\item $\dim\ker(T_K|_{\mathcal S_{\le Q}})\le\sum_{q\le Q}\dim\ker R_{K,q}$.
\item Under $\iota$, $\ker R_{K,q}\cong\{g\in\R[x]^m_{q-1}:g(x)\perp
\spn\{x,\Xi_1x,\dots,\Xi_rx\}\ \forall x\}$.
\item $\ker R_{K,q}=\{0\}$ for every $q$ iff the shapes are shape-rich in the
sense of Thm.~\ref{thm:gauge}; generically this holds iff $r\ge m-1$.
\item For $r=m-2$ and generic shapes, $\ker R_{K,q}=\{p\,w:p\in\mathcal
P_{q-m}\}$ under $\iota$, of dimension $D_{q-m}$, with $w$ as in
Lem.~\ref{lem:wedge}.
\end{enumerate}
\end{proposition}

\begin{proof}
(i) For $h\in S_q$, $\A_{\rho_k}h=\A_{\rho_\star}h+h\cdot(s_k-s_\star)=
h\cdot(s_k-s_\star)$, so $\A_Kh=(T_Kh,0)$; and $\rho_\star$ is itself observed.

(ii) By Lem.~\ref{lem:split}, $T_K$ maps $S_q$ into $\mathcal P_q\oplus\mathcal
P_{q-2}$, with the $\mathcal P_q$ part equal to $R_{K,q}$. If $0\ne h=\sum_{q\le
Q'}h_q\in\ker T_K$ with top component $h_{Q'}\ne0$, the only contribution to the
$\mathcal P_{Q'}$ component of $T_Kh$ is $R_{K,Q'}h_{Q'}$, so $h_{Q'}\in\ker
R_{K,Q'}$. The map $h\mapsto h_{Q'}$ has kernel
$\ker(T_K|_{\mathcal S_{\le Q'-1}})$; induct on $Q'$.

(iii) By Lem.~\ref{lem:split}, $R_{K,q}h=0$ iff $g\cdot \Xi_kx\equiv0$ for all
$k$, and by Lem.~\ref{lem:koszul}, $h\in S_q$ iff $g\cdot x\equiv0$; by
Prop.~\ref{prop:Sq-structure} such $g$ are exactly the Koszul images
$g_i=\sum_jx_ja_{ij}$ of antisymmetric potentials.

(iv) If the span is $\R^m$ on an open set, a polynomial field orthogonal to it
vanishes there, hence identically. If $r\le m-2$ the span has dimension at most
$r+1<m$ everywhere, and (v) or its analogue with fewer shapes gives nonzero
orthogonal polynomial fields. For genericity when $r=m-1$, the tuples with
$\det[\,x\,|\,\Xi_1x\,|\cdots|\,\Xi_{m-1}x\,]\equiv0$ form an algebraic subset, and
it is proper: for $\Xi_j=\Xi^j$ with $\Xi=\operatorname{diag}(\lambda_1,\dots,\lambda_m)$
and distinct $\lambda_i$, the determinant is $\prod_ix_i$ times the Vandermonde
determinant of the $\lambda_i$, not identically zero.

(v) For $r=m-2$ the pointwise orthogonal complement of the span is generically
the line through $w(x)$, whose entries are the maximal minors of
$[\,x\,|\,\Xi_1x\,|\cdots|\,\Xi_{m-2}x\,]$, homogeneous of degree $m-1$ and without
common polynomial factor for generic shapes. A polynomial field proportional to
$w$ at every $x$ is then $p\,w$ with $p$ a polynomial, of degree
$(q-1)-(m-1)=q-m$.
\end{proof}

Prop.~\ref{prop:gauge-structure}(ii) bounds the full gauge by its leading-degree
part but does not show that leading elements extend to full kernel elements,
because the lowering part of a degree-$(q+2)$ component can cancel the raising
part of a degree-$q$ one. Sharpness therefore needs an explicit element.

\begin{lemma}[An explicit residual gauge field]\label{lem:wedge}
Let $r=m-2$ and define $w_i(x):=\det[\,e_i\,|\,x\,|\,\Xi_1x\,|\cdots|\,\Xi_{m-2}x\,]$.
Then $\nabla\!\cdot w=0$, $x\cdot w=0$ and $\Xi_kx\cdot w=0$ for every $k$, so
$w\in\ker\A_K$; $w$ is a polynomial field of degree $m-1$, hence lies in
$\mathcal S_{\le m}$, and it is nonzero for generic shapes. More generally
$p(\varphi_0,\dots,\varphi_{m-2})\,w\in\ker\A_K$ for every polynomial $p$, where
$\varphi_0:=|x|^2/2$ and $\varphi_k:=x^\top \Xi_kx/2$.
\end{lemma}

\begin{proof}
For any $v$, $v\cdot w=\det[\,v\,|\,x\,|\,\Xi_1x\,|\cdots]$, which vanishes when
$v$ repeats a column; this gives the three orthogonality relations. For the
divergence, differentiate the determinant column by column:
$\partial_iw_i$ is a sum of terms $\sum_i\det[\,e_i\,|\cdots|\,Me_i\,|\cdots]$
with $M\in\{I,\Xi_1,\dots,\Xi_{m-2}\}$ in one column and the others fixed. Each such
sum is $\sum_i\Omega(e_i,Me_i)=\tr(\Omega M)$ for the antisymmetric bilinear
form $\Omega$ obtained by fixing the other columns, and it vanishes because $M$
is symmetric. Hence $\A_{\rho_\star}w=\nabla\!\cdot w-x\cdot w=0$ and
$\A_{\rho_k}w=\nabla\!\cdot w-w\cdot(I+\Xi_k)x=0$. For the multiplied field,
$\A_{\rho_\star}(pw)=\nabla p\cdot w+p\,\A_{\rho_\star}w=0$ because
$\nabla p\in\spn\{x,\Xi_1x,\dots,\Xi_{m-2}x\}$, and the cross-slice conditions are
unchanged. Nonvanishing: $w(x)=0$ iff the $m-1$ columns are dependent, which for
generic shapes fails on an open set.
\end{proof}

\begin{proof}[Proof of Thm.~\ref{thm:gauge}]
Sufficiency: under shape richness Prop.~\ref{prop:gauge-structure}(iv) gives
$\ker R_{K,q}=\{0\}$ for every $q$, and (ii) then gives
$\ker(T_K|_{\mathcal S_{\le Q}})=\{0\}$, which by (i) is the claim. Genericity is
(iv). Sharpness is Lem.~\ref{lem:wedge}: with $r=m-2$, i.e.\ $K=m-1$ distinct
shapes, a nonzero polynomial gauge direction survives.
\end{proof}

For $m=3$ and one shape change, $w(x)=x\times \Xi_1x$ is the vector field of
Euler's equations for a free rigid body. It is divergence-free, tangent to the
spheres $|x|=$ const, and tangent to the ellipsoids $x^\top \Xi_1x=$ const, which
is exactly why neither Gaussian slice can see it.

\subsection{One shape below the threshold}\label{app:below}

By Prop.~\ref{prop:gauge-structure}(ii) and (v), for $r=m-2$ the residual gauge
gains at most $D_{q-m}$ new directions at each degree $q$, so
$\dim(\mathcal S_{\le Q}\cap\ker\A_K)\le\sum_{q\le Q}D_{q-m}$. Since
$D_{q-m}\sim q^{m-1}/(m-1)!$ while $\dim S_q\sim(m-1)q^{m-1}/(m-1)!$, the bound
is a fraction tending to $1/(m-1)$ of $\dim\mathcal S_{\le Q}$ --- one half at
$m=3$. The bound
is attained in every case we computed. Solving the full linear system
$\{\nabla\!\cdot h-x\cdot h=0,\ h\cdot \Xi_kx=0\}$ on polynomial fields of degree
at most $\ell$ with random integer shapes gives
\begin{center}
\begin{tabular}{@{}lcl@{}}
\toprule
$m$, $r$ & degrees $\ell$ & $\dim(\mathcal S_{\le \ell+1}\cap\ker\A_K)$\\
\midrule
$3$, $0$ & $1$--$5$ & $3,11,26,50,85$ \ $=\sum_q\dim S_q$\\
$3$, $1$ & $1$--$8$ & $0,1,4,10,20,35,56,84$ \ $=\sum_q D_{q-3}$\\
$3$, $2$ & $1$--$8$ & $0$ throughout\\
$4$, $2$ & $1$--$6$ & $0,0,1,5,15,35$ \ $=\sum_q D_{q-4}$\\
$4$, $3$ & $1$--$6$ & $0$ throughout\\
\bottomrule
\end{tabular}
\end{center}
Whether equality holds in general is open. The fields of Lem.~\ref{lem:wedge}
account for only a small part of the residual gauge: their number grows like
$q^{m-2}$, against $q^{m-1}$ for the bound, so most residual directions have
leading part $p\,w$ with $p$ not a function of the quadratic invariants,
completed at lower degrees.

\emph{A degree cap below the threshold.} Fewer shapes do not remove
identifiability; they cap its degree. For $r\le m-2$ and constant vectors
$c_1,\dots,c_{m-r-2}$, the field
$(w_c)_i(x):=\det[\,e_i\,|\,x\,|\,\Xi_1x\,|\cdots|\,\Xi_rx\,|\,c_1\,|\cdots|\,
c_{m-r-2}\,]$ lies in $\ker\A_K$ by the proof of Lem.~\ref{lem:wedge} verbatim,
constant columns not entering the divergence; it is homogeneous of degree $r+1$
and nonzero for generic shapes and $c$. With $K=r+1$ distinct generic slices an
exact polynomial gauge of degree $K$ therefore survives. Conversely, if
$\ker R_{K,q}=\{0\}$ for every $q\le K$, then
$\mathcal S_{\le K}\cap\ker\A_K=\{0\}$ by
Prop.~\ref{prop:gauge-structure}(ii), and every polynomial drift of degree at
most $K-1$ is identified by the source constraints, since the difference of two
such drifts lies in that intersection. For $r=m-2$ this is (v), for every $m$.
For $r<m-2$ we certify it exactly for $m\le6$: for one draw of integer shapes the
homogeneous system of (iii) has full column rank modulo the prime $2^{31}-1$ at
every degree $q-1\le r$, which implies full rank over $\mathbb Q$ and hence,
full rank being a Zariski-open condition, for generic shapes; at degree $r+1$
its kernel is nonzero, as $w_c$ requires. The general case is open. The number
of distinct shapes thus sets the polynomial degree to which the drift is
identified, and $K\ge m$ removes the cap.

This is the algebraic counterpart of Obs.~\ref{obs:tension}. When the shapes
stop moving the failure is not a gradual loss of precision: one shape short of
the threshold a fixed fraction of every high-degree block reappears in the
gauge, and with no shape change ($r=0$) all of it does.

\section{Linearization, propagation, and the bridge assumptions}\label{app:transient}

\subsection{The tangent equation}

Write $\rho_t^\delta=\bar\rho_t(1+\delta u_t)+o(\delta)$ for the system with drift $F^\delta=\bar F+\delta g$.
Substituting into $\partial_t\rho=-\nabla\!\cdot(\rho F)+\sigma^2\Delta\rho$,
subtracting the reference equation and dividing by $\bar\rho_t$ gives, to first
order,
\[
    \partial_tu_t=\mathcal L_tu_t-\A_{\bar\rho_t}g,
    \qquad
    \mathcal L_tu:=-\bar F\cdot\nabla u+\sigma^2\bigl(\Delta u
      +2\bar s_t\cdot\nabla u\bigr),
\]
which is the first line of \eqref{eq:tangent}; with $u_0=0$ Duhamel's formula
gives the second. If the initial law is unknown, $u_0\ne0$ propagates as
$\mathcal U(t_k,0)u_0$, an additional nuisance to be profiled. The map
$g\mapsto(u_{t_1},\dots,u_{t_K})$ is the derivative of the snapshot map
$\mathscr S_K$ at $\bar F$, so local identifiability in the snapshot experiment
(\S\ref{sec:autonomy}) is a statement about the column rank of $\mathcal B_K$.

\begin{lemma}[Hermite truncation is closed under propagation]\label{lem:closure}
For the Gaussian reference, $\mathcal L_t$ maps $\mathcal P_q$ into
$\mathcal P_q\oplus\mathcal P_{q-2}$ and never raises degree. Hence
$\mathcal U(t,s)$ preserves $\mathcal P_{\le Q}$ and is invertible on it, and for
$g\in W_{\le Q}$ the Duhamel integral lies in $\mathcal P_{\le Q}$.
\end{lemma}

\begin{proof}
Each term of $\mathcal L_t$ is a differential operator with constant or linear
coefficients: $-\bar F\cdot\nabla=-\sum_{ij}\bar A_{ij}x_j\partial_i$ for a
linear reference drift, $\sigma^2\Delta$, and
$2\sigma^2\bar s_t\cdot\nabla=-2\sigma^2\sum_{ij}(\bar\Sigma_t^{-1})_{ij}x_j
\partial_i$. Each $\partial$ lowers degree by one and $x_j=\delta_j+\partial_j$
changes it by $\pm1$, so each term changes degree by $0$ or $-2$. A
degree-nonincreasing generator has a degree-nonincreasing, block lower
triangular propagator, invertible because its diagonal blocks are, and
$\A_{\bar\rho_s}$ maps $W_{\le Q}$ into $\mathcal P_{\le Q}$ by
Lem.~\ref{lem:split}.
\end{proof}

Lem.~\ref{lem:closure} says a \emph{truncated input} propagates exactly. It does
not say that low-degree observations exclude high-degree inputs, and they do
not (App.~\ref{app:guard}).

\subsection{How the snapshot kernel relates to the source kernel}

Three kernels are in play. If $\A_{\bar\rho_s}h=0$ for every $s$ --- the path
constraints --- then every Duhamel integrand vanishes and $\mathcal B_Kh=0$, so
the path kernel lies in $\ker\mathcal B_K$ and in $\ker\bar\A_K$. Between
$\ker\mathcal B_K$ and $\ker\bar\A_K$ there is no inclusion in general.

\emph{Source-invisible but snapshot-visible.} A direction with
$\A_{\bar\rho_{t_k}}h=0$ at every observed time can have
$\A_{\bar\rho_s}h\ne0$ for $s$ between them, and propagation accumulates that
footprint. The oracle design is informed by the continuum of shapes
$\{\bar\Sigma_s\}$ traversed along the path, whereas $T_{K,q}$ sees only the $K$
observed shapes; this is the precise sense in which an oracle that knows the
reference trajectory is stronger than an estimator that sees only snapshots.

\emph{Source-visible but snapshot-invisible.} Because $\mathcal U(t,s)$ is
invertible on each truncation (Lem.~\ref{lem:closure}), $\mathcal B_K$ can fail
to be injective on directions the source constraints see only through
cancellation in the $s$-integral --- an analytic condition on the path. That is
a first-order form of the aliasing in Ex.~\ref{ex:alias}, and it is what
the faithfulness condition of \S\ref{sec:how-fast} excludes, together with the analogous
possibility that the profiling projection in \eqref{eq:Jsol} cancels a
solenoidal direction exactly.

Numerically, for an autonomous Gaussian reference with
$\dot{\bar\Sigma}=\bar A\bar\Sigma+\bar\Sigma\bar A^\top+2\sigma^2I$, building
$\mathcal B_K$ by integrating the tangent equation from $u_0=0$ gives
$\ker T_{K,q}=\ker\mathcal J_{\mathrm{sol},q}=\{0\}$ in every shape-rich case
tested ($m=3$, $K=4$, $Q\le6$; $m=4$, $K=5$, $Q\le5$; two reference paths each),
so the faithfulness condition holds there with equality. Without shape
richness the inclusion is strict in the favourable direction: at $m=3$, $K=2$ the
source gauge has dimension $10$ while the profiled kernel has dimension $1$ and
$0$ on two paths, and at $m=4$, $K=2$ the dimensions are $18$ against $7$ and
$6$. The directions propagation recovers are recovered weakly --- the profiled
spectrum has $\kappa_{\min}/\kappa_{\max}\approx10^{-5}$ against $\approx10^{-1}$
for the source operator --- so shape-poor paths convert exact gauge into
near-gauge rather than into usable information.

\subsection{Propagated aliasing: the guard band does not transfer}
\label{app:guard}

The instantaneous source has only raising and lowering bands, and
$P_{\le Q-2}T_{>Q}=0$ for the source operator. This does \emph{not} justify the
same statement for $\mathcal B_K$: exponentiating a degree-lowering generator
produces arbitrarily many downward bands, so invariance of $\mathcal P_{\le Q}$
proves exact propagation of a truncated \emph{input} but not that projecting
observations to low degree excludes high-degree inputs.

The leakage is nonzero. Take the Brownian reference $\bar F=0$, $\sigma^2=1$,
$\bar\Sigma_t=(1+2t)I_2$, and with $H_j$ the unnormalized probabilists' Hermite
polynomials set $h=(x_2H_5(x_1),-H_6(x_1))\in S_7$, which satisfies
$\A_{\rho_\star}h=0$ at $\rho_\star=\mathcal N(0,I_2)$. The transient source is
$\A_{\bar\rho_t}h=5(1-\tfrac{1}{v})x_2H_4(x_1)$ with $v=1+2t$. Writing
$u_t=x_2[c_0(t)+c_2(t)H_2(x_1)+c_4(t)H_4(x_1)]$, the tangent generator gives
\[
    \dot c_4=-\tfrac{10c_4}{v}-5\bigl(1-\tfrac1v\bigr),\quad
    \dot c_2=-\tfrac{6c_2}{v}+12\bigl(1-\tfrac2v\bigr)c_4,\quad
    \dot c_0=-\tfrac{2c_0}{v}+2\bigl(1-\tfrac2v\bigr)c_2 .
\]
From zero initial conditions, $c_0(0.4)=-7.7213162\times10^{-3}$: a degree-$7$
field reaches scalar degree $1$. The degreewise benchmark \eqref{eq:sequence}
has no such term by construction, since it treats degrees as independent; in the
linearized experiment high-degree leakage is an additional bias, one of the
transfer steps of Rem.~\ref{rem:oracle-gap}, and its control is open
(App.~\ref{app:ledger}).

\subsection{Joint versus blockwise information}\label{app:joint}

Asm.~\ref{asm:comparability} bounds the jointly profiled information above by
the blockwise one. The reverse bound, which would be needed to carry the upper
bound of Thm.~\ref{thm:oracle-rate} from the benchmark to the tangent
experiment (App.~\ref{app:transfer}), is not assumed, and it can fail: good
information in every block does not control the joint inverse trace. With
\[
    A=\begin{pmatrix}1&1\\0&\delta\end{pmatrix},
\]
the two coordinates have individual information $1$ and $1+\delta^2$, whereas
$\tr[(A^\top A)^{-1}]=1+2/\delta^2$, which diverges as $\delta\to0$. Because
propagation is not Hermite-degree diagonal --- $\mathcal U(t,s)$ mixes degrees,
and profiling removes only $G_q$ within each block rather than gradients at
other degrees or guard-band nuisance fields --- overlapping output degrees can
in principle create exactly this configuration. Neither direction of the
comparison has been verified numerically, and both are open.

\section{Near-gauge spectral geometry and the thin tail}\label{app:spectral}

\subsection{From a thin tail to the inverse-trace law}

\begin{lemma}[Trace law]\label{lem:trace}
Under Asm.~\ref{asm:scale} and \ref{asm:thin-tail}, every $\kappa_{q,j}^2$ is
positive and
\[
    \frac{r_q}{\bar\mu_q}\;\le\;\tr\bigl(\mathcal J^{-1}_{\mathrm{sol},q}\bigr)
    \;\le\;\Bigl(1+\frac{C'}{\gamma-1}\Bigr)\frac{r_q}{\bar\mu_q}
    \;\asymp\;\frac{q^m}{K}.
\]
\end{lemma}

\begin{proof}
Since $\mathcal N_q(t)\le C't^\gamma\to0$ as $t\to0$ and $\mathcal N_q$ takes
values in $r_q^{-1}\mathbb Z$, no eigenvalue is zero. The lower bound is
Cauchy--Schwarz, $r_q^2=(\sum_j\kappa_{q,j}\kappa_{q,j}^{-1})^2\le
(\sum_j\kappa_{q,j}^2)(\sum_j\kappa_{q,j}^{-2})$. For the upper bound put
$t_j:=\kappa_{q,j}^2/\bar\mu_q$; the layer-cake formula gives
$r_q^{-1}\sum_jt_j^{-1}=\int_0^\infty r_q^{-1}\#\{j:t_j<1/u\}\,du
\le1+\int_1^\infty\mathcal N_q(1/u)\,du\le1+C'\int_1^\infty u^{-\gamma}du$,
finite exactly because $\gamma>1$. At $\gamma=1$ the integral is cut off at
$u\le C'r_q$, since $\mathcal N_q(t)<r_q^{-1}$ forces $\mathcal N_q(t)=0$, giving
an extra $\log r_q\asymp\log q$.
\end{proof}

The thin tail bounds no individual eigenvalue: what it excludes is a large
population of near-gauge directions, and it needs no degree-uniform spectral
gap.

\subsection{What has been measured, and on which operator}

The diagnostics bearing on Asm.~\ref{asm:scale} and \ref{asm:thin-tail} were
computed by direct construction in the tensor-Hermite basis, and they are not all
computed on the operator the assumptions concern. We separate them.

\emph{On the propagated, profiled operator $\mathcal J_{\mathrm{sol},q}$.}
Fitting $\tr(\mathcal J^{-1}_{\mathrm{sol},q})\sim q^p$ with propagation and
gradient profiling included gives $p=3.20$ at $m=3$, $4.08$ at $m=4$ and $4.76$
at $m=5$, each with $r=m-1$ shapes, against the predicted $p=m$. This is the
\emph{conclusion} of Lem.~\ref{lem:trace}, measured where it is used, and it is
the evidence the upper bound actually rests on.

\emph{On the source operator $T^\ast_{K,q}T_{K,q}$.} The tail exponent
itself has been measured only here: $\gamma\approx1.6$ at $m=3$ and
$\gamma\approx2.0$--$2.2$ at $m=4$ (degrees up to $40$ and $14$), comfortably above
the threshold. For the source operator the inverse trace scales as $q^{m-2}$,
and one shape short of richness it drifts only logarithmically while two short it
gains a power. Asm.~\ref{asm:thin-tail} is stated for
$\mathcal J_{\mathrm{sol},q}$, so these source-level exponents support it only
through the structural relation between the two operators; they are not a
measurement of it.

\emph{Why the symbol does not settle it.} The raising part of the source
operator alone has a semiclassical symbol whose lower-tail exponent is
$\gamma_{\mathrm{sym}}=1/2$, below the threshold, and the exponent measured for
the raising part drifts downward toward that value as $q$ grows. What restores
$\gamma>1$ for the full source operator is the lowering part of
Lem.~\ref{lem:split}, which the symbol discards. That is also where a proof would
have to come from, and no fixed-degree computation substitutes for a bound
uniform in $q$.

\section{The tangent benchmark: reduction, rate proofs, and transfer}
\label{app:rate-proofs}

\subsection{Reduction to a Gaussian shift, and comparison}

In the tangent snapshot experiment the gradient part of $g$ is an unrestricted
nuisance. For the lower bound it suffices to consider, at a cutoff $Q$, the
subproblem in which $h\in\mathcal S_{\le Q}$, the gradient
directions of degree at most $Q$ are unrestricted, and all components above
degree $Q$ vanish: its parameter set is contained in the full one, so its
minimax risk is no larger. With standard Gaussian noise the subproblem is a
linear Gaussian model with unrestricted linear nuisance. Projecting $Y$ onto
the orthocomplement of the range of the retained gradient columns leaves a
Gaussian shift for $h$ with information operator
$\mathcal J_{\mathrm{sol},\le Q}$ of Asm.~\ref{asm:comparability}, and this
reduction loses nothing for minimax risk, by the usual limit of least
favourable priors on the nuisance with growing variance.

\begin{lemma}[Comparison of Gaussian shifts]\label{lem:compare}
For $J\succeq0$ on a finite-dimensional parameter space, let $\mathcal E_J$ be
the Gaussian shift with sufficient statistic $S\sim\mathcal N(J\theta,J)$. If
$J_2\preceq J_1$, then $\mathcal E_{J_2}$ is a randomization of $\mathcal
E_{J_1}$, so every minimax risk in $\mathcal E_{J_1}$ is at most the
corresponding one in $\mathcal E_{J_2}$. For estimators linear in $S$ the same
conclusion holds when the Gaussian law is replaced by any noise with the same
mean and covariance.
\end{lemma}

\begin{proof}
If $J_1x=0$ then $0\le x^\top J_2x\le x^\top J_1x=0$, so $J_2x=0$: hence
$\ker J_1\subseteq\ker J_2$ and $J_2J_1^+J_1=J_2$, with $J_1^+$ the
pseudo-inverse. Put $M:=J_2J_1^+$. Then $MS_1$ has mean $J_2\theta$ and
covariance $J_2J_1^+J_2$, and $J_2-J_2J_1^+J_2=J_2^{1/2}(I-J_2^{1/2}J_1^+
J_2^{1/2})J_2^{1/2}\succeq0$, because $J_2\preceq J_1$ gives
$(J_1^+)^{1/2}J_2(J_1^+)^{1/2}\preceq I$ and the nonzero spectra of $XX^\ast$
and $X^\ast X$ coincide for $X=(J_1^+)^{1/2}J_2^{1/2}$. Adding independent
Gaussian noise with that covariance yields exactly the law of the sufficient
statistic of $\mathcal E_{J_2}$, so any estimator in $\mathcal E_{J_2}$ can be
run in $\mathcal E_{J_1}$ with the same risk. The risk of a linear estimator
depends only on the first two moments, which gives the last claim. Scaling $J$
by a constant is the same as scaling $n$.
\end{proof}

Under Asm.~\ref{asm:comparability} the subproblem at cutoff $Q$ is therefore a
randomization of the benchmark \eqref{eq:sequence}, restricted to degrees at
most $Q$, with $n$ replaced by $Cn$. The comparison is stated at every cutoff
and one-sidedly because that is exactly what the lower bound uses. A statement
about inverse traces alone would not suffice: comparable traces do not prevent
the joint experiment from being more informative along the weak directions on
which the lower bound is built.

Two conventions. The weights $(1+q)^{2s}$ define the smoothness class and are not
identified with a conventional $s$-derivative norm. And $\theta$ are coordinates
of $h$ itself: if one writes $F^\delta=\bar F+\delta g$ with $g$ fixed, the signal
is $\delta\kappa\theta$, $n$ becomes $n\delta^2$, and the risk must be
rescaled accordingly. Coordinates with $\kappa_{q,j}=0$ are unidentified, which
only helps the lower bound.

\subsection{Where the \texorpdfstring{$K/q$}{K/q} scale comes from}\label{app:scale}

Asm.~\ref{asm:scale} is assumed, not proved. The heuristic behind it is a
factorization of each slice's contribution to $\tr\mathcal J_{\mathrm{sol},q}$
into a source strength and a propagator gain.

\emph{Source strength $\asymp q$.} For $h\in S_q$ the footprint at time $s$ is
$\A_{\bar\rho_s}h=-h^\top\Xi_sx$ with $\Xi_s:=\bar\Sigma_s^{-1}-I$. By
Lem.~\ref{lem:split} its raising part multiplies by $x$, and on degree-$q$
Hermite modes multiplication by $x_j$ has norm $\asymp\sqrt q$, so
$\|\A_{\bar\rho_s}h\|^2\lesssim\|\Xi_s\|^2q\,\|h\|^2$. A matching lower bound
uniform over $S_q$ is exactly what shape richness does not quantify: it makes
$R_{K,q}$ injective (Prop.~\ref{prop:gauge-structure}) without bounding its
smallest singular value.

\emph{Propagator gain $\asymp q^{-2}$ per slice.} For an Ornstein--Uhlenbeck-type
reference with relaxation rate $\nu>0$, the diagonal part of $\mathcal L_t$
on $\mathcal P_q$ is $\approx-q\nu$, so on degree-$q$ components
$\int_0^{t_k}\mathcal U(t_k,s)\,ds\approx(q\nu)^{-1}$ once
$q\nu t_k\gg1$, and the squared gain is $\asymp q^{-2}$. Summing over the $K$
observed slices gives the factor $K$, provided each slice sees a shape
perturbation $\Xi_s$ of order one on the preceding window of length $\asymp
(q\nu)^{-1}$.

Multiplying, $\bar\mu_q\asymp q\cdot K/q^2=K/q$. The gaps are explicit: the
uniform lower bound on the source strength, uniformity of the damping along a
transient path, the requirement that every slice be preceded by genuine shape
change --- which fails in the short-gap design of Obs.~\ref{obs:tension} ---
and a non-degeneracy of profiling against $G_q$, which can only decrease the
trace and could in principle decrease it by more than a constant. Without
diffusion there is no damping and the factorization changes. Numerically, the
trace law that Asm.~\ref{asm:scale} and \ref{asm:thin-tail} jointly imply is
observed on the propagated design (App.~\ref{app:spectral}).

\begin{corollary}[Bulk consequence]\label{cor:bulk}
Under Asm.~\ref{asm:scale}, at least half the degree-$q$ coordinates satisfy
$\kappa^2_{q,j}\le2\bar\mu_q\lesssim K/q$.
\end{corollary}

\begin{proof}
Markov's inequality for the empirical measure of $\{\kappa^2_{q,j}\}_j$, whose
mean is $\bar\mu_q$. No tail hypothesis is used.
\end{proof}

\subsection{Upper bound}

Estimate $\theta_{q,j}$ by $Z_{q,j}/\kappa_{q,j}$ for $q\le Q$ and by zero
beyond; this is well defined because Asm.~\ref{asm:thin-tail} excludes zero
eigenvalues (Lem.~\ref{lem:trace}). The risk is
\[
    \sum_{q\le Q}\frac1n\tr\bigl(\mathcal J^{-1}_{\mathrm{sol},q}\bigr)
    +\sum_{q>Q}\|\theta_q\|^2
    \;\lesssim\;\sum_{q\le Q}\frac{q^m}{nK}+R^2(1+Q)^{-2s}
    \;\asymp\;\frac{Q^{m+1}}{nK}+R^2Q^{-2s},
\]
using Lem.~\ref{lem:trace} for the first term and the Sobolev constraint in
\eqref{eq:sequence} for the second. Taking $Q_\star\asymp(nK)^{1/(2s+m+1)}$ gives
risk $\asymp(nK)^{-2s/(2s+m+1)}$. At $\gamma=1$, Lem.~\ref{lem:trace} carries an
extra $\log q$, which costs a logarithmic factor in the rate. This is the only
place Asm.~\ref{asm:thin-tail} is used, and since the estimator is linear the
calculation uses only second moments of the noise.

\subsection{Lower bound}

The lower bound uses only Cor.~\ref{cor:bulk}. Fix a small constant $c>0$ and
consider the band of degrees $q\in[Q_\star/2,Q_\star]$ with
$Q_\star:=(nK)^{1/(2s+m+1)}$. In each such degree keep the at-least-half of the
coordinates with $\kappa^2_{q,j}\le2\bar\mu_q$, whose noise variance
$1/(n\kappa^2_{q,j})$ is $\gtrsim q/(nK)$, and let $\theta_{q,j}$ range over
$[-\tau_q,\tau_q]$ with $\tau_q^2:=c\,q/(nK)$, setting all other coordinates to
zero. For $c$ small this hyperrectangle lies in the Sobolev ball, since
\[
    \sum_{q\in[Q_\star/2,Q_\star]}(1+q)^{2s}\,r_q\,\tau_q^2
    \;\lesssim\;c\,Q_\star^{2s}\cdot Q_\star\cdot Q_\star^{m-1}\cdot
    \frac{Q_\star}{nK}
    \;=\;c\,\frac{Q_\star^{2s+m+1}}{nK}\;=\;c\;\le\;R^2 .
\]
The minimax risk over a hyperrectangle in a Gaussian sequence model is at least
a constant multiple of $\sum_j\min(\tau_j^2,\sigma_j^2)$
\citep{dlm1990,ik1981}, and here $\tau_q^2\lesssim\sigma_{q,j}^2$ on every kept
coordinate, so the risk is
\[
    \gtrsim\sum_{q\in[Q_\star/2,Q_\star]}\frac{r_q}{2}\cdot\frac{c\,q}{nK}
    \;\asymp\;\frac{Q_\star^{m+1}}{nK}\;\asymp\;(nK)^{-2s/(2s+m+1)} .
\]
Restricting to a single degree $q\asymp Q_\star$ would give only
$Q_\star^m/(nK)$, short of the rate by a factor $Q_\star$; the band of
$\asymp Q_\star$ degrees is what recovers it.

\subsection{Proof of Thm.~\ref{thm:oracle-rate}}

\emph{Lower bound.} The hyperrectangle of the lower-bound subsection lives in
degrees $q\le Q_\star$. By the reduction above, the minimax risk of the tangent
experiment is at least that of the subproblem at cutoff $Q_\star$, a Gaussian
shift with information $\mathcal J_{\mathrm{sol},\le Q_\star}$. By
Asm.~\ref{asm:comparability} and Lem.~\ref{lem:compare} this is at least the
minimax risk of \eqref{eq:sequence} restricted to degrees at most $Q_\star$ at
sample size $Cn$, which the lower-bound subsection bounds below by a constant
multiple of $(CnK)^{-2s/(2s+m+1)}$. Constants in $n$ do not change the rate.

\emph{Upper bound.} This is the upper-bound subsection, stated in
\eqref{eq:sequence}; the estimator is linear, so only the first two moments of
the noise enter.

In both parts the $L^2(\rho_\star)$ loss equals the coordinate loss, because the
eigenbases are orthonormal in $L^2(\rho_\star)$ and the Sobolev weights are
constant on each degree.

\subsection{Where the \texorpdfstring{$+1$}{+1} comes from}

Direct regression on the $\asymp Q^m$ Hermite coefficients up to degree $Q$,
each with unit inverse information, gives exponent $2s+m$. Here each degree-$q$
direction has inverse information $\asymp q/K$, one power of $q$ worse, which
multiplies the cumulative variance by $Q$ and gives $2s+m+1$. In the
factorization of App.~\ref{app:scale}, the propagator's degree-proportional
damping costs two powers and the degree-raising footprint returns one. This is
bookkeeping under Asm.~\ref{asm:scale}; the proof is the two bounds above.

\subsection{From the benchmark to the tangent experiment, and to sampled snapshots}
\label{app:transfer}

The upper bound of Thm.~\ref{thm:oracle-rate} is stated in the benchmark.
Carrying it to the tangent experiment along the route of this appendix would
need two further ingredients. The first is the reverse comparison
$\mathcal J_{\mathrm{sol},\le Q}\succeq c\,\mathcal J_Q^{\mathrm{blk}}$, uniformly
in $Q$; with it, the second-moment part of Lem.~\ref{lem:compare} transfers the
variance of the cutoff estimator with $n$ replaced by $cn$. It is not assumed,
and it can fail (App.~\ref{app:joint}). The second is control of omitted-degree
leakage. Components above the cutoff reach the retained statistics through
propagation, which lowers degree (App.~\ref{app:guard}), so the bias of a cutoff
estimator is not just the Sobolev tail. For the solenoidal part this is a
boundedness condition on the leakage map; for the gradient part it cannot hold
while the nuisance is unrestricted in every degree, so the gradient nuisance
above the cutoff must be either restricted or profiled jointly with the retained
degrees. Neither ingredient is established.

The nonlinear sampled-snapshot experiment observes
$X_i^{(k)}\stackrel{\mathrm{iid}}{\sim}\rho_{t_k}^F$. Transfer from the tangent
experiment would need uniform control of the linearization remainder and of the
sampling likelihood through the growing cutoff. A direct lower bound, for
example, would need
$\mathrm{KL}(\rho_{t_k},\rho'_{t_k})
=\tfrac12\|u_k-u'_k\|_{L^2(\bar\rho_{t_k})}^2(1+o(1))$ uniformly over the
hypercube used in the proof. Finally, the tangent oracle is credited with the
reference path entering $\mathcal B_K$ up to each observation time, whereas
$T_{K,q}$ uses only the source footprints at the observed slices: the oracle
knows shapes that the data never show.

\subsection{Proof of the covariance-only audit}

\begin{proof}[Proof of Prop.~\ref{prop:hessian}]
Let the slices be centered Gaussians $\mathcal N(0,\Sigma_k)$ and the field class
linear, $F=Hx$. By Lem.~\ref{lem:linear} the footprint at slice $k$ is
$\A_{\rho_k}(Hx)=\tr H-x^\top H^\top\Sigma_k^{-1}x$, affine in $H$, and $b_k$
does not depend on $H$, so the Hessian of
$\frac1{2K}\sum_k\E_k[(\A_{\rho_k}(Hx)-b_k)^2]$ is the quadratic form
$G\mapsto\frac1K\sum_k\E_k[(\A_{\rho_k}(Gx))^2]$. Put $M:=G^\top\Sigma_k^{-1}$,
which is not symmetric in general, and $M_s:=(M+M^\top)/2$. Since
$x^\top Mx=x^\top M_sx$ and $\E_k[x^\top Mx]=\tr(M\Sigma_k)=\tr G$, the footprint
$\A_{\rho_k}(Gx)=-(x^\top M_sx-\E_kx^\top M_sx)$ is centered. The consolidated
fourth-moment identity
$\E[(x^\top Ax)(x^\top Bx)]=\tr(A\Sigma)\tr(B\Sigma)+2\tr(A\Sigma B\Sigma)$ from
Isserlis' theorem holds for \emph{symmetric} $A,B$ --- for general $A,B$ the last
term is $\tr(A\Sigma B\Sigma)+\tr(A\Sigma B^\top\Sigma)$ --- so it must be
applied to $M_s$, not $M$:
\[
    \E_k[(\A_{\rho_k}(Gx))^2]=2\tr(M_s\Sigma_kM_s\Sigma_k)
    =\tr(G^2)+\tr(G^\top\Sigma_k^{-1}G\Sigma_k).
\]
Every term is a function of $\Sigma_k$ alone. With $\vec G:=\operatorname{vec}(G)$
column-stacked, $\operatorname{vec}(\Sigma_k^{-1}G\Sigma_k)
=(\Sigma_k\otimes\Sigma_k^{-1})\vec G$ and $\tr(G^2)=\vec G^\top\mathsf C_m\vec G$, so the
Hessian is $\frac1K\sum_k[\Sigma_k\otimes\Sigma_k^{-1}+\mathsf C_m]$; in
row-stacked coordinates the Kronecker factors appear in the opposite order.
Profiling against the symmetric block is the Schur complement. As a check, the
exact gauge $G=\Sigma_kN$ with $N^\top=-N$ (Lem.~\ref{lem:linear}) gives
$\tr(\Sigma_kN\Sigma_kN)+\tr(N^\top\Sigma_kN\Sigma_k)=0$, a null direction of the
slice-$k$ term, as it must be.
\end{proof}

The Gaussian hypothesis is essential and its failure is not subtle. In one
dimension with $g(x)=x$, a variance-one Gaussian gives
$\E[(\A_\rho g)^2]=2$, whereas a variance-one rescaling of a density
proportional to $e^{-ax^4}$ gives $4$. The two densities have the same
covariance, so no covariance-determined expression can distinguish them: for a
general density, linearity of $Hx$ does not make $\A_\rho(Hx)$
covariance-determined. Nonzero means and an arbitrary weighting operator add
further dependence. Prop.~\ref{prop:hessian} is therefore a second-moment
\emph{design} audit under a Gaussian reference, not a general information audit;
a GLS information audit for the weak form needs fourth moments as well.

\section{The Gaussian specialization}\label{app:gauss}

\subsection{The commutant criterion}

Take the compared slices Gaussian with a common mean, whitened as in
\S\ref{sec:how-much} so that $\rho_\star=\mathcal N(0,I_m)$ and the others are
$\mathcal N(0,\Sigma_k)$; the statements below are exact. Restrict to
$q=2$, so
$h(x)=Gx$ with $G^\top=-G$ by Lem.~\ref{lem:linear}. With
$\Xi_k=\Sigma_k^{-1}-I$, the cross-slice footprint is
$T_{K,2}h=-(x^\top G^\top \Xi_kx)_k=(x^\top G\Xi_kx)_k$, which vanishes for all $x$
iff the symmetric part of $G\Xi_k$ vanishes, i.e.
\[
    G\Xi_k+\Xi_k^\top G^\top=G\Xi_k-\Xi_kG=[G,\Xi_k]=0
    \iff[G,\Sigma_k^{-1}]=0 .
\]
So the degree-two gauge is the commutant of the observed precision family
inside the antisymmetric matrices. It is $\{0\}$ exactly when the $\Sigma_k^{-1}$
have no common nontrivial antisymmetric commutant --- and at $q=2$ a single
generic shape difference already suffices: a symmetric matrix with distinct
eigenvalues has only symmetric commutant, whose intersection with the
antisymmetric matrices is trivial, so $K=2$ is enough. The $K\ge m$ threshold of
Thm.~\ref{thm:gauge} is therefore \emph{not} driven by $q=2$; it comes from
shape richness (Thm.~\ref{thm:gauge}), which needs $r\ge m-1$ and governs all
degrees jointly. The theorem gives a sufficient condition uniform in $q$, and
the low-degree blocks are cleared sooner. Both statements are verified in
App.~\ref{app:verification}.

The common-mean hypothesis is where the degree separation lives. For a general
mean, an affine field $Hx+c$ is invisible under $\mathcal N(\mu,\Sigma)$ iff
$H\mu+c=0$ and $H\Sigma+\Sigma H^\top=0$: with $y:=x-\mu$,
$\A_\rho(Hx+c)=\tr H-y^\top H^\top\Sigma^{-1}y-(H\mu+c)^\top\Sigma^{-1}y$,
and the three terms have different degrees in $y$. Differing means across
slices mix these degrees, which is why Thm.~\ref{thm:gauge} assumes a common
one.

\subsection{Comparison with \texorpdfstring{\citet{guan2024}}{Guan et al. (2024)}}

By Lem.~\ref{lem:linear} the $\rho$-solenoidal linear fields are exactly the
$\Sigma$-generalized rotations of \citet{guan2024}, so the algebra agrees. The
statistical experiments do not, in three ways, and the comparison is worth
making precisely because it is easy to overstate.

\emph{What is conditioned on.} Their criterion is a property of the
\emph{initial law}: they characterize when the pair (drift, diffusion) is
identifiable for the \emph{whole} linear class given $p_0$, and a Gaussian
$p_0$ is always auto-rotationally invariant, so their theorem declines to
certify it. Ours is a property of a fixed \emph{family of Gaussian shapes} --- the reference
slices in the local analysis, the observed marginals when these are Gaussian ---
for a given
system: we ask whether the particular $\Sigma_k$ move. There is no
contradiction --- their non-identifiability witness has zero base drift, so
every traversed shape commutes with the perturbation and shape richness fails
--- but the two statements are not the same statement, and ours is the
per-system refinement their own limitations section leaves open.

\emph{What is estimated.} They identify drift \emph{and} diffusion jointly.
We take $\sigma$ known throughout, which removes their second clause, the
mean-zero Gaussian cross-section, since that clause encodes a drift--diffusion
rescaling ambiguity that a known $\sigma$ eliminates. Our comparison is
therefore only with the rotational clause.

\emph{What is observed.} Their result uses marginals on a continuum of times;
ours uses $K$ discrete shapes, and Ex.~\ref{ex:alias} shows the discrete case
admits aliasing that the continuum case does not. We therefore do not claim to
recover their theorem, and we do not claim our criterion implies theirs.

\subsection{Fibrewise conditions become algebraic}

Under a fixed finite-dimensional model class the infinite-dimensional
difficulty of Prop.~\ref{prop:count} disappears: with $F=Hx$ the gauge directions are $h=Gx$, $G^\top=-G$, and the conditions
$h(x)\in\mathcal C_x$ become $[G,\Sigma_k^{-1}]=0$, finitely many linear
equations in the $\binom m2$ free parameters of $G$. This is why the linear
benchmark of \S\ref{sec:exp} has an exactly computable gauge and why
Prop.~\ref{prop:hessian} can refuse a design before any field is fitted.

\section{Intrinsic dimension and the normal scale}\label{app:intrinsic}

The rate of Thm.~\ref{thm:oracle-rate} is governed by $m$, not the ambient
dimension $d$. Suppose the slices concentrate near an $m$-dimensional affine
subspace $E$ with normal width $\tau_\perp$, so that
$\Sigma_k=\Sigma_k^\parallel\oplus\tau_\perp^2I_{d-m}$ with
$\Sigma_k^\parallel$ of order one. Two consequences follow, and we keep the
diffusion amplitude $\sigma$ and the normal width $\tau_\perp$ notationally
distinct because they have different physical dimensions.

First, by Prop.~\ref{prop:symmetry} every rotation supported in the isotropic
normal block is an exact gauge, so the estimable antisymmetric parameters
number $\binom m2$ rather than $\binom d2$, and the remaining
$\binom{d-m}2$ directions are declared, not estimated. Second, the typical
score magnitude in the normal directions is $O(\tau_\perp^{-1})$ and the
corresponding block of the information is $\Theta(\tau_\perp^{-2})$; these are
a score magnitude and a precision respectively, and conflating them is how the
normal block acquires a spurious reputation for carrying information about the
field. It carries information about the \emph{support}, not the circulation, and
recovery is correspondingly flat in $d$ (Fig.~\ref{fig:intrinsic}).

The subspace is recovered rather than assumed: the isotropic bulk of each
$\Sigma_k$ is its maximal-multiplicity eigenvalue cluster, and the curl-visible
subspace is the average of the complements' projectors across $k$. This is
basis-free, and \S\ref{sec:exp-ablations} reports that a random orthogonal
rotation of all $12$ ambient coordinates leaves the recovered tangent space
correct to a principal angle of $2.58^\circ$ and the solenoidal error
essentially unchanged ($0.6373$ against $0.6361$, both under the velocity
convention of App.~\ref{app:log}).

\begin{figure}[htb]
\centering
\includegraphics[width=0.55\linewidth]{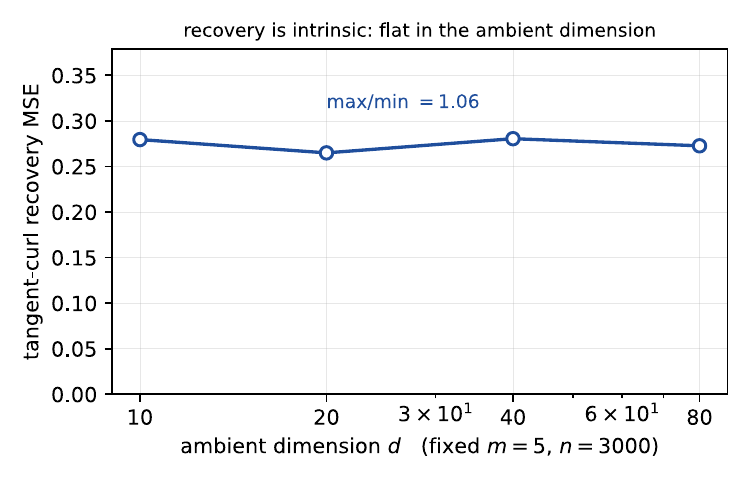}
\caption{\textbf{Recovery is intrinsic.} Tangent-curl recovery mean squared
error against ambient dimension $d$ at fixed intrinsic $m=5$ and $n=3000$, over
$300$ trials per point with the parametric estimator. The error is flat in $d$:
the ratio of largest to smallest across $d\in\{10,20,40,80\}$ is close to one.
Ambient dimension enters the cost of assembly, not the statistical difficulty of
the solenoidal channel.}
\label{fig:intrinsic}
\end{figure}

We do not claim a solution to intrinsic-dimension selection. The audit above
assumes a spectral gap between tangent and normal blocks; when that gap is
absent the cluster is ill-defined and the method has no principled fallback.
Nonlinear intrinsic geometry is discussed in App.~\ref{app:manifold}.

\section{Manifold concentration and rate separation}\label{app:manifold}

When the slices concentrate near a curved $m$-dimensional submanifold rather
than an affine subspace, the score acquires a large normal component that
encodes support geometry rather than dynamics. \citet{li2025scores} show that
near a smoothed manifold the score is dominated by the normal projection at
scale $\tau_\perp$, with the tangential component an order smaller. In our
setting this produces a clean separation.

\begin{proposition}[Rate separation under manifold concentration]
\label{prop:ratesplit}
In the smoothed-manifold regime with normal width $\tau_\perp\to0$ and
curvature bounded on the region carrying the mass, the footprint decomposes as
$\A_\rho g=\A^\parallel_\rho g^\parallel
-\tau_\perp^{-2}\,r\,\langle g^\perp,e_\perp\rangle+O(1)$, where $r$ is the signed
normal coordinate, $e_\perp$ the outward unit normal, $\A^\parallel$ the intrinsic Stein
operator, and the $O(1)$ term is controlled by the curvature bound and the $C^1$
norm of $g$. Since $r=O(\tau_\perp)$ on the mass, the normal term has size
$\tau_\perp^{-1}$ and its block of the information is $\Theta(\tau_\perp^{-2})$,
against $\Theta(1)$ tangentially. Hence the normal component of $g$ is identified at rate $\tau_\perp^{2}$ times the
tangential rate, and the solenoidal estimand, which is tangential by
Prop.~\ref{prop:symmetry}, is governed by the intrinsic operator alone.
\end{proposition}

\begin{proof}[Sketch]
Write $\rho$ in tubular coordinates as a product of an intrinsic density and a
Gaussian factor of width $\tau_\perp$ in the normal fibre. Then
$\nabla\log\rho=\nabla^\parallel\log\rho^\parallel
-\tau_\perp^{-2}re_\perp+O(1)$ with $r$ the signed normal coordinate, and
substituting into $\A_\rho g=\nabla\!\cdot g+g\cdot\nabla\log\rho$ gives the
stated split; the curvature bound controls the cross terms at $O(1)$ --- the
mean-curvature term from the volume element and the $O(r)$ curvature distortion of
tangential derivatives off the manifold. (For higher codimension read $re_\perp$ as
the normal displacement.) Because $\E[r^2]=\tau_\perp^2$ under the normal
factor, $\E[(\tau_\perp^{-2}r\langle g^\perp,e_\perp\rangle)^2]
=\tau_\perp^{-2}\,\E\langle g^\perp,e_\perp\rangle^2$: the normal block of the
information is $\Theta(\tau_\perp^{-2})$, which is the claimed separation.
Without the factor $r$ it would be $\Theta(\tau_\perp^{-4})$.
\end{proof}

Two cautions. The Euclidean diffusion in \eqref{eq:sde} is not silently
projected onto an estimated manifold; a genuinely intrinsic treatment requires
intrinsic differential operators and boundary conditions, which we do not
develop. And the separation is a statement about scales, not a selection
procedure: choosing $m$ and $\tau_\perp$ from data remains open.

Numerically, on a curved benchmark with normal width $\tau_\perp$ swept over
$1$, $0.1$, $0.03$, $0.01$, the within-manifold and moving-frame readouts
separate as predicted. At $\tau_\perp=1$ the two are comparable
($0.210\pm0.030$ and $0.201\pm0.087$ over five seeds); at $\tau_\perp=0.01$
with $10\%$ score error they differ by more than an order of magnitude
($0.954\pm0.085$ against $0.028\pm0.030$), and the score-free weak form gives
$0.116$ and $0.054$ on the same design.

\section{Stationary separable-OU specialization}\label{app:temporal-ou}

For a stationary reference with separable space--time structure the propagated
operator factorizes, and the temporal channel can be analysed in closed form.
Let the reference be an OU process at stationarity, so $\bar\rho_t\equiv
\bar\rho$, and let the temporal profile of the perturbation be $\varpi(t)$ with
relaxation rate $\nu$. Then $\mathcal B_K$ separates into a spatial factor,
which is the stationary Stein operator, and a temporal factor determined by
$\varpi$.

This regime is a warning, not a target. With $\bar\rho_t$ constant, all
snapshot shapes coincide, $r(x)\equiv0$ in Prop.~\ref{prop:count}, and
Thm.~\ref{thm:gauge} gives no contraction at all: the entire degree-$q$ gauge
$S_q$ survives. The separable calculation therefore describes how much
\emph{temporal} information a shared field supplies when the spatial shapes
supply none, and it should not be read as a rate for the transient problem.

\subsection{Cost of estimating the temporal spectrum}

Take one spatial mode of amplitude $\theta$ that is an eigenfunction of the
stationary tangent generator with eigenvalue $-\nu$. From $u_0=0$, the Duhamel
formula \eqref{eq:tangent} gives it the temporal profile
$\varpi(t;\nu)=(1-e^{-\nu t})/\nu$, so slice $k$ carries $\theta\,\varpi(t_k;\nu)$.
Suppose $\nu$ is itself unknown. With $\varpi:=(\varpi(t_k;\nu))_{k\le K}$ and
$\varpi_\nu:=(\partial_\nu \varpi(t_k;\nu))_{k\le K}$ in $\R^K$, the joint Fisher
information for $(\theta,\nu)$ is proportional to
$\bigl(\begin{smallmatrix}\|\varpi\|^2&\theta\langle \varpi,\varpi_\nu\rangle\\
\theta\langle \varpi,\varpi_\nu\rangle&\theta^2\|\varpi_\nu\|^2\end{smallmatrix}\bigr)$, and
profiling out $\nu$ multiplies the variance of $\widehat\theta$ by
\begin{equation}\label{eq:inflation}
    \frac{1}{1-\rho_{\mathrm{temp}}^2},
    \qquad
    \rho_{\mathrm{temp}}
    :=\frac{\langle \varpi,\varpi_\nu\rangle}{\|\varpi\|\,\|\varpi_\nu\|},
\end{equation}
where $\rho_{\mathrm{temp}}$ is the cosine between the profile and its
derivative in the \emph{parameter} $\nu$, not in time, and
$\rho_{\mathrm{temp}}^2$ is the squared cosine. With $K\ge2$ distinct times
$\varpi_\nu/\varpi$ is not constant in $k$, so $\rho^2_{\mathrm{temp}}<1$ and the
inflation is finite for every $\nu t_1$. It is nevertheless large once the mode
has relaxed before the first observation: the slices then see it only near its
plateau $\theta/\nu$, in which amplitude and rate are confounded, and separating
them requires the transient, of size $e^{-\nu t_1}$. Numerically the inflation
tracks $e^{2\nu t_1}/(\nu t_1)^2$ up to a constant (Fig.~\ref{fig:inflation}).

\begin{figure}[htb]
\centering
\includegraphics[width=0.55\linewidth]{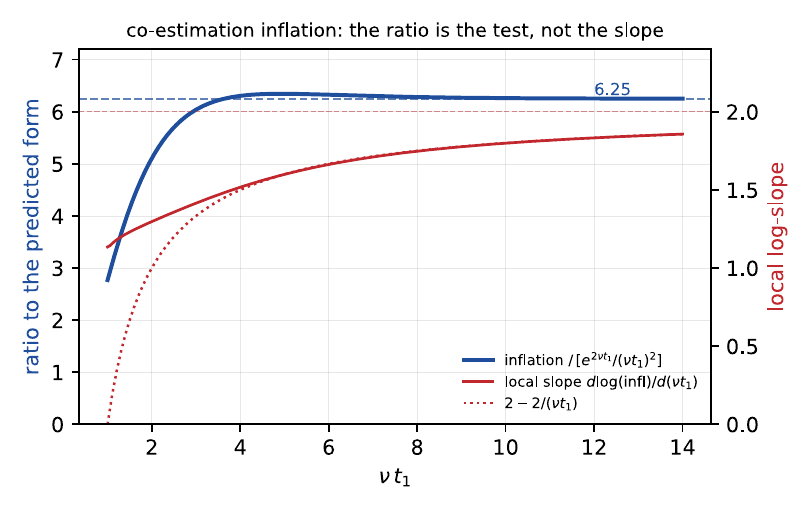}
\caption{\textbf{Cost of profiling the temporal nuisance.} Ratio of the exact
inflation \eqref{eq:inflation} to the surrogate $e^{2\nu t_1}/(\nu t_1)^2$
against $\nu t_1$, for $K=5$ observation times with
$t_k/t_1\in\{1,1.6,2.4,3.4,5\}$: the ratio settles at ${\approx}6.25$ for
$\nu t_1\gtrsim4$ while the inflation grows by more than seven orders of
magnitude, so the surrogate captures the growth and misses only a constant; the
local log-slope approaches $2$ only as $2-2/(\nu t_1)$. The inflation is finite
at every $\nu t_1$ and large once the mode has relaxed before the first
observation.}
\label{fig:inflation}
\end{figure}

This is a nuisance-conditioning statement about one scalar, and we state it as
such. Earlier drafts of this work claimed a phase transition in which
co-estimating the dynamics collapses the polynomial rate to a logarithmic one.
That claim was tied to a spectral model this paper no longer uses, and we
withdraw it: nothing here establishes a dichotomy, and \eqref{eq:inflation} is
a finite factor for each mode, not a change of rate.

\section{Strong-form diagnostics and failed repairs}\label{app:strong}

\subsection{First-order nuisance sensitivity}

\begin{proof}[Proof of Prop.~\ref{prop:notorth}]
Write the strong-form residual at slice $k$ as
$r_\theta=\A^{\eta}_{\rho_k}F_\theta-b^\eta_k$, where $\eta$ collects the score
and temporal-source nuisances, and the moment as $m_p(\theta,\eta)
=\E[r_\theta\chi_p]$. Then $\partial_\eta m_p=\E[(\partial_\eta
\A^\eta F_\theta-\partial_\eta b^\eta)\chi_p]$, which contains
$\E[F_\theta\cdot\delta s\,\chi_p]$ and does not vanish at the truth unless
$F_\theta$ is orthogonal to the score perturbation in every probe direction.
Solving the moment equations to first order, the induced parameter bias is
$-\mathcal H^{-1}\partial_\eta m\,\delta\eta$ with $\mathcal H$ the profiled
Hessian of Prop.~\ref{prop:hessian}. Reading out a solenoidal coordinate with
curvature $\lambda$ multiplies the bias by $\lambda^{-1}$. Since $\lambda\to0$
precisely along near-gauge directions, nuisance error is amplified exactly
where the signal is weakest. Exact orthogonalization of this moment is
unavailable whenever $F$ has a nonzero solenoidal component, because
$\partial_\eta m$ then has a component along the solenoidal block that no
reweighting of $\chi_p$ removes.
\end{proof}

\begin{proposition}[Filtering cannot separate signal from nuisance bias]
\label{prop:filter}
Let $\mathcal H=\sum_i\lambda_iv_iv_i^\top$ be the profiled Hessian in the
solenoidal basis and let a spectral filter act diagonally,
$\widehat\theta_\phi=\sum_i\phi(\lambda_i)\lambda_i^{-1}v_i^\top z$, where $z$ is the right-hand side of the profiled normal equations. Then both
the signal and the nuisance bias along $v_i$ are multiplied by the same factor
$\phi(\lambda_i)$, so the ratio of bias to signal in every coordinate is
invariant under $\phi$. No diagonal filter improves it.
\end{proposition}

\begin{proof}
Both the signal contribution and the first-order nuisance bias enter
$v_i^\top z$ linearly and are acted on by the same scalar
$\phi(\lambda_i)\lambda_i^{-1}$; their ratio is therefore unchanged. Filtering
trades variance against bias uniformly, which is useful, but it cannot separate
two contributions that occupy the same eigendirection.
\end{proof}

Fig.~\ref{fig:repair} separates the two failures: the optimization one is a
solver choice, the nuisance one is structural.

\subsection{The coefficient-side correction is not an orthogonalization}

A natural repair solves $\mathcal H^\top r_c=\ell$ for a Riesz representer and
reports $\ell^\top a-r_c^\top m(a,\eta)$. In the local model
$m(a,\eta)=\mathcal H(a-a^\star)+b(\eta)$ with $b(\eta^\star)=0$, the corrected
functional is exactly $\ell^\top a^\star-r_c^\top b(\eta)$, whose nuisance
derivative is $-r_c^\top D_\eta b$ --- generally nonzero. The correction removes
first-order sensitivity to the \emph{initial coefficient estimate}, not to the
nuisance, and at a same-sample exact moment root it is identically zero.
Cross-fitting does not change this derivative. We therefore describe the
implemented method as a one-step estimating-equation correction and report its
behaviour empirically; we do not claim Neyman orthogonality, and
Prop.~\ref{prop:notorth} is not a proof that every orthogonal construction must
fail.

\begin{proposition}[The targeting window]\label{prop:targetsnr}
The correction's own sampling error is amplified by the same $\lambda^{-1}$
that amplifies the bias it removes. The corrected estimator improves on the
plug-in only when the nuisance error lies below a threshold set by that
amplification.
\end{proposition}

Measured on the planted-curl benchmark over eight paired seeds
(Fig.~\ref{fig:targeting}), the threshold sits below the practical
denoising-score-matching error floor: the targeting study's plug-in at
$\varepsilon=0.10$ score error gives $2.547\pm1.201$ (the Gauss--Newton arm of
Tab.~\ref{tab:ou}, on a different eight-seed set, gives $2.461\pm0.689$) and the targeted variant
does not recover the score-free weak form's performance. The two failure modes
--- amplitude bias from the temporal log-density value channel, and gradient
error from the score --- separate cleanly: a probe-wise decomposition attributes
$38\%$ to the former against $7\%$ to the latter. Denoising score matching
supervises $\nabla\log\rho_t$ but not the additive normalization of
$\log\rho_t$, which is exactly the quantity \eqref{eq:weak} eliminates.

\begin{figure}[htb]
\centering
\includegraphics[width=0.92\linewidth]{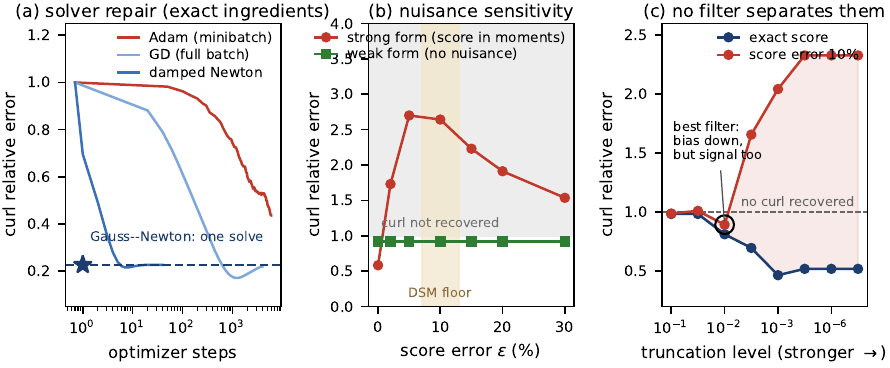}
\caption{\textbf{Repairing the strong form, and the limit of repair.}
(a) The optimization failure is a solver choice: Gauss--Newton returns the exact
minimizer in one linear solve, damped Newton in ${\sim}5$ steps, first-order
methods in $10^3$--$10^4$. (b) The nuisance failure is structural: as score
error grows the strong form degrades through the ``no curl recovered'' line
while the weak form stays flat, its moments containing no score; the band is the
$7$--$13\%$ error floor of denoising score matching. (c) Filtering cannot rescue
it (Prop.~\ref{prop:filter}) --- the truncation that best suppresses the
nuisance bias suppresses the curl with it.}
\label{fig:repair}
\end{figure}

\begin{figure}[htb]
\centering
\includegraphics[width=\linewidth]{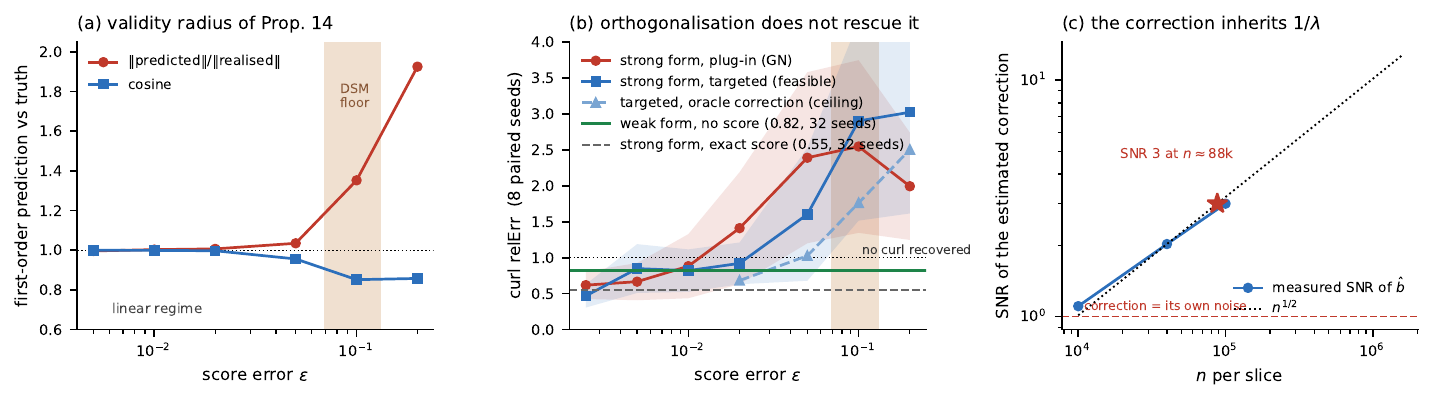}
\caption{\textbf{The targeting window.} Correcting the strong-form moment helps
only for score error $\varepsilon\lesssim0.05$, and the window shuts below the
practical denoising-score-matching floor, so the one-step route cannot rescue
the strong form at feasible nuisance quality (Prop.~\ref{prop:targetsnr}). The
correction's own sampling error is amplified by the same $\lambda^{-1}$ that
amplifies the bias it removes.}
\label{fig:targeting}
\end{figure}

\subsection{Reference arms}

With exact score and temporal source, a closed-form linear least squares
reaches solenoidal relative error $0.351$; a Gauss--Newton solve on the same
objective gives $0.535\pm0.177$ over eight seeds and $0.547\pm0.182$ over
thirty-two; Adam on the same objective gives $0.510\pm0.123$. Fitting the score
alone raises the closed-form arm to $0.862$, and fitting both score and source
to $2.19$. These arms are handed oracle nuisances and are reported as
references, not as competitors to the weak form.

\section{Weak-form implementation}\label{app:weak-impl}

\subsection{Assembly}

Probes are tensor frames $\psi_e$, either monomial $\prod_ax_a^{e_a}$ or
normalized probabilists' Hermite $\prod_a\Psi_{e_a}(x_a/\hat\varsigma_a)$, with $\Psi_k$ the
univariate case of the basis of \S\ref{sec:how-much} and $\hat\varsigma$ the pooled
per-coordinate standard deviation. Both require gradients and
Laplacians; for the Hermite frame $\Psi_k'=\sqrt k\,\Psi_{k-1}$ gives
$\partial_a\psi_e=\hat\varsigma_a^{-1}\sqrt{e_a}\,\psi_{e-e_a}$ and
$\Delta\psi_e=\sum_a\hat\varsigma_a^{-2}\sqrt{e_a(e_a-1)}\,\psi_{e-2e_a}$, and for the
monomial frame $\Delta x^e=\sum_ae_a(e_a-1)x^{e-2e_a}$. The Laplacian is what
carries the $\sigma^2\Delta\psi$ term of \eqref{eq:weak}; omitting it silently
changes the estimand, as \S\ref{sec:exp} quantifies. For degree-two monomial
probes $\Delta\psi$ is constant, so the correction is a deterministic offset to
the response and leaves the noise model untouched.

The design entry for field basis element $\phi_p=x^{e_p}e_{i_p}$ and probe
$\psi_j$ is $\tfrac{\Delta_k}{2}(\E_k+\E_{k+1})[\partial_{i_p}\psi_j\cdot
x^{e_p}]$, an empirical moment of a single monomial; the response is
$\E_{k+1}[\psi_j]-\E_k[\psi_j]-\tfrac{\Delta_k}{2}(\E_k+\E_{k+1})
[\sigma^2\Delta\psi_j]$. Both sides are computable from raw samples with no
score, density, or pointwise source.

\subsection{Feasible GLS}

Prop.~\ref{prop:gls} gives the covariance at a fixed coefficient vector, and
the covariance depends on $a$ through $f_a$. We use a same-sample unweighted
least-squares pilot $\widehat a_0$, build $\widehat V(\widehat a_0)$, and take
one frozen GLS step; the one-step feasible variant that rebuilds
$\widehat V(\widehat a_1)$ is reported separately where it differs. The
covariance is regularized by an eigenvalue floor at $10^{-10}$ times the
largest eigenvalue before the Cholesky factorization, and singular blocks are
handled by the same truncation as the design. The distinction between the exact
population covariance algebra and the estimated whitening matters: the
finite-sample coverage of a same-sample pilot is not the coverage the
fixed-weight proposition describes, and we do not claim it.

\subsection{Truncation, gauge exclusion, and the quadrature floor}

Fig.~\ref{fig:estimation} collects the recovery, calibration and conditioning
diagnostics. The whitened design is column-equilibrated before a truncated SVD
solve, so the
truncation metric is comparable across monomial columns spanning orders of
magnitude. The cutoff is selected by split-sample scoring on independently
assembled halves, with a tie-break preferring the least truncation among
candidates within $15\%$ of the best score: the half-sample fit has a
$\sqrt2$-higher noise floor, so its optimum is biased toward over-truncation.

Gauge columns are excluded structurally rather than penalized. For the linear
benchmark these are the zero-curvature directions of Prop.~\ref{prop:hessian};
for the Hermite construction they are the blockwise kernels of
\S\ref{sec:how-much}. A downstream monitor verifies that the fitted field has
exactly zero component along the declared gauge.

Because the snapshot times are fixed, increasing $n$ does not remove the
trapezoidal bias. The residual of \eqref{eq:weak} is $O(\Delta_k^3)$ per
interval, but the design entries carry the factor $\Delta_k/2$, so the induced
bias in the coefficients, and hence in the drift, is $O(\Delta^2)$. On a two-dimensional OU control with exact
moments, the population drift error falls as $\Delta^2$ --- $3.0\times10^{-2}$,
$7.5\times10^{-3}$, $1.9\times10^{-3}$, $4.6\times10^{-4}$ at gaps $0.2$,
$0.1$, $0.05$, $0.025$ --- confirming the order and showing that the floor is a
property of the design, not of the sample size. On the main benchmark the
corresponding floor is $0.033$ in solenoidal relative error.

\begin{figure}[htb]
\centering
\includegraphics[width=\linewidth]{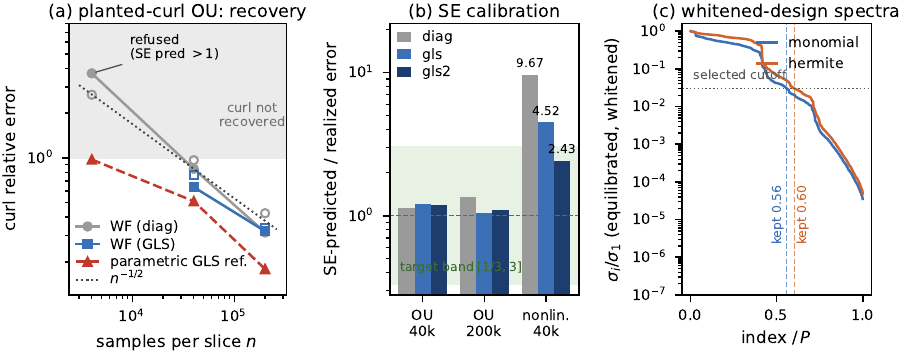}
\caption{\textbf{Recovery, calibration, conditioning.} (a) Curl error against
$n$, filled for realized and open for standard-error predicted; the power check
of Prop.~\ref{prop:hessian} refuses the $n{=}4$k design. (c) Equilibrated
spectra of the whitened nonlinear design: Hermite frames retain more of the
spectrum and decay more slowly, which is where their advantage lies.
\textbf{Panel (b) is retained as a record of a withdrawn claim} and should not
be read as a current result: its $9.7\times$ calibration gap was measured on the
lineage-tracked design of App.~\ref{app:caseB}, where the slices are not
independent, and with independent slices the diagonal predictor is calibrated at
$1.00\times$ with nothing to repair (App.~\ref{app:log}). Panels (a) and (c)
were computed under the velocity convention and are qualitatively unchanged
under the correction, but their absolute values are those of
App.~\ref{app:log}'s first column.}
\label{fig:estimation}
\end{figure}

\section{A snapshot-only finite-basis error bound}\label{app:weak-proof}

Thm.~\ref{thm:oracle-rate} is a statement about the tangent experiment and its degreewise benchmark,
under unproved assumptions.
This appendix gives a complementary guarantee of a different kind: finite
sample, snapshot only, no Gaussian-reference or spectral hypothesis, for the
estimator actually implemented. It is much weaker than the oracle rate --- it
says nothing about a nonparametric exponent --- but it is unconditional in the
assumptions the rate needs, and it makes the quadrature obstruction explicit.

Fix the probe set, the field basis, and a retained coefficient subspace with
orthonormal coordinates, and restrict $\widehat X$ to that subspace. Write
$X=\E\widehat X$, $b=\E\widehat b$, and $r=b-Xa^\star$ for the population
residual of a reference coefficient vector $a^\star$.

\begin{theorem}[Snapshot-only finite-basis error bound]\label{thm:weak-bound}
Let $a^\star$ and the whitener $W$ be deterministic, and let
$\alpha:=s_{\min}(WX)>0$ on the retained subspace. Suppose the required second
moments are finite and
$\Pr(\|W(\widehat X-X)\|_{\mathrm{op}}>\alpha/2)\le\eta$. Then for the
unregularized solution $\widehat a_0$ of
$\min_a\|W(\widehat b-\widehat Xa)\|^2$ and any linear readout $C$, with
probability at least $1-\eta-\delta$ and $0<\delta<1$,
\begin{equation}\label{eq:weak-bound}
    \bigl\|C(\widehat a_0-a^\star)\bigr\|
    \le
    \frac{2\|C\|_{\mathrm{op}}}{\alpha}
    \left(\|Wr\|+\sqrt{\frac{\tr\bigl(WV(a^\star)W^\top\bigr)}{\delta}}\right).
\end{equation}
One may take $\eta\le4\,\E\|W(\widehat X-X)\|_F^2/\alpha^2$. The statement holds
conditionally on an independent pilot fixing $W$ and the retained subspace,
provided the conditional hypotheses hold.
\end{theorem}

\begin{proof}
Write $E:=W(\widehat X-X)$ and $\varepsilon:=W(\widehat b-b)$, and let
$M:=WX$. On the event $\mathcal E:=\{\|E\|_{\mathrm{op}}\le\alpha/2\}$, which
has probability at least $1-\eta$, Weyl's inequality gives
$s_{\min}(M+E)\ge\alpha-\|E\|_{\mathrm{op}}\ge\alpha/2>0$, so $M+E$ has
full column rank on the retained subspace, the least-squares solution is unique,
$\widehat a_0=(M+E)^{+}W\widehat b$, and $(M+E)^+(M+E)=I$ there with
$\|(M+E)^+\|_{\mathrm{op}}\le2/\alpha$. The product $(M+E)\widehat a_0$ is only
the projection of $W\widehat b$ onto the range of $M+E$, so we do not equate it
with $W\widehat b$; instead, since $M+E=W\widehat X$,
\[
    \widehat a_0-a^\star=(M+E)^+\bigl(W\widehat b-(M+E)a^\star\bigr),
    \qquad
    W\widehat b-(M+E)a^\star=W(\widehat b-\widehat Xa^\star)=Wr+Z,
\]
with $Z:=W(\widehat b-\widehat Xa^\star)-Wr=\varepsilon-Ea^\star$, using
$Wb=Wr+Ma^\star$. Hence, on $\mathcal E$,
$\|C(\widehat a_0-a^\star)\|\le\frac{2\|C\|_{\mathrm{op}}}{\alpha}
(\|Wr\|+\|Z\|)$. Since $\E[\widehat b-\widehat Xa^\star]=b-Xa^\star=r$, $Z$ is
the centred whitened empirical moment at $a^\star$, with mean zero and covariance
$WV(a^\star)W^\top$, $V$ the block-tridiagonal covariance of
Prop.~\ref{prop:gls}. Chebyshev's inequality in the form
$\Pr(\|Z\|^2>\tr\Cov(Z)/\delta)\le\delta$ for mean-zero $Z$ gives
\eqref{eq:weak-bound} on the intersection, of probability at least
$1-\eta-\delta$. For the bound on $\eta$, Markov's inequality applied to
$\|E\|_{\mathrm{op}}^2\le\|E\|_F^2$ gives
$\Pr(\|E\|_{\mathrm{op}}>\alpha/2)\le4\E\|E\|_F^2/\alpha^2$.
\end{proof}

Three consequences deserve statement, because they are what the theorem is for.

\emph{Sampling error and quadrature bias are different terms.} The second
bracket in \eqref{eq:weak-bound} is $O(n_{\min}^{-1/2})$ through
$V(a^\star)$; the first, $\|Wr\|$, is not stochastic at all. At fixed snapshot
times $r$ is the whitened trapezoidal residual, of order $\Delta_{\max}^3$ in
moment scale by \eqref{eq:weak}, and no amount of data removes it. The design
entries carry the factor $\Delta_k/2$ while the whitener does not shrink with
the gaps (the moment noise of Prop.~\ref{prop:gls} contains the gap-independent
term $\mp\Psi$), so $\alpha$ scales with the gaps and, for comparable gaps, the
amplified bias $\|Wr\|/\alpha$ is $O(\Delta_{\max}^2)$ in coefficient scale ---
the formal version of the $\Delta^2$ floor measured in App.~\ref{app:weak-impl}.

\emph{Both terms are amplified by $1/\alpha$.} The same near-gauge
ill-conditioning that governs the oracle spectrum governs the finite-basis
bound, which is why gauge exclusion --- removing directions with $\alpha=0$
rather than penalizing them --- is not a numerical convenience.

\emph{The population target need not be the truth.} With $W$ fixed the
population least-squares target is $(WX)^{+}Wb$, which equals $a^\star$ only
when $Wr\perp\operatorname{range}(WX)$. Consistency for the drift along
refining grids therefore requires the amplified bias $\|Wr\|/\alpha\to0$ as
well as design stability; increasing $n$ alone does not deliver it. A readout
in a moment-null direction is not identified at all unless $C$ annihilates that
null space, and excluding such a direction is a modelling convention rather
than recovery of its coefficient.

Two scope limits. The theorem needs $\alpha>0$ on the retained subspace, and
empirically small singular values indicate weak identification, not a proved
structural gauge: the exclusion in $\mathcal A_{\rm vis}$ is certified by
Prop.~\ref{prop:hessian} only in the centered-Gaussian linear case, so for the
frame constructions the retained subspace is a modelling choice. And the field
metric must be specified --- orthonormality of coordinates is imposed through a
field Gram matrix under a chosen evaluation measure, not assumed for arbitrary
features.

\section{Measurement noise}\label{app:noise}

Suppose $Y=X+\varepsilon$ with $\varepsilon$ independent of $X$, mean zero and
covariance $\Sigma_\varepsilon$. Independence, or at least zero
cross-covariance, is required: mean-zero noise alone does not give
$\E[YY^\top]=\E[XX^\top]+\Sigma_\varepsilon$.

It is tempting to argue that the antisymmetric readout is immune, because
shifting a covariance by $cI$ leaves $G\Sigma-\Sigma G$ unchanged. That
argument is incomplete. For $H=S+G$ the same shift moves the drift moment by
$2cS$, so an uncorrected symmetric block can bias a \emph{jointly} profiled
estimate of $G$. The effect is real at the population level: with
\[
    H=\begin{pmatrix}-1&-0.6\\0.6&-2\end{pmatrix},\quad
    \Sigma_0=\begin{pmatrix}2&0.3\\0.3&0.5\end{pmatrix},\quad D=0.4I,
\]
regressing exact diffusion-corrected covariance derivatives on
$H\Sigma+\Sigma H^\top$ at $t=0.1,0.3,0.7,1$ recovers skew coefficient
$0.600000$ with clean covariances and $0.441094$ after adding $0.2I$ to the
design covariances without deconvolution --- a $26\%$ bias with no sampling
noise at all. We therefore do not claim universal immunity (Fig.~\ref{fig:noise}). Moment
deconvolution of the design, using a known or separately estimated
$\Sigma_\varepsilon$, restores the clean value; without it, the symmetric
nuisance reaches the solenoidal readout.

A transparent second-moment variant of the estimator isolates the design
perturbation (clean-data curl $0.849$, against $0.844$ and $0.636$ for the full
pipeline at diagonal and \textsf{GLS} weighting). Under isotropic contamination
the curl readout degrades from $0.8494$ clean to $1.145$ naive, and
deconvolution does \emph{not} repair it ($1.1501$); under a curl-aligned
contamination the same numbers are $0.8494$, $0.9345$ and $1.0274$. What
deconvolution does repair is the symmetric channel, where the distance to the
truth falls from $0.2823$ to $0.1134$. This is the honest version of the
earlier claim: the fix is exact at second moments on the block where the bias
lands, and the solenoidal channel is not protected.

\begin{figure}[htb]
\centering
\includegraphics[width=0.55\linewidth]{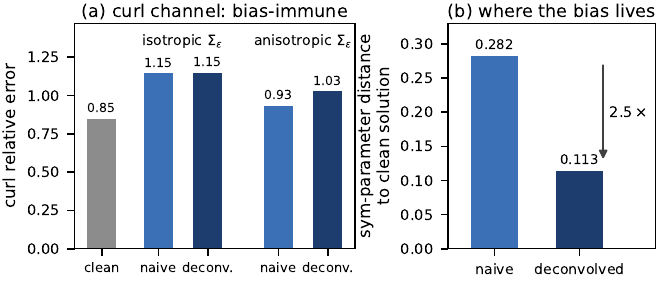}
\caption{\textbf{Measurement noise.} The symmetric channel absorbs the constant
design shift and its residual bias is removed by design deconvolution; the
solenoidal channel is \emph{not} exactly immune, contrary to an earlier version
of this claim (App.~\ref{app:log}). Isotropic contamination is the adverse case:
deconvolution of the design covariances repairs the symmetric block without
restoring the curl readout.}
\label{fig:noise}
\end{figure}

\section{Target-aware stabilization}\label{app:selection}

\begin{proposition}[Selection misalignment]
\label{prop:misalign}
In the whitened linear experiment with ridge penalty, and with risks averaged
over the signs of the coordinates of the true coefficient vector in the
eigenbasis of the profiled Hessian, the regularization strengths minimizing
moment-prediction risk and the risk of a linear functional of the field differ
in general; they coincide when the functional's squared weights are
proportional to the curvatures or the coordinate signal is constant across the
spectrum. For the solenoidal readout the gap is governed by the near-gauge
spectrum.
\end{proposition}

\begin{proof}[Proof of Prop.~\ref{prop:misalign}]
Write the whitened experiment as $y=Xa^\star+\varepsilon$ with
$\varepsilon\sim\mathcal N(0,I)$, let
$\mathcal H:=X^\top X=\sum_i\lambda_iv_iv_i^\top$, and write the ridge path as
$\widehat a_\zeta=(\mathcal H+\zeta I)^{-1}X^\top y
=(\mathcal H+\zeta I)^{-1}\mathcal Ha^\star+(\mathcal H+\zeta I)^{-1}\xi$ with
$\xi:=X^\top\varepsilon\sim\mathcal N(0,\mathcal H)$. In the coordinates
$a^\star_i:=v_i^\top a^\star$ and $\ell_i:=v_i^\top\ell$, coordinate $i$ has bias
$-\zeta a^\star_i/(\lambda_i+\zeta)$ and variance $\lambda_i/(\lambda_i+\zeta)^2$,
independently across $i$. The moment-prediction risk is
\[
    \E\|X(\widehat a_\zeta-a^\star)\|^2
    =\sum_i\frac{\zeta^2\lambda_i(a^\star_i)^2+\lambda_i^2}{(\lambda_i+\zeta)^2},
\]
which has no cross terms. The risk of $\ell^\top a$ is
$\bigl(\sum_i\ell_i\zeta a^\star_i/(\lambda_i+\zeta)\bigr)^2
+\sum_i\ell_i^2\lambda_i/(\lambda_i+\zeta)^2$; averaging over independent signs
of the $a^\star_i$, which leaves the moment risk unchanged, removes the cross
terms and gives
$\sum_i\ell_i^2\,(\zeta^2(a^\star_i)^2+\lambda_i)/(\lambda_i+\zeta)^2$. Since
$\frac{d}{d\zeta}\frac{\zeta^2c^2+\lambda}{(\lambda+\zeta)^2}
=\frac{2\lambda(\zeta c^2-1)}{(\lambda+\zeta)^3}$, the two stationarity
conditions are
\[
    \sum_iw_i\,\frac{\lambda_i\bigl(\zeta(a^\star_i)^2-1\bigr)}
      {(\lambda_i+\zeta)^3}=0,
    \qquad w_i=\lambda_i\ \text{(moment)},\quad w_i=\ell_i^2\ \text{(functional)}.
\]
If $(a^\star_i)^2\equiv c^2$ both are solved by $\zeta=1/c^2$; otherwise the two
conditions weight the same terms differently, and their roots differ unless
$\ell_i^2\propto\lambda_i$. For the solenoidal readout $\ell$ is supported on the
near-gauge coordinates where $\lambda_i$ is smallest, which the moment criterion
downweights by $\lambda_i$, so the two criteria disagree most exactly where the
estimand lives.
\end{proof}

Because whitening makes the response law exactly $\mathcal N(0,I)$, target risk
can be estimated without ground truth: perturb the whitened response under its
known law, re-run the estimator along its regularization path, and average the
resulting functional deviations. The rule needs only that the estimand be a
linear functional of the field. Its failure mode is a plain estimate that is
already near-unbiased, in which case the perturbation estimate is dominated by
variance and the selected $\zeta$ collapses to zero; on the main benchmark
split-sample selection picks $\zeta=0$ on every seed, and a fixed ridge at
$\zeta=0.03$ outperforms the selected value. We report the rule as
correctly motivated and unreliable in this regime.

\section{Single snapshots and Lagrangian information}
\label{app:caseB}

\subsection{One snapshot}

With $K=1$ the gauge is the full $\ker\A_{\rho_{t_1}}$ and
Prop.~\ref{prop:count} gives $r(x)\equiv0$: nothing about the circulation is
identified. A single snapshot constrains only the gradient part, and any
reported circulation is a property of the prior, not of the data.

\subsection{Lineage-tracked observation}

Autonomy is one way to couple slices; particle identity is another, and it is a
different observation model with different statistics. If the same particles
are observed at every time, the slices are deterministic images of one another
and the independence hypothesis of Prop.~\ref{prop:gls} fails badly.

Our nonlinear benchmark is of exactly this type, and we report it here rather
than in \S\ref{sec:exp} for that reason. The system is a deterministic ODE
($\sigma=0$) in $d=6$ with a planted degree-three solenoidal field,
$F(x)=-0.7x+(3cx_1x_2^2,-cx_2^3,0,\dots)$ with $c=0.2$, integrated by RK4 from a
Gaussian initial ensemble to five closely spaced times. Because $\sigma=0$ the
continuity identity is exact and the diffusion correction of \S\ref{sec:weak}
does not apply.

Measuring the true moment-residual covariance over $250$ independent
replications at $n=1500$ against the surrogates the estimator uses:

\begin{center}
\begin{tabular}{@{}lcc@{}}
\toprule
& $\tr V$ & ratio to truth\\
\midrule
true Monte-Carlo covariance & $2.11\times10^{-3}$ & ---\\
block-tridiagonal surrogate & $1.08\times10^{-1}$ & $51.3\times$\\
diagonal surrogate & $1.08\times10^{-1}$ & $51.3\times$\\
\bottomrule
\end{tabular}
\end{center}

Moreover $54.5\%$ of the true covariance mass lies in non-adjacent blocks,
which both surrogates set to zero by construction. That the two surrogates
agree to four digits is itself diagnostic: with coupled particles the
off-diagonal $-C_{k+1}$ correction is negligible against a covariance dominated
by cross-slice terms neither model contains.

The consequences for recovery are large. Under the coupled design the estimator
attains curl-channel error $0.4027$ with diagonal weights and $0.2095$ with
\textsf{GLS}, with a null-circulation control at $0.038$ and standard-error
ratios of $9.67\times$ and $4.52\times$. Re-drawing each slice independently ---
the observation model of \S\ref{sec:setup} --- gives $1.6455$ and $1.0142$
respectively, against a no-information level of $1.0$, with a null control at
$0.1835$ and standard-error ratios of $1.00\times$ and $1.09\times$.

Three readings follow. The recovery was a consequence of particle identity, not
of the snapshot moments. The apparent $9.67\times$ conservatism of the diagonal
predictor was not a pathology to be fixed: with independent slices that
predictor is calibrated, and its own power check, which prescribed $6.7$M
samples per slice, was correct. And the lineage-tracked design is a legitimate
and arguably more realistic observation model --- it is what pairing or lineage
barcoding supplies --- but it belongs to this appendix, because Lagrangian
information is precisely what \S\ref{sec:identifiability} sets out to do
without. An estimator for it would need a covariance model containing the
cross-slice blocks measured above; we do not develop one.

\section{Numerical verification of the analytic identities}\label{app:verification}

Each identity below is checked by an executable entry point listed in
App.~\ref{app:repro}; these are machine-precision checks of the stated algebra,
not validation of the rate theorem.

\emph{Reference-slice algebra.} $\dim S_q=mD_{q-1}-D_q$ from
Lem.~\ref{lem:surj} is checked against explicit kernels of
$\A_{\rho_\star}|_{W_q}$ for $m=2,\dots,6$ and $q\le8$; surjectivity is checked
by rank. The linear characterization of Lem.~\ref{lem:linear} is checked by
sampling random $\Sigma$ and verifying $\|\A_\rho(\Sigma Nx)\|=0$ to machine
precision for antisymmetric $N$, and $\ne0$ otherwise.

\emph{Gauge contraction.} The $q=2$ commutant is computed by explicit nullspace
sweep over $K$ for $m=3,\dots,6$: it has full dimension $\binom m2$ at $K=1$ and
collapses to $\{0\}$ already at $K=2$, confirming that the degree-two block is
cleared well before the uniform threshold. The shape-richness condition
of Thm.~\ref{thm:gauge} is checked separately by sampling $x$ and computing
$\rank[x\,|\,\Xi_1x\,|\,\cdots\,|\,\Xi_rx]$: it is deficient at $r=m-2$ and full at
$r=m-1$, i.e.\ exactly at $K=m$. The two gates are distinct and the paper
should not be read as claiming otherwise.

\emph{Propagated spectrum.} The first-moment law
$\bar\mu_q\asymp K/q$ is checked by assembling $\mathcal B_K$ on the reference
path and computing $\tr\mathcal J_{\mathrm{sol},q}/r_q$ across $q$; the fitted
inverse-trace powers are $3.20$, $4.08$ and $4.76$ at $m=3,4,5$ against the
predicted $m$, and the fitted lower-tail exponent lies in $[1.6,2.2]$.

\emph{Estimator assembly.} The Laplacian tables of App.~\ref{app:weak-impl} are
checked against central finite differences for both frames, with maximum
deviation $8\times10^{-8}$. The block-tridiagonal covariance of
Prop.~\ref{prop:gls} is checked against a $1500$-replication Monte-Carlo
covariance on an independent-slice design, agreeing to $4.6\%$ in relative
Frobenius norm with exactly zero non-adjacent blocks; the same check on the
lineage-tracked design fails by the factor reported in App.~\ref{app:caseB},
which is how that defect was found.

\section{How the curl is parameterized}\label{app:param}

The solenoidal component can be carried by several parameterizations of the same
field class; the antisymmetric potential of Prop.~\ref{prop:Sq-structure} is the
$\rho$-weighted analogue of the divergence-free parameterization of \citet{ncl2022}. They differ enormously in conditioning, and --- because first-order
optimization is sensitive to conditioning while an exact solve is not --- in how
close a descent method gets to the exact minimizer. Five variants on the linear
benchmark, five seeds each (Tab.~\ref{tab:param}), against Gauss--Newton on the same
objective:

\begin{table}[htb]
\centering
\caption{\textbf{Curl parameterizations.} Solenoidal relative error after
$8000$ descent steps; ``reached'' counts seeds whose loss came within
$2.7\%$ of the Gauss--Newton value. Conditioning is of the whitened design in
that parameterization.}
\label{tab:param}
\begin{tabular}{@{}lrrrcc@{}}
\toprule
parameterization & params & rank & cond. & error & reached\\
\midrule
antisymmetric potential & $5226$ & $144$ & $3.1\times10^{4}$ & $0.542\pm0.105$ & $4/5$\\
generator frame & $144$ & $144$ & $2.7\times10^{4}$ & $0.521\pm0.092$ & $5/5$\\
\quad + gauge reduction & $116$ & $116$ & $1.0\times10^{3}$ & $0.521\pm0.092$ & $5/5$\\
\quad + audit preconditioning & $116$ & $116$ & $1.0$ & $0.963\pm0.385$ & $2/5$\\
\quad + tangent-only support & $84$ & $84$ & $9.2\times10^{2}$ & $\mathbf{0.520\pm0.081}$ & $5/5$\\
\midrule
Gauss--Newton (exact solve) & --- & --- & --- & $0.479\pm0.194$ & ---\\
\bottomrule
\end{tabular}
\end{table}

Three readings. Over-parameterization is not free: the antisymmetric potential
carries $5226$ parameters for a design of rank $144$, and is the only variant
that fails to reach the exact minimizer on every seed. Gauge reduction buys a
factor of $26$ in conditioning at no cost in error, which is the quantitative
case for excluding gauge columns rather than penalizing them. And perfect
conditioning is not the objective: audit preconditioning drives the condition
number to $1.000$ and makes the result \emph{worse}, because it whitens the
parameterization rather than the residual, destroying the correspondence between
descent geometry and the estimand. The exact solve is insensitive to all of
this, which is the argument for using one.

Under a $10\%$ score error the ordering changes and the margins compress: the
antisymmetric potential gives $2.470$, the generator frame $2.468$,
Gauss--Newton $2.397$, and tangent-only support $1.019$. Restricting the support
to the audit-identified tangent block is the only variant that retains a
meaningful advantage once the nuisance is corrupted, which is consistent with
Prop.~\ref{prop:symmetry}: the normal-block coordinates it removes were never
identified, so they contribute variance and nuisance sensitivity without
signal.

\section{Conditioning of the linear design}\label{app:conditioning}

The Schur complement of Prop.~\ref{prop:hessian} is computable before any field
is fitted, and its spectrum is the design's own account of which solenoidal
directions are visible. On the $d=12$, $m=4$ benchmark the six tangent
antisymmetric directions carry eigenvalues
$\{5.1,7.8,14.0,23.7,40.8,71.4\}\times10^{-3}$, a spread of $14\times$ between
the best- and worst-determined direction within a single design. Recovery error
is not uniform across the curl; it is concentrated in the small end of this
spectrum, which is what makes a scalar readout of ``curl error'' a summary
rather than a diagnosis.

Adding snapshots helps, but with sharply diminishing returns, and the gain is in
the weakest direction rather than the strongest:

\begin{center}
\begin{tabular}{@{}lccc@{}}
\toprule
& $K=3$ & $K=5$ & $K=9$\\
\midrule
smallest Schur eigenvalue & $1.79\times10^{-3}$ & $3.77\times10^{-3}$ & $3.85\times10^{-3}$\\
largest Schur eigenvalue & $2.38\times10^{-2}$ & $5.47\times10^{-2}$ & $5.65\times10^{-2}$\\
pinned-direction eigenvalue & $7.56\times10^{-2}$ & $7.05\times10^{-2}$ & $6.82\times10^{-2}$\\
\bottomrule
\end{tabular}
\end{center}

Going from three to five snapshots roughly doubles the weakest eigenvalue;
going from five to nine adds $2\%$. This is Thm.~\ref{thm:gauge} seen
numerically: once $K\ge m$ the polynomial gauge is gone and further snapshots
buy only the shape diversity they happen to add, which for a fixed reference
trajectory saturates quickly. The pinned direction, fixed by a constraint rather
than by the data, is flat as expected and serves as a control.

A composite check, run under the velocity convention of App.~\ref{app:log},
closes the loop between the audit and the fit. Predicting the
solenoidal error from the Schur spectrum and the quadrature residual alone gives
$0.6326$ against a realized $0.6323$ at $n=40$k, a consistency of $0.001$; the
cosine between the estimated and planted curl rises from $0.825$ at $n=40$k to
$0.950$ at $n=200$k; and the gauge monitor returns exactly $0$, confirming the
excluded directions carry no estimate. A gradient-head control on the same data,
handed the true generator, returns $0.2196$ --- the residual is genuinely in the
solenoidal channel and not an artifact of the gradient fit.

\section{Additional robustness and negative results}\label{app:robustness}

\emph{Frames.} Monomial and Hermite frames recover comparably at $n=40$k
(full-field $0.062$ against $0.067$, curl channel $0.210$ against $0.226$); the
Hermite frame mainly improves conditioning, with reciprocal condition number
$0.03$ against $0.01$ at $n=100$k.

\emph{Degree ladder.} At $m=6$ with a planted degree-three field, fitting
degrees one through five gives $0.896$, $0.935$, $0.535$, $2.513$ and an
infeasible solve. Under-specification is graceful; over-specification is not,
and the failure is conditioning rather than variance.

\emph{Representation scaling.} $P$ grows as $504$, $5{,}460$, $81{,}900$,
$1{,}171{,}300$ at $d=6,12,25,50$ for the ambient degree-three construction.
The binding cost is the $O(P^3)$ whitening Cholesky, not enumeration. A neural
field at $d=25$ reaches $0.989$, essentially the no-information level.

\emph{Hard latent regime.} With $d_{\mathrm{latent}}=8$, $n=3000$ and four
snapshots, latent-space cosine similarity is $0.83$ and $0.86$ on two seeds,
against $0.51$ and $0.48$ for the earlier strong-form pipeline on the same
data.

\emph{Optimization.} Holding whitened moments and representation fixed,
first-order descent on the same objective continues to reduce moment loss after
the solenoidal error has turned over; over a long trace the final-minus-best
gap is $+0.040$ with standard error $0.015$ ($t=+2.68$, six seeds). We report
this as near-null overfitting only. An earlier draft attributed it to the
fitted field leaking into the exact design gauge; that attribution was wrong.
For Euclidean descent on $\|Aa-z\|^2$ every update lies in
$\operatorname{range}(A^\top)$, so $P_{\ker A}a_{r+1}=P_{\ker A}a_r$ and an
exact design-null component cannot move. The monitor that appeared to grow was
measuring a population gauge against a fitted design whose gauge columns had
been excluded in the other arm, so the two procedures did not share a feasible
model. The corrected comparison --- projected descent against the direct solve
on identical columns, initialization, whitening and penalty --- shows the exact
solve avoiding the failure by never building those columns, and shows no motion
into the design kernel.

\section{Hidden-state stress tests}\label{app:stress}

Asm.~\ref{asm:autonomy} is what converts an untestable model into a testable
one, so it deserves adversarial testing. Both diagnostics below are exact-moment
population calculations at the benchmark parameters $\nu=\sigma^2=1$: there
is no sampling and no quadrature error, so any residual is misspecification and
nothing else. They are population contrasts, not calibrated hypothesis tests.

\subsection{Misspecification that is detectable}

Take the autonomous two-dimensional system $H=-I+1.5J$ with
$J=\bigl(\begin{smallmatrix}0&-1\\1&0\end{smallmatrix}\bigr)$, isotropic
diffusion $\sigma^2I$, and $\Sigma_0=\operatorname{diag}(2,0.7)$, and retain
only the first coordinate. Fitting the exact variance derivative on $101$ times
in $[0,2]$ to the scalar autonomous linear model $\dot v=2av+2\sigma^2$ gives
$\widehat a=-1.1083$ with relative diffusion-corrected residual $0.1204$, while
fitting the full two-dimensional covariance derivatives gives residual
$3.8\times10^{-16}$.

The projection is therefore \emph{visibly} inadequate: the shape of the
observed variance path is incompatible with any autonomous scalar linear drift,
and the diagnostic says so without reference to the truth. This is the
favourable case, and it is the one a practitioner hopes for. Note its scope: it
rejects the scalar \emph{linear} model, not autonomy against every nonlinear
scalar drift.

\subsection{Misspecification that is not}

Now start the same system at its stationary covariance. Every snapshot is then
exactly compatible with the zero-rotation drift $-x$: the endpoint-moment
residual against that drift is $0$ to machine precision, at every time. An
analyst fitting autonomous models to these marginals would accept a detailed-balance
explanation of a system that is in fact circulating at $\omega=1.5$.

The marginals are not merely uninformative, they are actively misleading, and
nothing in the snapshot experiment reveals it. The projected process is not
Markov: its lag autocovariance oscillates, changing sign three times over
$[0,6]$, whereas a scalar Ornstein--Uhlenbeck autocovariance is a decaying
exponential with no sign change at all. Paired full-state measurements would
separate the two models immediately --- their lag-$0.3$ cross-covariances differ
by $0.4675$ in Frobenius norm --- but such measurements are exactly the
Lagrangian information \S\ref{sec:identifiability} sets out to do without.

This is the sharp form of the caution in \S\ref{sec:autonomy}, and it connects
to two results already in the paper. It is the extreme case of
Obs.~\ref{obs:tension}: stationary shapes give $r(x)\equiv0$ in
Prop.~\ref{prop:count}, so Thm.~\ref{thm:gauge} yields no contraction
whatsoever and the entire gauge survives, which is also why
App.~\ref{app:temporal-ou} is a warning rather than a target. Enforcing
autonomy on an inadequate representation does not here manufacture spurious
circulation; it does something quieter and worse, returning the confident
answer zero.

\begin{figure}[htb]
\centering
\includegraphics[width=0.92\linewidth]{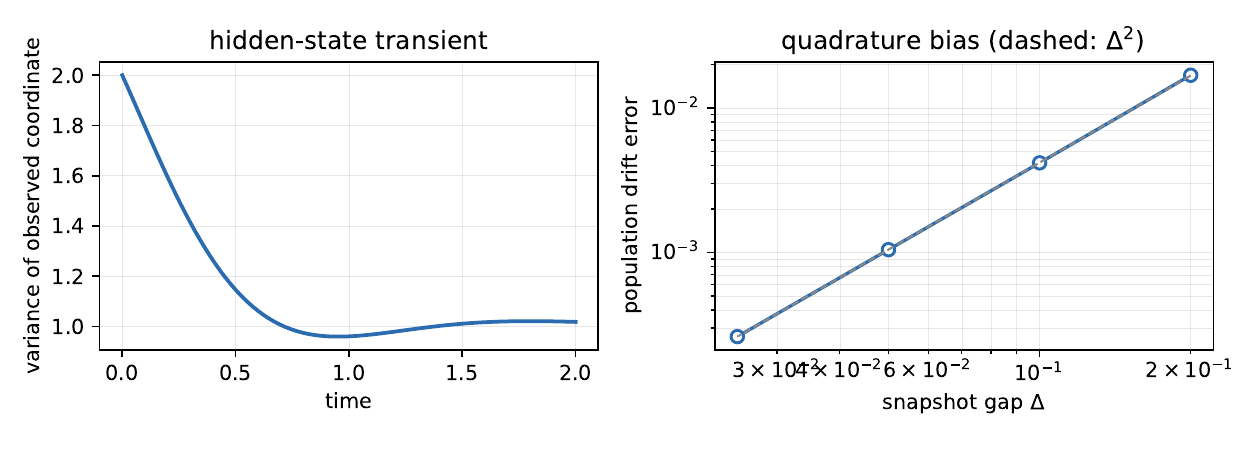}
\caption{\textbf{Controlled diagnostics.} Left: transient variance of the
projected coordinate in the hidden-state example, whose shape no autonomous
scalar linear drift reproduces. Right: population error of the
diffusion-corrected weak fit as the time grid is refined, against a
$\Delta^2$ reference; sampling alone cannot remove the quadrature bias at fixed
gap (Thm.~\ref{thm:weak-bound}).}
\label{fig:validation}
\end{figure}

Fig.~\ref{fig:validation} shows the two mechanisms. A sampling-calibrated
misspecification test and large-scale nonlinear stress studies remain future
work; neither is represented by these population checks.

We therefore do not claim a calibrated test for representation adequacy. What
the paper offers is narrower: shape movement is necessary for any recovery at
all, is computable in advance from second moments by
Prop.~\ref{prop:hessian}, and its absence is detectable even when the
inadequacy of the representation is not.

\section{Scope: autonomy, diffusion, and design size}\label{app:scope}

\subsection{Autonomy is a property of the representation}

Asm.~\ref{asm:autonomy} is not the biological claim that a regulatory program
is fixed. Suppose the observed coordinates obey
$dX_t=G(X_t,Z_t)\,dt+\sqrt{2\sigma^2}\,dW_t$, with $Z_t$ latent regulatory state
--- chromatin, protein levels, signalling context --- evolving in any way. If
$Z_t$ is part of the modelled state, the enlarged system is autonomous however
much regulation changes along $Z$. If it is omitted, It\^o's formula gives
$\frac{d}{dt}\E[\psi(X_t)]=\E[F_t(X_t)\cdot\nabla\psi(X_t)+\sigma^2\Delta\psi(X_t)]$
with $F_t(x):=\E[G(x,Z_t)\mid X_t=x]$, so the marginals of $X$ solve the
Fokker--Planck equation with drift $F_t$. This drift is autonomous when the
conditional law of $Z_t$ given $X_t=x$ does not change with $t$ --- more
precisely, when the conditional mean of $G(x,Z_t)$ does not --- and genuinely
time-dependent otherwise, the projected process then being non-Markov.
Asm.~\ref{asm:autonomy} is therefore a state-sufficiency assumption on the
chosen representation. Its failure is sometimes detectable
(App.~\ref{app:stress}, first case) and sometimes not (second case), so it can
be falsified but not certified from snapshots. Restoring it by enlarging the
state raises the intrinsic dimension $m$, the quantity that sets the threshold
of Thm.~\ref{thm:gauge} and the rate of Thm.~\ref{thm:oracle-rate}:
representation sufficiency and statistical efficiency pull in opposite
directions. Any model that predicts responses from the current state,
perturbation prediction included, presupposes the same sufficiency.

\subsection{Diffusion: known, state-dependent, unknown}

With $dZ_t=F(Z_t)\,dt+\sqrt2\,\Lambda(Z_t)\,dW_t$ and known
$D:=\Lambda\Lambda^\top$, the Fokker--Planck equation is
$\partial_t\rho=-\nabla\!\cdot(\rho F)+\sum_{ij}\partial_i\partial_j(D_{ij}\rho)$,
and dividing by $\rho_t$ gives $\A_{\rho_t}F=b^D_t$ with
$b^D_t:=-\partial_t\log\rho_t+\rho_t^{-1}\sum_{ij}\partial_i\partial_j
(D_{ij}\rho_t)$. The operator acting on the drift is unchanged; only the known
source changes. The results of \S\ref{sec:identifiability} that use the
diffusion only through the source therefore hold verbatim:
Prop.~\ref{prop:equiv}, Prop.~\ref{prop:count}, Cor.~\ref{cor:blind} applied to
drifts, and Prop.~\ref{prop:compensation}, whose Poisson problem becomes
$\nabla\!\cdot(\rho_t\nabla\Phi_t)=\partial_t\rho_t
-\sum_{ij}\partial_i\partial_j(D_{ij}\rho_t)$, again with zero integral. The
relation $F_t=v_t+\sigma^2s_t$ of \S\ref{sec:blind} becomes
$F_t=v_t+\rho_t^{-1}\nabla\!\cdot(\rho_tD)$, which is not a gradient for
anisotropic or state-dependent $D$, so drift-level saturation should then be
read from Prop.~\ref{prop:compensation} with $h_t=0$. Thm.~\ref{thm:gauge}
concerns the source constraints of Gaussian slices and involves no diffusion.
The weak moments extend as well: It\^o's formula gives
$\frac{d}{dt}\E_{\rho_t}[\psi]=\E_{\rho_t}[F\cdot\nabla\psi+\tr(D\nabla^2\psi)]$,
so \eqref{eq:weak} and Prop.~\ref{prop:gls} hold with $\sigma^2\Delta\psi_j$
replaced by the empirical moment of $\tr(D\nabla^2\psi_j)$. What does not
extend is \S\ref{sec:how-fast}: the Hermite structure of the tangent generator
(Lem.~\ref{lem:closure}) uses a Gaussian reference with constant isotropic
diffusion, and no rate is claimed beyond it.

Unknown diffusion is a different problem. The drift is then identified only
jointly with $D$, which already for linear SDEs requires conditions on the
initial law (\citealp{guan2024}; App.~\ref{app:gauss}), and every
identifiability statement here is conditional on $D$. Misspecifying $D$ biases
the weak form: with $\sigma^2$ replaced by $0$ it fits an autonomous
approximation to the time-dependent probability-flow velocity, which on the main
benchmark raises the population floor from $0.0332$ to $0.3545$
(App.~\ref{app:log}).

\subsection{Design size when \texorpdfstring{$K<m$}{K<m}}

The threshold $K\ge m$ of Thm.~\ref{thm:gauge} is uniform over all polynomial
degrees. With fewer distinct shapes, as in many single-cell time courses, the
source constraints still identify every polynomial drift of degree at most
$K-1$ and leave an exact gauge at degree $K$ (App.~\ref{app:below}; proved for
$K=m-1$, certified for $m\le6$). Three consequences follow. First, $m$ is the
intrinsic dimension of the chosen representation, not the number of measured
features; ambient dimension does not enter (App.~\ref{app:intrinsic}). Second,
for fixed $K$ one may restrict the drift class to degree at most $K-1$; the
linear class needs only $K=2$ (App.~\ref{app:gauss}). Third, the factor $K$ in
Asm.~\ref{asm:scale} measures temporal excitation rather than the snapshot
count, and designs whose shapes barely move lose it even when $K\ge m$
(Obs.~\ref{obs:tension}).

\section{Experimental protocol and negative-result log}\label{app:log}

\subsection{Protocol}

Gates were fixed before each run and are recorded with the receipts. All
comparisons within a table share seeds and draws; seed sets are stated wherever
arms come from different runs. Readouts are relative errors in the solenoidal
channel unless marked otherwise, normalized by the Frobenius norm of the
planted circulation; for null-circulation controls that denominator vanishes,
so those values are reported against the norm of the symmetric part instead and
are not comparable to the main readout. Standard errors are plug-in and
noise-only in the whitened metric. The full suite runs on a single CPU core.

Multi-seed entries are mean\,$\pm$\,standard deviation over seeds. Paired
differences are mean\,$\pm$\,standard error of the per-seed differences, and $t$
is their ratio, the paired $t$-statistic with one fewer degree of freedom than
seeds. Seed counts are small and the spread is heavy-tailed. A $32$-seed
replication of the weak-form \textsf{GLS} arm in the targeting study gives
$0.824\pm0.313$, a standard deviation about $1.2\times$ the $12$-seed value in
Tab.~\ref{tab:ou} and $1.8\times$ an earlier $8$-seed estimate, so we treat any paired difference below roughly two standard errors as no effect,
and we do not report equivalence without an equivalence test.

\subsection{Reading Tab.~\ref{tab:ou}}

The four snapshot-only arms share one $12$-seed set with identical draws, so
their differences are paired; the two strong-form rows come from a separate
$8$-seed set and are not paired against them. The single-seed columns are one
draw and are noisier than they look: seed $0$ is a high draw for the parametric
arm ($0.511$ against a $12$-seed mean of $0.308$), so the paired column is the
reliable comparison. Seed spread is heavy-tailed relative to these counts
(App.~\ref{app:log}, protocol), so small paired differences should not be read
as effects; for reference, the strong-form oracle gives $0.547\pm0.182$ over
thirty-two seeds against $0.535\pm0.177$ over eight. The strong-form oracle row uses a Gauss--Newton
solve; a closed-form linear least squares on exact ingredients reaches $0.351$
and Adam on the same objective $0.510\pm0.123$ (App.~\ref{app:strong}).

\subsection{Retracted and superseded claims}

\emph{The estimand was the wrong field.} The shipped weak form used the
deterministic continuity identity, omitting the $\sigma^2\Delta\psi$ term of
\eqref{eq:weak}, and therefore fitted an autonomous approximation to the
probability-flow velocity $v_t=F-\sigma^2s_t$, which is time-dependent even for
autonomous $F$. On the main benchmark this has a population floor of $0.3545$
against $0.0332$ with the term retained. The velocity-convention numbers are
retained here in full. The within-run comparisons reported in
Apps.~\ref{app:intrinsic}, \ref{app:noise}, \ref{app:conditioning} and
\ref{app:robustness} (frames, latent regime) come from receipts computed under
the same velocity convention and have not been rerun under the correction.

\begin{center}
\begin{tabular}{@{}llcc@{}}
\toprule
$n$ & weights & velocity convention & drift convention\\
\midrule
$40$k & diagonal & $0.8442$ & $0.8928$\\
$40$k & \textsf{GLS} & $0.6361$ & $0.6505$\\
$40$k & \textsf{GLS}$^2$ & $0.6514$ & $0.6534$\\
$200$k & diagonal & $0.3153$ & $0.3030$\\
$200$k & \textsf{GLS} & $0.3217$ & $0.2818$\\
$200$k & \textsf{GLS}$^2$ & $0.3067$ & $0.2556$\\
\midrule
population floor & --- & $0.3545$ & $0.0332$\\
\bottomrule
\end{tabular}
\end{center}

At $n=40$k sampling error dominates and the numbers barely move. At $n=200$k
the uncorrected estimator has converged to its own floor, so that column
measured misspecification rather than sample size, and the fitted $n^{-1/2}$
slope of $-0.61$ was measuring saturation. Under the correction the floor is
ten times lower, so the two sample sizes are no longer separated by a floor;
with one draw at each they are consistent with $n^{-1/2}$ scaling but do not
measure it (\S\ref{sec:exp-estimator}). A
side effect: under the old convention \textsf{GLS} was \emph{worse} than
diagonal weighting at $n=200$k, an artifact of a bias-dominated error in which
variance-optimal whitening buys nothing.

\emph{Whitening was claimed as an efficiency gain, twice, in both directions.}
An early single-seed reading ($0.844\to0.636$) was reported as a gain; an
eight-seed replication reduced it to $t=1.67$ and the claim was retracted. A
subsequent draft reinstated a stronger claim from statistics that compared
unweighted against \textsf{GLS} rather than diagonal against \textsf{GLS}. The
corrected twelve-seed comparison shows no effect of any weighting
(\S\ref{sec:exp-estimator}), and \S\ref{sec:exp-tension} gives the structural
reason.

\emph{A calibration pathology that did not exist.} The $9.67\times$
conservatism of the diagonal predictor, and its repair by \textsf{GLS}, were
measured on a lineage-tracked design in which the slices are not independent.
With independent slices the diagonal predictor is calibrated at $1.00\times$
and there is nothing to repair; its power check, which prescribed $6.7$M
samples per slice and was overridden, was correct. See App.~\ref{app:caseB}.

\emph{Descent into the exact gauge.} Withdrawn; see App.~\ref{app:robustness}.

\emph{A co-estimation phase transition.} Withdrawn; see
App.~\ref{app:temporal-ou}.

\emph{Superseded spectral model.} An earlier version analysed a stationary
ladder inside the Euclidean divergence-free class $\ker(\operatorname{div})\cap
W_q$. The present theory uses $S_q=\ker\A_{\rho_\star}\cap W_q$ and a transient
cross-slice operator, on which the reference Stein operator vanishes
identically. Machine-precision checks of the former are not validation of the
latter, and the associated lower bound, figures and verification claims have
been removed rather than re-labelled.

\subsection{Negative results retained}

Split-sample selection picks $\zeta=0$ on every seed while a fixed ridge at
$\zeta=0.03$ does better ($0.468\pm0.112$, $t=-5.09$). Random features and
the correctly specified degree-three frame show no detectable difference at six
seeds. A neural field at $d=25$ returns the no-information level. Targeting the
strong form does not close the gap to the score-free weak form at any score
error we could reach.

\section{Reproducibility}\label{app:repro}

The supplement contains the estimator library, the benchmark suite, every
experiment script, and machine-readable JSON receipts. The weak-form path
depends only on NumPy and SciPy; the strong-form comparison arms require
PyTorch and are imported lazily, so every number in Tab.~\ref{tab:ou} can be
regenerated without it.

\begin{center}
\begin{tabular}{@{}p{0.36\linewidth}p{0.31\linewidth}p{0.27\linewidth}@{}}
\toprule
reported quantity & script (in \nolinkurl{scripts/}) & receipt (in \nolinkurl{results/})\\
\midrule
Tab.~\ref{tab:ou} scaling columns; population floors & \nolinkurl{diffusion_rerun.py} & \nolinkurl{diffusion_rerun.json}\\
Tab.~\ref{tab:ou} paired column & \nolinkurl{paired_rerun.py}, \nolinkurl{parametric_matched.py} & \nolinkurl{table1_paired.json}\\
Obs.~\ref{obs:tension} & \nolinkurl{ou_shortgap.py} & \nolinkurl{ou_shortgap.json}\\
Fig.~\ref{fig:concept}(b); structural predictions of \S\ref{sec:exp} & \nolinkurl{ot_blindness.py} & printed to stdout\\
App.~\ref{app:caseB} & \nolinkurl{nonlinear_independent.py} & \nolinkurl{nonlinear_independent.json}\\
strong-form arms (App.~\ref{app:strong}) & \nolinkurl{repair_strongform.py}, \nolinkurl{paired_bench.py} & \nolinkurl{../repair/*.json}\\
targeting study (Fig.~\ref{fig:targeting}) & \nolinkurl{targeting/} & \nolinkurl{../scripts/targeting/}\\
App.~\ref{app:stress}; Fig.~\ref{fig:validation} & \nolinkurl{hidden_state_stress.py} & \nolinkurl{hidden_state_stress.json}\\
App.~\ref{app:verification} & \nolinkurl{verify_identities.py} & \nolinkurl{verification.json}\\
degree cap (App.~\ref{app:below}) & \nolinkurl{check_gauge_degree.py} & \nolinkurl{gauge_degree.json}\\
\bottomrule
\end{tabular}
\end{center}

Superseded intermediate values are retained in the receipts rather than
deleted, so the velocity-convention numbers of App.~\ref{app:log} and the
lineage-tracked numbers of App.~\ref{app:caseB} can be reproduced alongside the
corrected ones. The repository README maps each table and each figure regenerated from the current
receipts to its generating command.
\end{document}